\documentclass[journal]{arxiv}

\usepackage{listings}
\usepackage{color,soul}
\usepackage{graphicx}
\usepackage{balance}  
\usepackage{enumerate}
\usepackage{paralist}
\usepackage{multirow}
\usepackage{makecell}
\usepackage{booktabs}
 \usepackage{enumitem}
 
 \usepackage{subcaption}
 
\usepackage{layout}

\usepackage{amsfonts}
\usepackage{amsmath}
\usepackage{amsthm}
\usepackage{amssymb} 

\usepackage{mdframed} 
\usepackage{thmtools}

\definecolor{shadecolor}{gray}{0.95}
\declaretheoremstyle[
headfont=\normalfont\bfseries,
notefont=\mdseries, notebraces={(}{)},
bodyfont=\normalfont,
postheadspace=0.5em,
spaceabove=1pt,
mdframed={
  skipabove=8pt,
  skipbelow=8pt,
  hidealllines=true,
  backgroundcolor={shadecolor},
  innerleftmargin=4pt,
  innerrightmargin=4pt}
]{shaded}

 \usepackage{xcolor}
\usepackage{color}
\usepackage{graphicx}

\usepackage{verbatim}

\usepackage{babel}
\usepackage[font=small,labelfont=bf]{caption}

\newcommand{\cC}{{\cal C}}

\usepackage[colorinlistoftodos,bordercolor=orange,backgroundcolor=orange!20,linecolor=orange,textsize=scriptsize]{todonotes}

\newcommand{\myNum}[1]{(\emph{#1})}

\newcommand{\smartparagraph}[1]{\vspace{2pt} \noindent {\bf #1}}
\newcommand{\eqdef}{:=}

\usepackage{soul}

\newcommand{\ve}[2]{\langle #1 ,  #2 \rangle} 

\providecommand{\span}[1]{{\rm Span}\left\{ #1\right\}}

\newtheorem{assumption}{Assumption}
\newtheorem{lemma}{Lemma}

\newtheorem{theorem}{Theorem}

\theoremstyle{definition}

\theoremstyle{remark}
\newtheorem{remark}{Remark} 
\usepackage{bbm}

\newcommand{\R}{\mathbb{R}}

\newcommand{\E}{\mathbb{E}}
\usepackage{longtable}  

\definecolor{awesome}{rgb}{1.0, 0.13, 0.32}
\definecolor{lightcyan}{rgb}{0.88, 1.0, 1.0}
\newtheorem*{theorem*}{Theorem}
 \newtheorem*{proposition*}{Proposition}
 \newtheorem*{lemma*}{Lemma}

\usepackage{comment}
\usepackage[colorinlistoftodos,bordercolor=orange,backgroundcolor=orange!20,linecolor=orange,textsize=scriptsize]{todonotes}

\usepackage{algorithm}
\usepackage{algorithmic}
\usepackage[utf8]{inputenc} 
\usepackage[T1]{fontenc}    
\usepackage{hyperref}       
\usepackage{url}            
\usepackage{booktabs}       
\usepackage{amsfonts}       
\usepackage{nicefrac}       
\usepackage{microtype}      

\usepackage[T1]{fontenc}
\usepackage{mwe}    
\usepackage{float}
\usepackage{lineno}
\usepackage[font=footnotesize,format=plain,labelfont=bf,textfont=bf]{caption}
\usepackage{subcaption}
\usepackage{xcolor}
 
\definecolor{codegreen}{rgb}{0,0.6,0}
\definecolor{codegray}{rgb}{0.5,0.5,0.5}
\definecolor{codepurple}{rgb}{0.58,0,0.82}
\definecolor{backcolour}{rgb}{0.95,0.95,0.92}
\usepackage{pifont}

\usepackage{tikz}

\lstdefinestyle{mystyle}{
    backgroundcolor=\color{backcolour},   
    commentstyle=\color{codegreen},
    keywordstyle=\color{magenta},
    numberstyle=\tiny\color{codegray},
    stringstyle=\color{codepurple},
    basicstyle=\ttfamily\footnotesize,
    breakatwhitespace=false,         
    breaklines=true,                 
    captionpos=b,                    
    keepspaces=true,                 
    numbers=left,                    
    numbersep=5pt,                  
    showspaces=false,                
    showstringspaces=false,
    showtabs=false,                  
    tabsize=2
}
 
\let\oldenumerate\enumerate
\renewcommand{\enumerate}{
 \oldenumerate
 \setlength{\itemsep}{0pt}
 \setlength{\parskip}{0pt}
 \setlength{\parsep}{0pt}
}

\let\olditemize\itemize
\renewcommand{\itemize}{
 \olditemize
 \setlength{\itemsep}{0pt}
 \setlength{\parskip}{0pt}
 \setlength{\parsep}{0pt}
}

\begin{document}

\title{Personalized Federated Learning with Communication Compression}

\author{El~Houcine~Bergou, Konstantin Burlachenko, Aritra~Dutta,
     Peter Richt\'{a}rik
        \thanks {El~Houcine~Bergou is with Mohammed VI Polytechnic University (M6PU), Ben Guerir, Morocco.}
\thanks{Konstantin Burlachenko, Peter Richt\'{a}rik are with the Computer, Electrical and Mathematical Sciences and Engineering Division (CEMSE), King Abdullah University of Science and Technology (KAUST), Saudi Arabia.}
\thanks{Aritra~Dutta is with the Department of Mathematics and Computer Science (IMADA), University of Southern Denmark (SDU), Odense, Denmark.}}


\maketitle

\begin{abstract}
In contrast to training traditional machine learning~(ML) models in data centers, federated learning~(FL) trains ML models over local datasets contained on resource-constrained heterogeneous edge devices. Existing FL algorithms aim to learn a single global model for all participating devices, which may not be helpful to all devices participating in the training due to the heterogeneity of the data across the devices. Recently, Hanzely and Richt\'{a}rik (2020) proposed a new formulation for training personalized FL models aimed at balancing the trade-off between the traditional global model and the local models that could be trained by individual devices using their private data only. They derived a new algorithm, called {\em loopless gradient descent}~(L2GD), to solve it and showed that this algorithms leads to improved communication complexity guarantees in regimes when more personalization is required. In this paper, we equip their L2GD algorithm with a {\em bidirectional} compression mechanism to further reduce the communication bottleneck between the local devices and the server. Unlike other compression-based algorithms used in the FL-setting, our compressed L2GD algorithm operates on a probabilistic communication protocol, where communication does not happen on a fixed schedule. Moreover, our compressed L2GD algorithm maintains a similar convergence rate as vanilla SGD without compression. To empirically validate the efficiency of our algorithm, we perform diverse numerical experiments on both convex and non-convex problems and using various compression techniques.
\end{abstract}

\section{Introduction}

We live in the era of big data, and mobile devices have become a part of our daily lives. While the training of ML models using the diverse data stored on these devices is becoming increasingly popular, the traditional data center-based approach to training them faces serious {\em privacy issues} and has to deal with {\em high communication and energy cost} associated with the transfer of data from users to the data center~\cite{dean_distributed}.
{\em Federated learning}~(FL) provides an attractive alternative to the traditional approach as it aims to train the models directly on {\em resource constrained} heterogeneous devices without any need for the data to leave them \cite{FL:Jakub,kairouz2019advances}. 

The prevalent paradigm for training FL models is empirical risk minimization, where the aim is to train a {\em single global model} using the aggregate of all the training data stored across all participating devices. Among the popular algorithms for training FL models for this formulation belong FedAvg \cite{mcmahan17fedavg}, Local GD \cite{FirstLocalGDHeter, khaled_lgd}, local SGD \cite{stich_lsgd, khaled_lgd, LSGDunified2020} and Shifted Local SVRG \cite{LSGDunified2020}. All these methods require the participating devices to perform a local training procedure (e.g., by taking multiple steps of some optimization algorithm) and subsequently communicate the resulting model to an orchestrating server for aggregation; see Figure \ref{fig:l2gd}. This process is repeated until a model of suitable qualities is found. For more variants of local methods and further pointers to the literature, we refer the reader to \cite{LSGDunified2020}.

\begin{figure}
\centering
\includegraphics[width=\columnwidth]{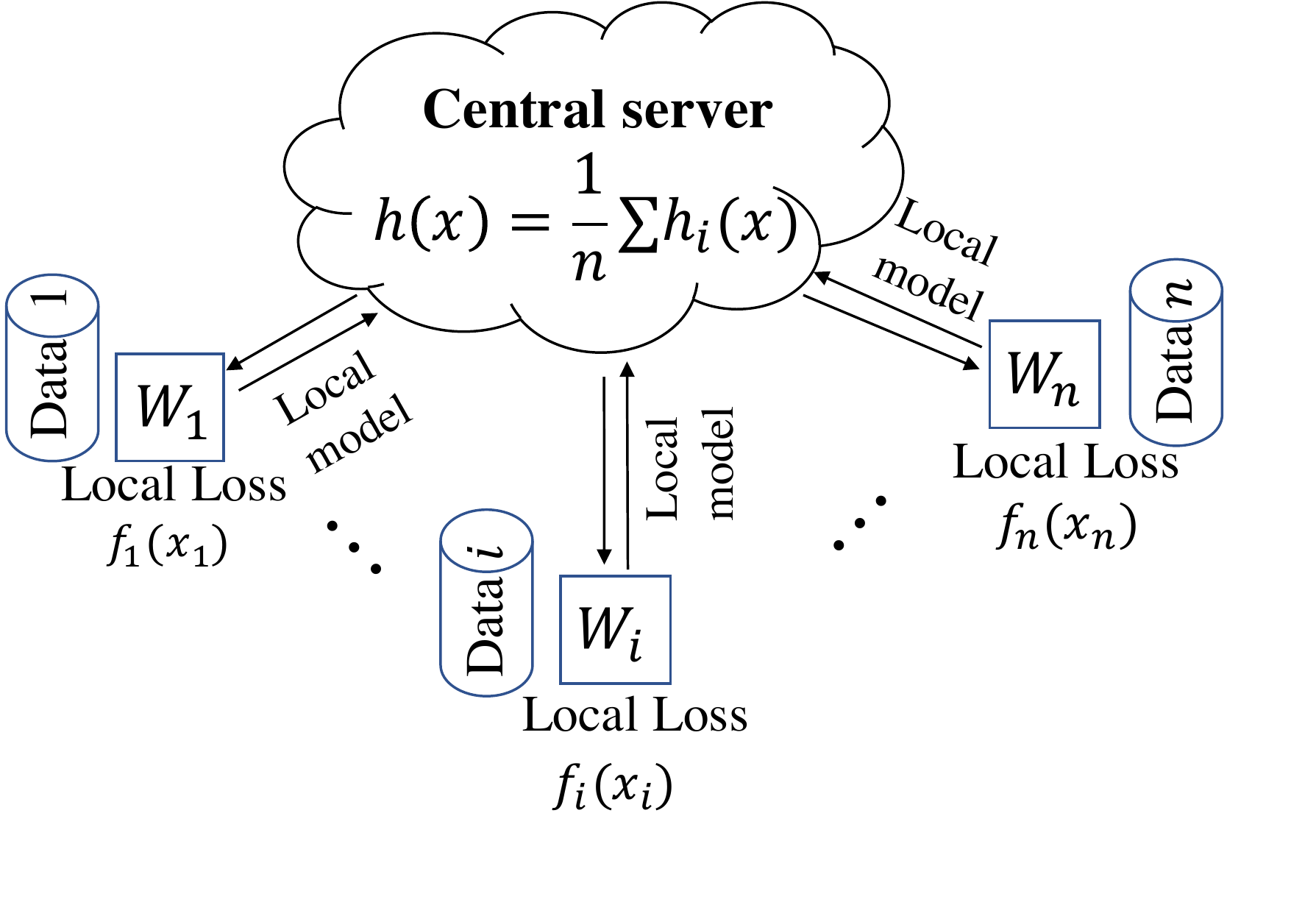}
\vspace{-3pt}
\caption{\small{Training $n$ local devices, $\{W_i\}$ on the loss, $f_i$ of their local model, $x_i$ with a central server/master node, where $h_i$ penalizes for dissimilarity between the local model, $x_i$ and the average of all local models, $\Bar{x}.$}}\label{fig:fl}
\end{figure}

\subsection{Personalized FL}
In contrast, \cite{Hanzely2020} recently introduced a new formulation of FL as an alternative to the existing ``single-model-suits-all'' approach embodied by empirical risk minimization. Their formulation explicitly aims to find a {\em personalized} model for every device; see Figure \ref{fig:fl}. In particular, \cite{Hanzely2020} considered the formulation\footnote{\cite{esgd} considered a similar model in a different context and with different motivations.}\begin{equation} \label{eq:problem}
\textstyle \min \limits_{x\in\R^{nd}} \left[F(x) := f(x) + h(x)\right]
\end{equation}
for simultaneous training of $n$ personalized FL models $x_1,\dots,x_n\in \R^d$ for $n$ participating devices. They chose
$$\textstyle{f(x) := \frac{1}{n}\sum \limits_{i=1}^n f_i(x_i), \quad \text{and}\quad h(x) := \frac{1}{n}\sum\limits_{i=1}^n h_i(x),}$$ where $f_i$ represents the loss of model $x_i$ over the local data stored on device $i$. Function $h_i$ penalizes for dissimilarity between the local model $x_i$ and the average of all local models $\Bar{x}:=\frac{1}{n} \sum_{i=1}^n x_i$, and is defined to be
$$h_i(x) = \tfrac{\lambda}{2} \|x_i -\Bar{x}\|_2^2,$$ where $\textstyle{\lambda>0}$ controls for the strength of this penalization. At one extreme, $\textstyle{\lambda\rightarrow\infty}$ forces the local models to be equal to their average and hence mutually identical. Therefore, \eqref{eq:problem} reduces to the classical empirical risk minimization formulation of FL 
\begin{equation*} 
{\textstyle\min \limits_{z\in\R^{d}} \frac{1}{n} \sum \limits_{i=1}^n f_i(z).}
\end{equation*}
On the other hand, for $\lambda=0$ problem \eqref{eq:problem} is equivalent to each client (node) training independently using their own data only. In particular, the $i^{\rm th}$ client solves 
\begin{equation*} 
\min_{x_i \in\R^{d}} f_i(x_i).
\end{equation*}
By choosing $\lambda$ to a value in between these two extremes, i.e., $\textstyle{0<\lambda<\infty}$, we control for the level of similarity we want the personalized models $\{x_i\}_{i=1}^n$ to posses. 

We remark that, local methods such as FedAvg by \cite{mcmahan17fedavg} (also see similar methods in \cite{haddadpour2019local, stich_lsgd, wang2018adaptive, zhou2018convergence, lin2019don}), are popular for training FL models. Nevertheless, their main drawback in the heterogeneous setting with data and device heterogeneity is inefficient communication. \cite{Hanzely2020} solved this, and we are using their model to build our compressed, personalized FL. 

To solve \eqref{eq:problem}, \cite{Hanzely2020} proposed a {\em probabilistic} gradient descent algorithm for which they coined the name loopless local gradient descent (L2GD). \cite{Hanzely2020} shows how L2GD can be interpreted as a simple variant of FedAvg, typically presented as a method for solving the standard empirical risk minimization (ERM) formulation of FL. However, alongside \cite{Hanzely2020} argue, L2GD is better seen as an algorithm for solving the personalized FL formulation \eqref{eq:problem}. By doing so, they interpret the nature of local steps in classical FL: the role of local steps in classical FL methods is to provide personalization and not communication efficiency as was widely believed---FedAvg can diverge on highly non-identical data partitions  \cite{mcmahan17fedavg}. Instead, communication efficiency in local methods comes from their tendency to gear towards personalization, and personalized models are provably easier to train.

\smartparagraph{Communication compression.} We observe that the {\em L2GD algorithm does not support any compression mechanism} for the master-worker and worker-master communication that needs to happen---This is the starting point of our work. {\em We believe that equipping personalized FL with fast and theoretically tractable communication compression mechanisms is an important open problem.} In distributed training of deep neural network (DNN) models, synchronous data-parallelism~\cite{dean_distributed} is most widely used and adopted by mainstream deep learning toolkits (such as {\tt PyTorch, TensorFlow}). However, exchanging the stochastic gradients in the network for aggregation creates a communication bottleneck, and this results in slower training \cite{grace}. One way to save on communication cost is to use compression operators \cite{alistarh2017qsgd, samuel_quant, grace}. Gradient compression techniques, such as quantization \cite{alistarh2017qsgd, signsgd, cnat, biased2020, UP2020}, sparsification \cite{Suresh2017, RDME, aji_sparse, sahu2021rethinking, stich2018sparsified, layer-wise, UP2020}, hybrid compressors \cite{Strom15}, and low-rank methods \cite{PowerSGD} have been proposed to overcome this issue.~\footnote{Model compression~\cite{guo2018, DFPMCI2019} is orthogonal to gradient compression and not in the scope of this work.}

Although recent works have introduced compression in traditional FL formulation \cite{FL:Jakub, fedpaq, artemis,9054168,amiri2020federated,kostopoulou2021deepreduce}; except \cite{cnat, artemis, EC-SGD, amiri2020federated}, others use compression only for the {\em throughput limited uplink} channel, that is, to upload the local models from the devices to the central server.~But limited bandwidth in the downlink channel may pose a communication latency between the server and the devices and consequently, slow down the training; see detailed discussion in \S \ref{sec:related_work}.~As of now, no study combines {\em bidirectional} compression techniques with a probabilistic communication protocol in the FL set-up by using a mixture of a local and global model as in \eqref{eq:problem}. In this work, we combine these aspects and make subsequent contributions.

\subsection{Contributions} 
\myNum{i} \textbf{L2GD algorithm with bidirectional compression.} Communication compression is prevalent in recent local FL training algorithms, but these algorithms are not robust to data and device heterogeneity. L2GD by \cite{Hanzely2020} remedies this issue by introducing personalization in FL. However, integrating compression with L2GD algorithm is a nontrivial task---unlike other FL algorithms, L2GD does not communicate after a fixed local steps, it communicates based on a probabilistic protocol; see \S\ref{sec:L2GD} and Figure \ref{fig:l2gd}. Additionally, due to this probabilistic protocol, the communication involves local model updates, as well as gradients; see \S\ref{sec:L2GD}. To reduce the communication bottleneck in L2GD, we use compression techniques on top of its probabilistic communication at both master and the participating local devices; see \S\ref{sec:L2GDComp}. To the best of our knowledge, we are the first to integrate {\em bidirectional compression techniques} with a probabilistic communication in the FL set-up, and we call our algorithm {\em compressed L2GD}; see Algorithm \ref{alg:ComL2GD}.

\myNum{ii} \textbf{Convergence analysis.} In \S\ref{sec:convergence}, we prove the convergence of our {\em compressed L2GD} algorithm based on the most recent theoretical development, such as expected smoothness as in \cite{Gower2019}. Admittedly, convergence analysis of first-order optimization algorithms with bidirectional compression exists in the literature, see \cite{tang2019doublesqueeze,cnat,amiri2020federated, layer-wise}, integrating arbitrary unbiased compressors 
with a probabilistic communication protocol into personalized FL, and showing convergence are nontrivial and the first one in its class. Our compressed L2GD algorithm maintains a similar asymptotic convergence rate as the baseline vanilla SGD without compression in both strongly convex and smooth nonconvex cases; see Theorem \ref{thm:mainconvergenceresult} and \ref{thm:nncc} in \S\ref{sec:convergence}.

\myNum{iii} \textbf{Optimal rate and communication.} We optimized the complexity bounds of our algorithm as a function of the parameters involved in the algorithm. This leads us to the ``optimal" setting of our algorithm. Mainly, we derived the optimal expected number of local steps to get the optimal iteration complexity and communication rounds; see \S\ref{sec:optimalrate}. Although our analysis is based on some hard-to-compute constants in real life, e.g., Lipchitz constant, this may help the practitioners to get an insight into the iteration complexity and communication trade-off; see Theorem \ref{thm:optimalrate} and \ref{thm:optimalcommunication} in \S\ref{sec:optimalrate}.

\myNum{iv} \textbf{Empirical study.} We perform diverse numerical experiments on {synthetic and real datasets by using both convex and non-convex problems (using 4 DNN models) and} invoking various compression techniques; see details in \S\ref{sec:empirical}, Table \ref{tab:summary}. In training larger DNN models, to obtain the same global Top-1 test accuracy, compressed L2GD reduces the communicated data-volume (bits normalized by the number of local devices or clients, $\mathrm{\#bits/n}$), from $10^{15}$ to $10^{11}$, rendering approximately $10^4$ times improvement compared to FedAvg; see \S\ref{sec:dnn}. Moreover, L2GD with natural compressor (that by design has smaller variance) empirically behaves the best and converges approximately $5$ times faster, and reaches the best accuracy on both train and the test sets compared to no-compression FedOpt \cite{reddi2020adaptive} baseline; see  \S\ref{sec:dnn} and \S\ref{sec:app_nn}. These experiments validate the effect of the parameters used, effect of compressors, and show the efficiency of our algorithm in practice.

\section{Related Work}\label{sec:related_work}
Numerous studies are proposed to reduce communication but not all of them are in the FL setting. In this scope, for completeness, we quote a few representatives from each class of communication efficient SGD.  

Smith et al.\ in \cite{smith2017federated} proposed a communication-efficient primal-dual optimization that learns separate but related models for each participating device. FedAvg by \cite{mcmahan17fedavg} performs local steps on a subset of participating devices in an FL setting. Similar to FedAvg, but without data and device heterogeneity, \cite{haddadpour2019local, stich_lsgd, wang2018adaptive, zhou2018convergence, lin2019don} independently proposed local SGD, where several local steps are taken on the participating devices before periodic communication and averaging the local models. While FedProx by \cite{fedprox} is a generalization of FedAvg, SCAFFOLD uses a variance reduction to correct local updates occurring from an non-i.i.d data in FedAvg. From the system's perspective, on {\tt TensorFlow}, \cite{bonawitz2019towards} built a FL system on mobile devices. 

Compression has also been introduced in the FL setup.~Shlezinger et al.\ \cite{9054168} combined universal vector quantization with FL for throughput limited uplink channel.~In FedPAQ by Reisizadeh et al.\ \cite{fedpaq}, each local device sends a compressed difference between its input and output model to the central server, after computing the local updates for a fixed number of iterations. While \cite{amiri2020federated} used a bidirectional compression in FL set-up, \cite{artemis} combined it with a memory mechanism or error feedback \cite{stich2018sparsified}.  

Among other proposed communication efficient SGDs, parallel restarted SGD \cite{hao2018b} reduces the number of communication rounds compare to the baseline SGD. Farzin et al.\ \cite{Farzin2018} showed that redundancy reduces residual error as compared with the baseline SGD where all nodes can sample from the complete data and this leads to lower communication overheads. CoCoA by \cite{NIPS2014_5599}, Dane by Shamir et al.\ \cite{dane} perform several local steps and hence fewer communication rounds before communicating with the other workers. Lazily aggregated gradient (LAG) algorithm by \cite{Chen2018LAGLA} selects a subgroup of workers and uses their gradients, instead of obtaining a fresh gradient from each worker in each iteration.  

In decentralized training, where the nodes only communicate with their neighbors, \cite{Koloskova2019chocosgd} implemented an {\em average consensus} where the nodes can communicate to their neighbors via a fixed communication graph. \cite{pipe_sgd} proposed Pipe-SGD---a framework with decentralized pipelined training and balanced communication. 

Personalization in FL is a growing research area. Arivazhagan et al.\ \cite{arivazhagan2019federated} proposed FedPer to mitigate statistical heterogeneity of data; also see adaptive personalized FL algorithm in \cite{deng2020adaptive}. Mei et al.\ \cite{layerwise_personalizedFL} proposed to obtain personalization in FL by using layer-wise parameters, and two-stage training; also see \cite{ma2022layer} and model personalization in \cite{shen2022cd2}. Shamsian et al.\ \cite{shamsian2021personalized} trained a central hypernetwork model to generate a set of personalized models for the local devices. Li et al.\ \cite{li2021hermes} proposed Hermes---a communication-efficient personalized FL, where each local device identifies a small subnetwork by applying the structured pruning, communicates these subnetworks to the server and the devices, the server performs the aggregation on only overlapped parameters across each subnetwork; also see \cite{pillutla2022federated} for partial model personalization in FL. DispFL is another communication-efficient personalized FL algorithm proposed by Dai et al.\ \cite{dai2022dispfl}. In recent work, Zhang et al.\ \cite{zhang2021personalized} introduce personalization by calculating optimal weighted model combinations for each client without assuming any data distribution. For connection between personalization in FL and model-agnostic-meta-learning (MAML), see \cite{fallah2020personalized}. Additionally, we refer to the surveys \cite{kulkarni2020survey,tan2022towards} for an overview of personalization in FL.

\begin{figure}
\centering
\includegraphics[width=\columnwidth]{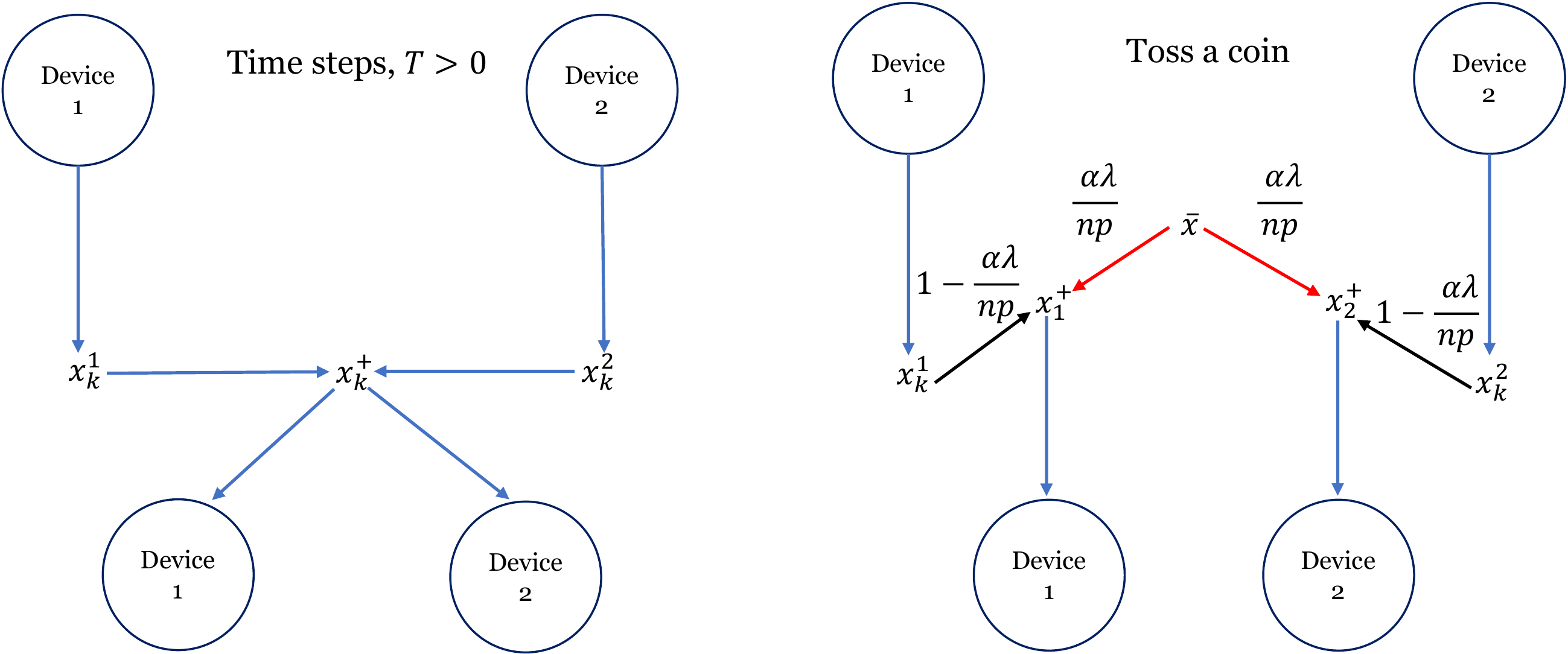}
\caption{\small{FedAvg \cite{mcmahan17fedavg} and L2GD \cite{Hanzely2020} algorithm on 2 devices. Unlike FedAvg, L2GD does not communicate after a fixed $T$ local steps, it communicates based on a probabilistic protocol.}}\label{fig:l2gd}
\end{figure}

\section{Background and Preliminaries}\label{sec:L2GD}

\smartparagraph{Notation.} For a given vector, $\textstyle{x\in\R^{nd}}$, by $x_i$ we denote the $i^{\rm th}$ subvector of $x$, and write $\textstyle{x=\left(x_1^\top,\ldots,x_n^\top \right)^\top,}$ where $\textstyle{x_i\in\R^d}$. We denote the $i^{\rm th}$ component of $x$ by $x_{(i)}$ and $\|x\|$ represents its Euclidean norm. 
By $\textstyle{[n]}$ we denote the set of indexes, $\textstyle{\{1,\ldots,n\}}.$
By $\E_{\xi}(\cdot)$ we define the expectation over the randomness of $\xi$ conditional to all the other potential random variables. 
The operator, $\textstyle{\mathcal{C}(\cdot)\eqdef\left(\mathcal{C}_1(\cdot)^\top,\ldots,\mathcal{C}_n(\cdot)^\top \right)^\top:\R^{nd}\to\R^{nd}}$ denotes a compression operator with each $\mathcal{C}_i(\cdot)$ being compatible with $x_i$.~Denote $\textstyle{Q\eqdef[I, I,\ldots,I]^\top\in\R^{nd \times d}}$, where $I$ denotes the identity matrix of $\R^{d \times d}$. With our Assumptions that we will introduce later in the paper, the problem \eqref{eq:problem} has a
unique solution, which we denote by  $x^*$ and we define  $\Bar{x}^*$ as $\textstyle{\Bar{x}^* = \frac{1}{n} \sum_{i=1}^n x_i^*.}$ By $|S|$ we denote the cardinality of a set, $S$.

\smartparagraph{Loopless local gradient descent~(L2GD).} We give a brief overview of the loopless local gradient descent (L2GD) algorithm by \cite{Hanzely2020} to solve \eqref{eq:problem} as a two sum problem. At each iteration, to estimate the gradient of $F$, L2GD samples either the gradient of $f$ or the gradient of $h$ and updates the local models via:
$$\textstyle{x_i^{k+1} = x_i^k - \alpha G_i(x^k), ~i=1,\ldots,n,}$$
where $G_i(x^k)$ for $~i=1,\ldots,n,$ is the $i^{\rm th}$ block of the vector $$\textstyle{G(x^k)=
 \left\{
    \begin{array}{ll}
        \frac{\nabla f(x^k)}{1-p} & \text{ with probability } 1-p, \\
        & \textbf{(Local gradient step)}
        \\
        \frac{ \nabla h(x^k)}{p} & \text{ with probability } p, \\
        & \textbf{(Aggregation step)}
    \end{array}
\right.}$$
 where $\textstyle{0< p <1}$, $\nabla f(x^k)$ is the gradient of $f$ at $x^k,$ and $\textstyle{\nabla_i h(x^k)=\frac{  \lambda}{n} \left( x_i^k - \Bar{x}^k \right)}$ is  the $i^{\rm th}$ block of the gradient of $h$ at $x^k$.  

In this approach, there  is a hidden communication between the local devices because in aggregation steps they need the average of the local models. That is, the communication occurs when the devices switch from a local gradient step to an aggregation step. Note that there is no need for communication between the local devices when they switch from an aggregation step to a local gradient step. There is also no need for communication after two consecutive aggregation steps since the average of the local models does not change in this case. If $k$ and $k+1$ are both aggregation steps, we have 
 $ \bar{x}^{k+1} = \frac{1}{n}\sum_{i=1}^n x_i^{k+1} = \frac{1}{n}\sum_{i=1}^n x_i^{k} - \frac{  \alpha \lambda}{n} \frac{1}{n}\sum_{i=1}^n \left( x_i^k - \Bar{x}^k \right)  = \bar{x}^{k}.$

\section{Compressed L2GD}\label{sec:L2GDComp}

Now, we are all set to describe the compressed L2GD algorithm for solving (\ref{eq:problem}). We start by defining how the compression operates in our set-up. 
\subsection{Compressed communication}
 Recall that original L2GD algorithm has a probabilistic communication protocol---the devices do not communicate after every fixed number of local steps. The communication occurs when the devices switch from a local gradient step to an aggregation step.~Therefore, instead of using the compressors in a fixed time stamp (after every $\textstyle{T>0}$ iterations, say), each device $i$, requires to compress its local model $x_i$ when it needs to communicate it to the master, based on the probabilistic protocol. We assume that device $i$ uses the compression operator, $\textstyle{\mathcal{C}_i(\cdot): \R^d \to \R^d}$. Moreover, another compression happens when the master needs to communicate with the devices. We assume that the master uses the compression operator, $\textstyle{\mathcal{C}_M(\cdot): \R^d \to \R^d}$. Therefore, the compression is used in uplink and downlink channels similar to \cite{layer-wise,cnat}, but occurs in a probabilistic fashion. There exists another subtlety---although the model parameters (either from the local devices or the global aggregated model) are communicated in the network for training FL model via compressed L2GD, the compressors that we use in this work are the compressors used for gradient compression in distributed DNN training; see \cite{grace}. 

\subsection{The Algorithm}
Note that, in each iteration $\textstyle{k\ge0}$, there exists a random variable, $\textstyle{\xi_k\in\{0,1\}}$ with $\textstyle{P(\xi_k=1)=p}$ and $\textstyle{P(\xi_k=0)=1-p}$. 
If $\textstyle{\xi_k=0}$, all local devices at iteration $k$, perform one local gradient step. Otherwise (if $\textstyle{\xi_k=1}$), all local devices perform an aggregation step. However, to perform an aggregation step, the local devices need to know the average of the local models. If the previous iteration (i.e, $\textstyle{{k-1}^{\rm th}}$ iteration) was an aggregation step (i.e, $\textstyle{\xi_{k-1}=1}$), then at the current iteration the local devices can use the same average as the one at iteration $\textstyle{k-1}$ (recall, the average of the local models does not change after two consecutive aggregation steps). Otherwise a communication happens with the master to compute the average. In this case, each local device $i$ compresses its local model $x_i^k$ to $\mathcal{C}_i(x_i^k)$ and communicates the result to the master. The master computes the average based on the compressed values of local models:
$$\textstyle{\Bar{y}^k := \frac{1}{n}\sum \limits_{j=1}^n \mathcal{C}_j(x_j^k),}$$
then it compresses $\Bar{y}^k$ to $\mathcal{C}_M(\Bar{y}^k)$ by using compression operator at the master's end and communicates it back to the local devices. The local devices further perform an aggregation step by using $\mathcal{C}_M(\Bar{y}^k)$ instead of the exact average. This process continues until convergence.  
From Algorithm \ref{alg:ComL2GD}, 
we have, for $i=1,\ldots,n: $
$$x_i^{k+1}= x_i^k - \eta G_i(x^k),$$ 
where 
$$\textstyle{G_i(x^{k})=
 \left\{
    \begin{array}{lll}
        \frac{\nabla f_i\left(x_i^k\right)}{n(1-p)} & \text{ if } \xi_k=0
       \\
        \frac{  \lambda }{np}\left( x_i^k - \mathcal{C}_M(\Bar{y}^k)\right) 
        & \text{ if } \xi_k=1 ~\& ~\xi_{k-1}=0, \\
           \frac{  \lambda}{n p} \left( x_i^k - \Bar{x}^k \right) & \text{ if } \xi_k=1 ~\& ~\xi_{k-1}=1.\\
    \end{array}
\right.}$$   
We give the pseudo code in Algorithm \ref{alg:ComL2GD}.
\begin{algorithm}
   \caption{Compressed L2GD}
   \label{alg:ComL2GD}
\begin{algorithmic}
   \STATE {\bfseries Input:} $\{x_i^0\}_{i=1,\ldots,n}$, stepsize $\eta>0$,  probability $p$,  $\xi_{-1}=1$, $\Bar{x}^{-1}=\frac{1}{n} \sum_{i=1}^n x_i^0$.
   \FOR{$k=0,1,2,\ldots$}
   \STATE \textbf{Draw:} $\xi_{k}=1$ with probability $p$\;
   \IF{$\xi_k = 0 $}
   \STATE \textbf{on all devices:} $x_i^{k+1}=x_i^k - \frac{\eta}{n(1-p)}\nabla f_i(x_i^k)$ for $i\in[n]$\;
   \ELSE 
   \IF {$\xi_{k-1}=0$}
   \STATE  \textbf{on all devices:} Compress $x_i^k$ to $\mathcal{C}_i(x_i^k)$ and communicate $\mathcal{C}_i(x_i^k)$ to the master\;
 \STATE		 \textbf{ on master:}\;
              receive $\mathcal{C}_i(x_i^k)$ from the device $i$, for all $i\in[n]$\; 
             compute $\Bar{y}^k := \frac{1}{n}\sum_{j=1}^n \mathcal{C}_j(x_j^k)$\;
             compress $\Bar{y}^k$ to $\mathcal{C}_M(\Bar{y}^k)$ and communicate it to all  devices\;
             
   \STATE      \textbf{ on all devices:} Perform aggregation step $x_i^{k+1}=x_i^k  - \frac{\eta \lambda }{np} \left( x_i^k - \mathcal{C}_M(\Bar{y}^k)\right)$\;    
  \ELSE
  \STATE  	 \textbf{ on all devices:}  $\Bar{x}^k = \Bar{x}^{k-1}$, Perform aggregation step $x_i^{k+1}=x_i^k  - \frac{\eta \lambda}{n p} \left( x_i^k - \Bar{x}^k\right)$\; 
  \ENDIF
   \ENDIF
   \ENDFOR
\end{algorithmic}
\end{algorithm}

\section{Convergence Analysis}\label{sec:convergence}
With the above setup, we now prove the convergence of Algorithm \ref{alg:ComL2GD}; see detailed proofs in \S\ref{sec:Proofs}. 

\subsection{Assumptions}
We make the following general assumptions in this paper. 
\begin{assumption} \label{ass:compression}
 For $i=1,\ldots,n$:
 \begin{itemize}
     \item The compression operator, $\textstyle{\mathcal{C}_i(\cdot):\R^d \to \R^d}$ is unbiased, 
 $$ \E_{\mathcal{C}_i}\left[\mathcal{C}_i(x)\right] = x, \quad \forall x\in \R^d.$$
 \item  There exists 
 constant, $\omega_i>0$ such that the variance of $\mathcal{C}_i$ is bounded as follows:
 $$\textstyle{\E_{\mathcal{C}_i} \left[\|\mathcal{C}_i(x) -  x\|^2\right] \le \omega_i \|  x\|^2, \forall x \in \R^d.}$$
 \item The operators, $\{\mathcal{C}_i(\cdot)\}_{i=1}^n$ are independent from each other, and independent from $\xi_k$, for all $k\ge 0$.
 \item The compression operator, $\mathcal{C}_M(\cdot)$ is unbiased, independent from $\{\mathcal{C}_i\}_{i=1}^n$ and has compression factor, $\omega_M$.
  \end{itemize}
\end{assumption}
From the above assumption we conclude that for all $x \in \R^d$, we have $$\textstyle{\E_{\mathcal{C}_i}
 \left[\|\mathcal{C}_i(x)\|^2\right] \le (1+\omega_i) \|  x\|^2.}$$
The following lemma characterizes the compression factor, $\omega$ of the joint compression operator, $\textstyle{\mathcal{C}(\cdot)= \left(\mathcal{C}_1(\cdot)^\top,\ldots,\mathcal{C}_n(\cdot)^\top \right)^\top}$ as a function of $\omega_1,\ldots,\omega_n$.
\begin{lemma}\label{lem:lm1}
Let $x \in R^{nd}$, then
$$\textstyle{\E_{\mathcal{C}} \left[\|\mathcal{C}(x)\|^2\right]\le (1+\omega) \|x\|^2,}$$
where $\textstyle{\omega = \max_{i=1,\ldots,n} \{\omega_i\}.}$
\end{lemma}
Our next assumption is on the function $\textstyle{f}$.
\begin{assumption}\label{ass:smoothBound}
We assume that $f$ is $L_f$-smooth and $\mu$-strongly convex. 
\end{assumption}

\subsection{Auxiliary results}

Before we state our main convergence theorem, we state several intermediate results needed for the convergence. In the following two lemmas, we show that based on the randomness of the compression operators, in expectation, we recover the exact average of the local models and the exact gradients for all iterations. 
\begin{lemma}\label{lem:compmean}
Let Assumption \ref{ass:compression} hold, then for all $\textstyle{k\ge 0}$, 
$\textstyle{\E_{\mathcal{C},\mathcal{C}_M} \left[\mathcal{C}_M(\Bar{y}^k)\right] = \Bar{x}^k.}$
\end{lemma}

\begin{lemma}\label{lem:unbiasGrad}
Let Assumptions \ref{ass:compression} hold. Then for all $k\ge 0$, knowing $x^k$,     
$G(x^k)$ is an unbiased estimator of the gradient of function $F$ at $x^k$. 
\end{lemma}

Our next lemma gives an upper bound on the iterate at each iteration. This bound is composed of two terms---the optimality gap, $F(x^k) - F(x^*)$, and the norm at the optimal point, $x^*$.  

\begin{lemma}\label{lem:xrelF}
Let Assumption \ref{ass:smoothBound} hold, then 
$$\textstyle{\left\| x^k\right\|^2 \le \frac{4}{\mu} \left( F(x^k) - F(x^*) \right) + 2 \left\| x^*\right\|^2.}$$
\end{lemma}

Lemma \ref{lem:mathcalA} helps us to prove the expected smoothness property \cite{Gower2019}. 
The bound in Lemma \ref{lem:mathcalA} is composed of---the optimality gap, the difference between the gradients of $h$ at $x^k$ and $x^*$, and an extra constant, $\beta$ that depends on the used compressors. 

 \begin{figure*}[t]
            \centering
            \begin{subfigure}[ht]{0.24\textwidth}
            	       \includegraphics[width=\textwidth]{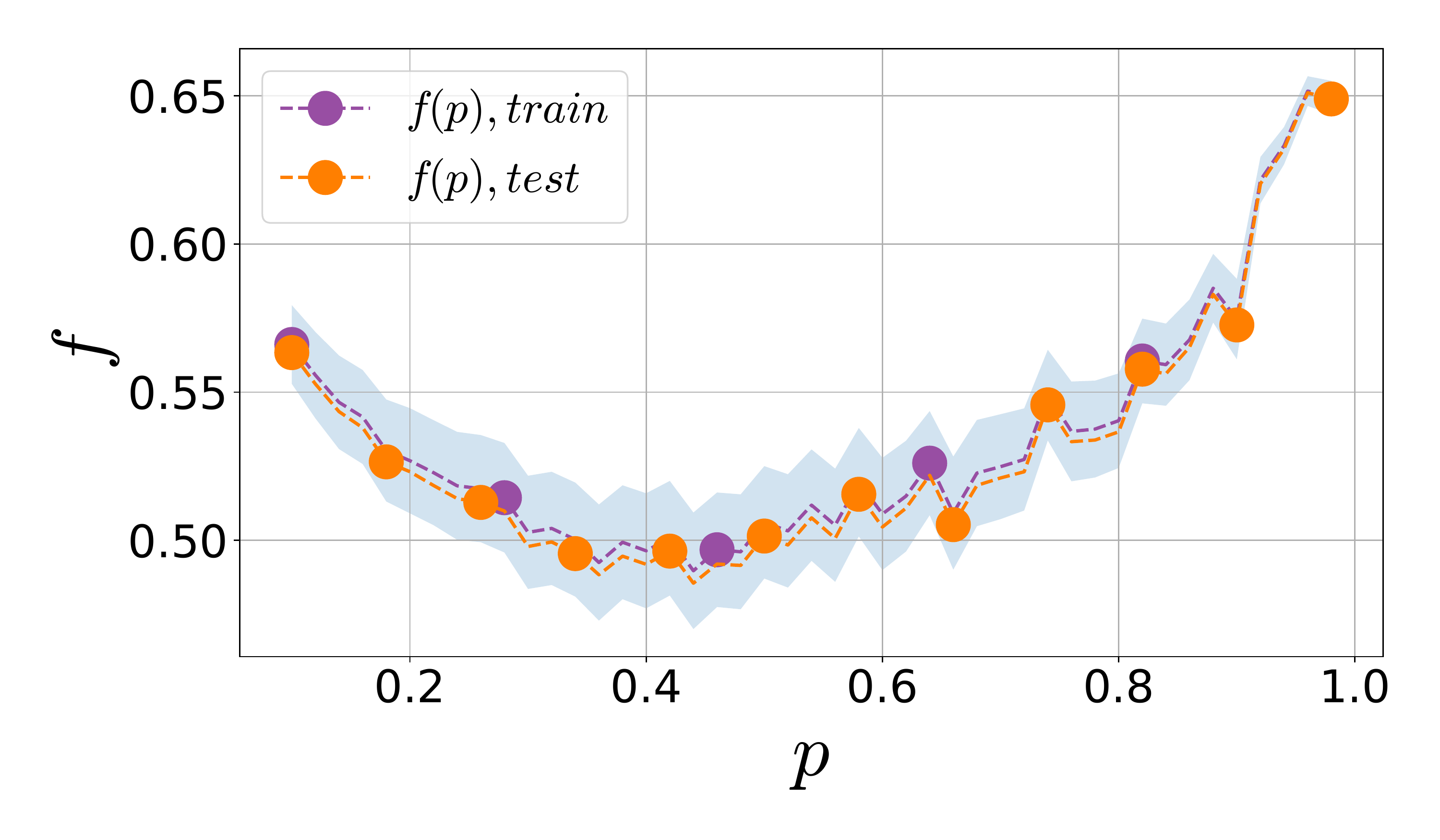} \caption{}
             \end{subfigure}
        \begin{subfigure}[ht]{0.24\textwidth}
            \includegraphics[width=\textwidth]{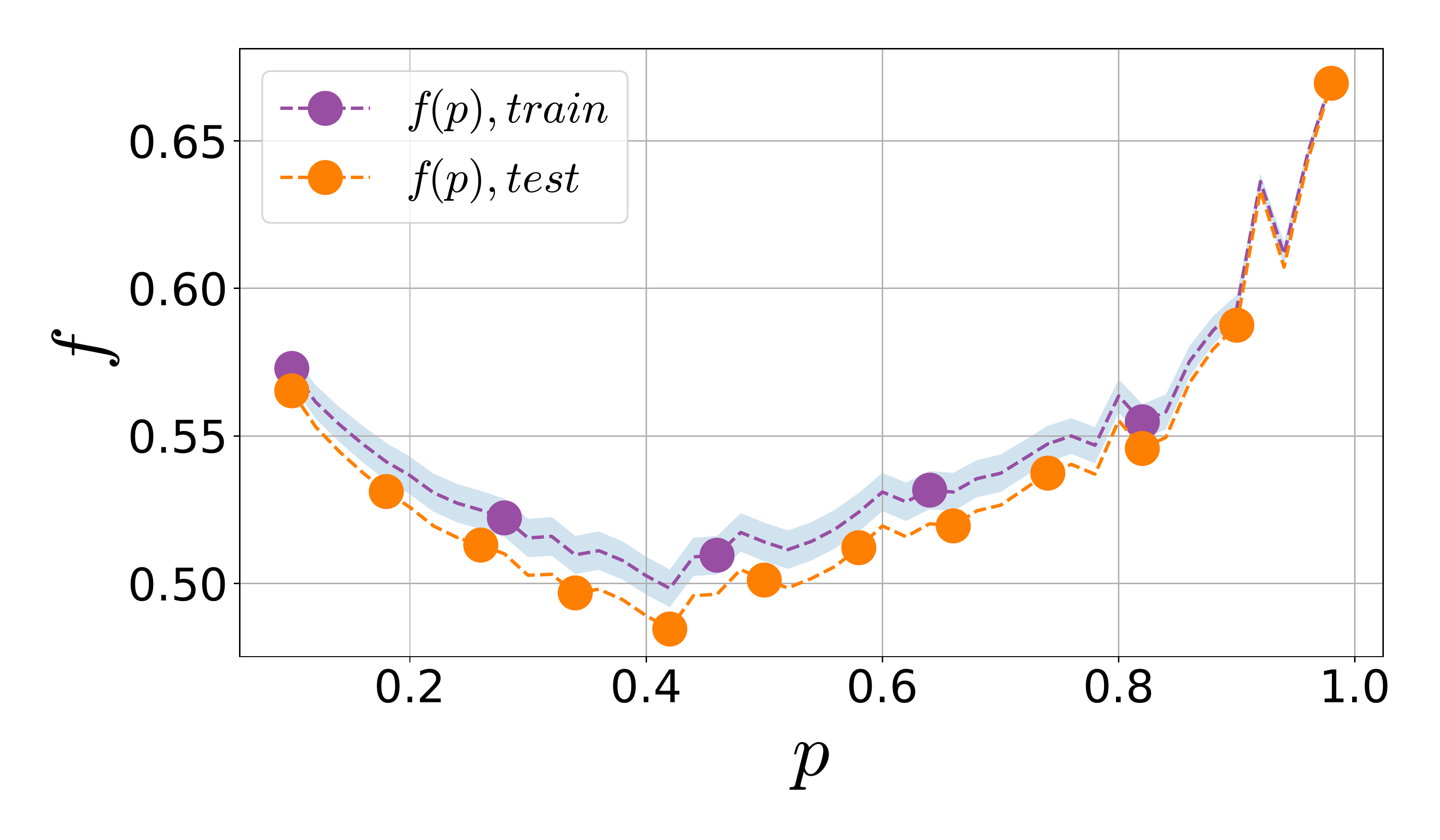} \caption{}
         \end{subfigure}
     \begin{subfigure}[ht]{0.24\textwidth}
            \includegraphics[width=\textwidth]{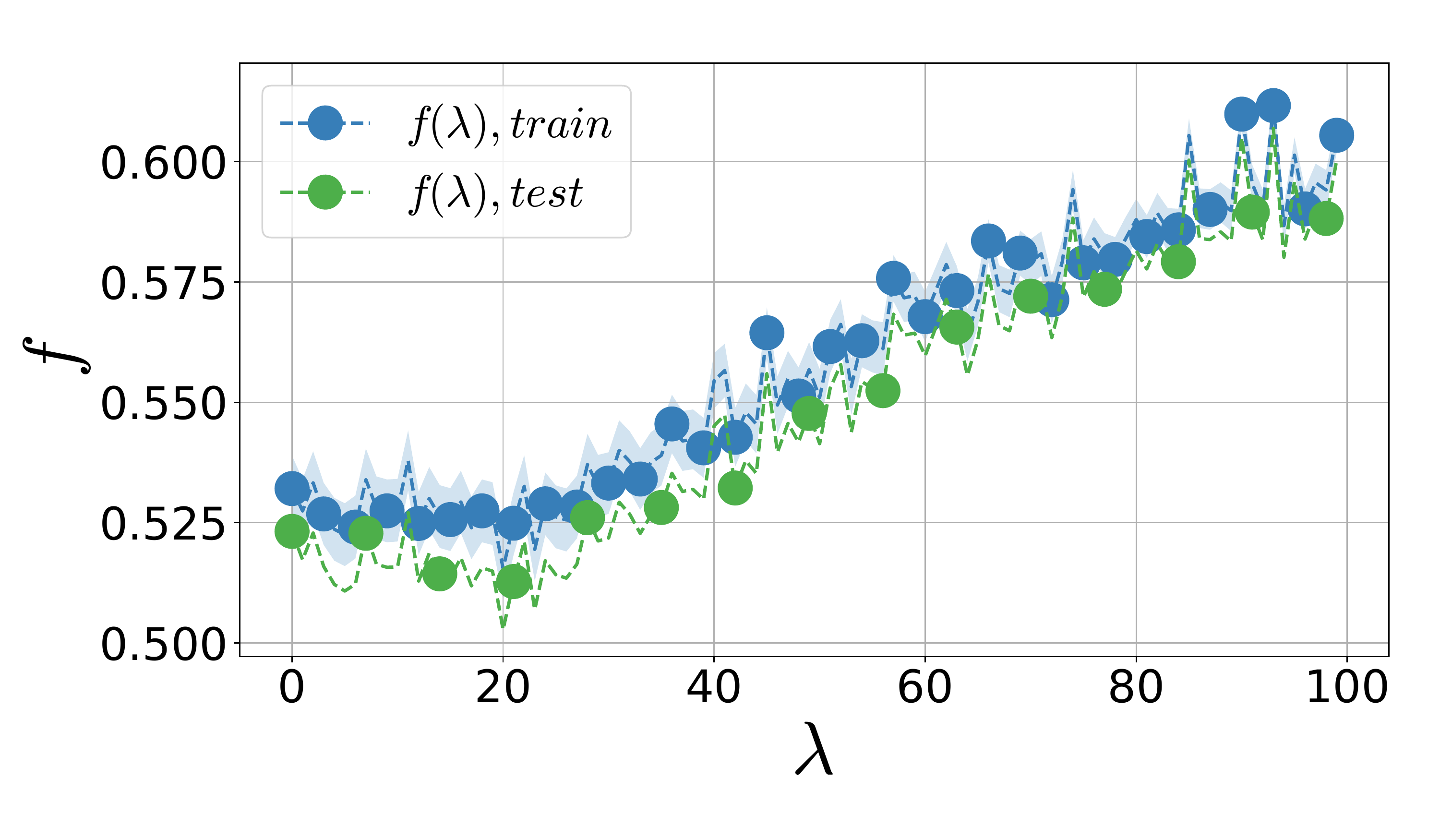}\caption{}
 \end{subfigure}
            \begin{subfigure}[ht]{0.24\textwidth}
            \includegraphics[width=\textwidth]{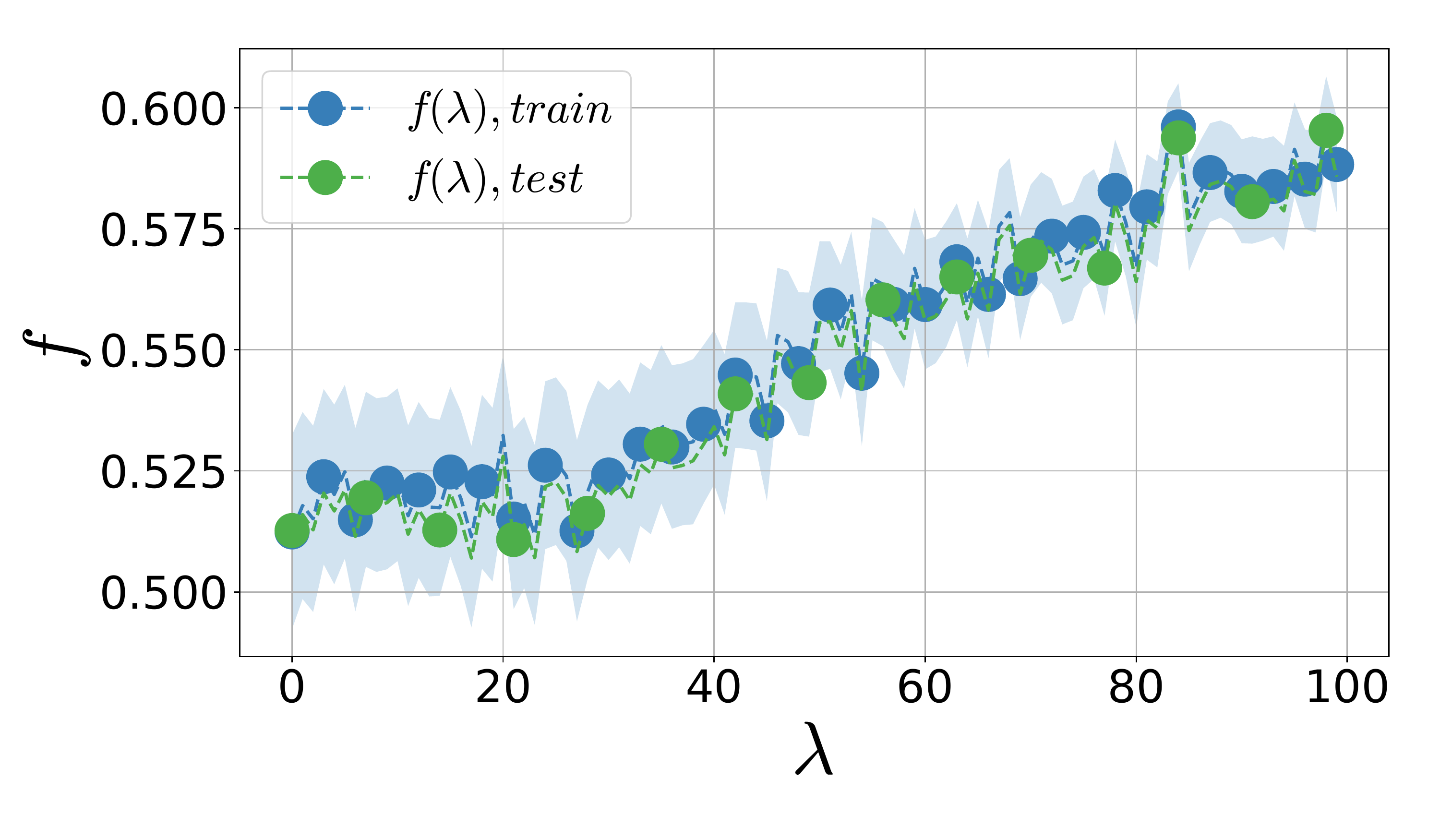} \caption{}
    \end{subfigure}
\caption{\small{Uncompressed L2GD on $n=5$ workers. We show the  loss, $f$ as a function of $p$ and $\lambda$ obtained after $K=100$ iterations of Algorithm \ref{alg:ComL2GD} with $\cC$ an identity compressor. (a) {\tt a1a} dataset, $d=124,\lambda=10$, (b) {\tt a2a} dataset, $d=124,\lambda=10$, (c) {\tt a1a} dataset, $d=124,p=0.65$ (d) {\tt a2a} dataset, $d=124,p=0.65$.}}
            	\label{fig:lambda_and_p_selection}
            \end{figure*}

\begin{lemma}\label{lem:mathcalA}
Let Assumptions \ref{ass:compression} and \ref{ass:smoothBound} hold. Then
\begin{eqnarray*}
\textstyle &\mathcal{A}:=\E_{\mathcal{C}_M,\mathcal{C}} \left\| x^k - Q\mathcal{C}_M(\Bar{y}^k) - x^* + Q\mathcal{C}_M(\Bar{y}^*) \right\|^2\\
&\le\frac{4n^2}{\lambda^2}\left\|\nabla h(x^k) - \nabla h(x^*)\right\|^2+ \alpha \left(F(x^k) - F(x^*)\right) + \beta,    
\end{eqnarray*}
where $\Bar{y}^* := \frac{1}{n}\sum_{j=1}^n \mathcal{C}_j(x_j^*)$, 
$\alpha:= \frac{4\left (4\omega + 4\omega_M (1+\omega) \right)}{\mu},$
 and 
\begin{eqnarray*}
 \beta &:=& 2 \left (4\omega + 4\omega_M (1+\omega) \right) \left\|{x}^*\right\|^2
 \\
 &+&
 4\E_{\mathcal{C}_M,\mathcal{C}}\left\|Q\mathcal{C}_M(\Bar{y}^*)-  Q\Bar{x}^*\right\|^2.
 \end{eqnarray*} 
\end{lemma}

Lemma \ref{lem:ES} is the final result that (together with Lemma \ref{lem:mathcalA}) shows the expected smoothness property and gives an upper bound on the stochastic gradient. This bound is composed of three terms---the optimality gap, $F(x^k) - F(x^*)$, the expected norm of the stochastic gradient at the optimal point, $\E \| G(x^{*})\|^2$, and some other quantity that involves interplay between the parameters used in Algorithm \ref{alg:ComL2GD} and the used compressor. 
\begin{lemma}[Expected Smoothness]\label{lem:ES}
Let Assumptions \ref{ass:compression} and \ref{ass:smoothBound} hold, then
\begin{equation}
\textstyle \E\left[\|G(x^{k})\|^2|x^k\right] \le 4 \gamma \left(F(x^k) - F(x^*)\right) + \delta,
\end{equation}  
where $$\textstyle{\gamma:= \frac{\alpha \lambda^2 (1-p)}{2 n^2 p} + \max \left\{ \frac{ L_f}{(1-p)}, \frac{\lambda}{n} \left(1+\frac{4(1-p)}{p}\right)
  \right\}}$$ and $$\delta:= \frac{2 \beta \lambda^2 (1-p)}{n^2 p}+2\E \| G(x^{*})\|^2.$$
\end{lemma}

\begin{remark} 
If there is no compression, the operators, $\mathcal{C}_i(\cdot)$, for $i\in[n]$, and $\mathcal{C}_M(\cdot)$ are equal to identity. The compression constants, $\omega_i$, for $i\in[n]$, and $\omega_M$ are equal to zero. Therefore, $\textstyle{\alpha = \beta=0}$, and the factor $4$ in the formula of $\gamma$ can be replaced by $1$ and thus  
$$\textstyle{\delta= \frac{2 \beta \lambda^2 (1-p)}{n^2 p}+2\E \| G(x^{*})\|^2 = 2\E \| G(x^{*})\|^2,}$$
with
\begin{eqnarray*}
\textstyle \gamma=\max \left\{ \frac{ L_f}{(1-p)}, \frac{\lambda}{n} \left(1+\frac{(1-p)}{p}\right)
  \right\}=\max \left\{ \frac{ L}{n(1-p)}, \frac{\lambda}{np}
  \right\},
\end{eqnarray*}
where $L = n L_f$.~Same constants arise in the expected smoothness property in \cite{Hanzely2020}.
\end{remark}

\subsection{Main result}

We now state the convergence result for Algorithm \ref{alg:ComL2GD} for both strongly convex and nonconvex cases.
\begin{theorem}(Strongly convex case)\label{thm:mainconvergenceresult} Let Assumptions \ref{ass:compression} and \ref{ass:smoothBound} hold. 
If $\textstyle{\eta \le \frac{1}{2\gamma}}$, then 
$$\textstyle{\E\left\| x^k - x^*\right\|^2 \le \left( 1-\frac{\eta \mu}{n}\right)^k \left\| x^0 - x^*\right\|^2 + \frac{n \eta \delta}{\mu}.}$$
\end{theorem}
\begin{proof}
The proof follows directly from Lemma \ref{lem:unbiasGrad}, \ref{lem:ES}, and Theorem 3.1 from \cite{Gower2019}.
\end{proof}

\begin{theorem}(Non convex case)\label{thm:nncc} Let Assumption \ref{ass:compression}  hold. Assume also that $F$ is $L$-smooth, bounded from below by $F(x^*)$. Then to reach a precision, $\textstyle{\epsilon>0}$, set the stepsize, $\textstyle{\eta=\min\{\frac{1}{\sqrt{2L\gamma K}},\frac{\epsilon^2}{L\delta}\}},$ such that for $
\textstyle{K\ge \frac{6L}{\epsilon^4}\max\{{12\gamma(F(x^0)-F(x^*))^2},{\delta}\},}$ we have 
$\min_{k=0,1,\dots,K}  \E \|\nabla F(x^k)\|_2\le \epsilon.$
\end{theorem}

\begin{remark}\label{rem:nncc}
For smooth non-convex problems, we recover the optimal $O(\epsilon^4)$ classical rate as vanilla SGD.
\end{remark}

\section{Optimal Rate and Communication} 
\label{sec:optimalrate} 
In this section, we provide the ``optimal" setting of our algorithm that is obtained by optimizing the complexity bounds of our algorithm as a function of the parameters involved.  The analysis on this section is based on the following upper bound of $\gamma$. We recall that 
\begin{eqnarray*}
\gamma &=& \frac{\alpha \lambda^2 (1-p)}{2 n^2 p} + \max \left\{ \frac{ L_f}{(1-p)}, \frac{\lambda}{n} \left(1+\frac{4(1-p)}{p}\right)
  \right\} 
\\  
  &\le&\frac{\alpha \lambda^2 (1-p)}{2 n^2 p} + \max \left\{ \frac{ L_f}{(1-p)}, \frac{4\lambda}{np} 
  \right\}:=\gamma_{u}.
  \end{eqnarray*}
Note that the number of iterations is linearly dependent on $\gamma$. Therefore, to minimize the total number of iterations, it suffices to minimize $\gamma$. Define $L := n L_f$.

\begin{theorem}[Optimal rate] \label{thm:optimalrate}
The probability $p^*$ minimizing $\gamma$ is equal to 
$\max\{p_e,p_A\}$, where $p_e = \frac{7 \lambda + L - \sqrt{\lambda^2 + 14 \lambda L + L^2}}{6 \lambda}$ and $p_A$ is the optimizer of the function $A(p) = \frac{\alpha \lambda^2}{2 n^2 p } + \frac{ L}{n(1-p)}$  in $(0,1)$.
\end{theorem}

\begin{remark}\label{rem:max_upp_bound}
If we maximize the upper bound, $\gamma_u$ instead of $\gamma$ then $p_e$ simplifies to $\frac{4 \lambda}{L + 4 \lambda}$.
\end{remark}

\begin{lemma}\label{lem:pa} The optimizer probability $p_A$   of the function $A(p) = \frac{\alpha \lambda^2}{2 n^2 p } + \frac{ L}{n(1-p)}$  in $(0,1)$ is equal to 
$$p_A=
 \left\{
    \begin{array}{lll}
       \frac{1}{2}  & \text{ if } 2nL =\alpha \lambda^2
       \\
     \frac{-2 \alpha \lambda^2 + 2\lambda \sqrt{2\alpha n L}}{2(2nL -\alpha \lambda^2)}   & \text{ if } 2nL > \alpha \lambda^2
       \\
        \frac{-2 \alpha \lambda^2 - 2\lambda \sqrt{2\alpha n L}}{2(2nL -\alpha \lambda^2)}   & \text{otherwise.} 
    \end{array}
\right.$$
\end{lemma}

Note that the number of communication rounds is linearly proportional to $C:=p(1-p)\gamma$. Therefore, minimizing the total number of communication rounds suffices to minimize $C$ or $nC$.

\begin{theorem}[Optimal communication] \label{thm:optimalcommunication}
The probability $p^*$ optimizing $C$ is equal to 
$\max\{p_e,p_A\}$, where $p_e = \frac{7 \lambda + L - \sqrt{\lambda^2 + 14 \lambda L + L^2}}{6 \lambda}$ and $p_A = 1 - \frac{Ln}{\alpha \lambda^2}$.
\end{theorem}

\begin{remark}
As in Remark \ref{rem:max_upp_bound}, we note that, if we use the upper bound, $\gamma_u$ instead of $\gamma$ then $p_e$ simplifies to $\frac{4 \lambda}{L + 4 \lambda}$.
\end{remark}  

We note that $\lambda\to0$ implies $p^*\to0$. This means that the optimal strategy, in this case, is {\em no communication at all}.~This result is intuitive since for $\lambda=0$, we deal with pure local models which can be computed without any communication.~As $\lambda\to\infty$ implies $p^*\to 1$ denoting that the optimal strategy is to communicate often to find the global model. 

\begin{table*}[t]
\caption{Gradient compression methods used in this work. Note that $\|\tilde{g}\|_0$ and $\|g\|_0$ are the number of elements in the compressed and uncompressed gradient, respectively; nature of operator $Q$ is random or deterministic. We implement that mechanisms for FedML.ai framework.}\label{tab:summary}
\small
\centering
\begin{tabular}{@{}llllccccl@{}}
\toprule
                    &   Compression                                & Ref. & Similar Methods & $\|\tilde{g}\|_0$ & Nature of $\cC$ & Implementation \\ \midrule
\multirow{2}{*}{{Quantization}}                      & QSGD              &   \cite{alistarh2017qsgd}                           & 
                    \makecell[l]{\cite{cnat,ATOMO,DBLP:conf/nips/WenXYWWCL17} \\ \cite{wu_memqsgd,Yu2019ExploringFA,zipml}}                         &                          $\|g\|_0$        &  Rand, unbiased             
                                    & FedML.ai standalone/distributed                      \\
                                        & Natural                           &   \cite{cnat}                           &  \cite{alistarh2017qsgd,Yu2019ExploringFA,zipml}                            &                        $\|g\|_0$          &   Rand, unbiased                                               & FedML.ai standalone/distributed                   \\
                    & TernGrad                          & \cite{DBLP:conf/nips/WenXYWWCL17}                             & \cite{alistarh2017qsgd,ATOMO,Yu2019ExploringFA}                             &                          $\|g\|_0$        & Rand, unbiased                                    & FedML.ai standalone/distributed                     \\
                    & Bernoulli                          & \cite{khirirat2018distributed}                             & ---                             &                          ---       & Rand, unbiased                                    & FedML.ai standalone/distributed                   
                                        \\\midrule
\multirow{1}{*}{{Sparsification}}                      
                    & Top-$k$                   &\cite{aji_sparse}                              & \cite{Alistarh18_sparse,stich2018sparsified}                             &         $k$                           &   Det, Biased                                         & FedML.ai standalone/distributed                   \\
                    \bottomrule
\end{tabular}
\end{table*}


\begin{figure*}[t]
	\centering
	\captionsetup[sub]{font=scriptsize,labelfont={}}	
	\begin{subfigure}[ht]{0.9\textwidth}
		\includegraphics[width=\textwidth]{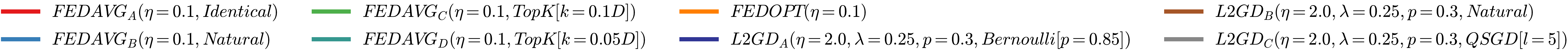}
	\end{subfigure}
	
	\begin{subfigure}[ht]{0.245\textwidth}
		\includegraphics[width=\textwidth]{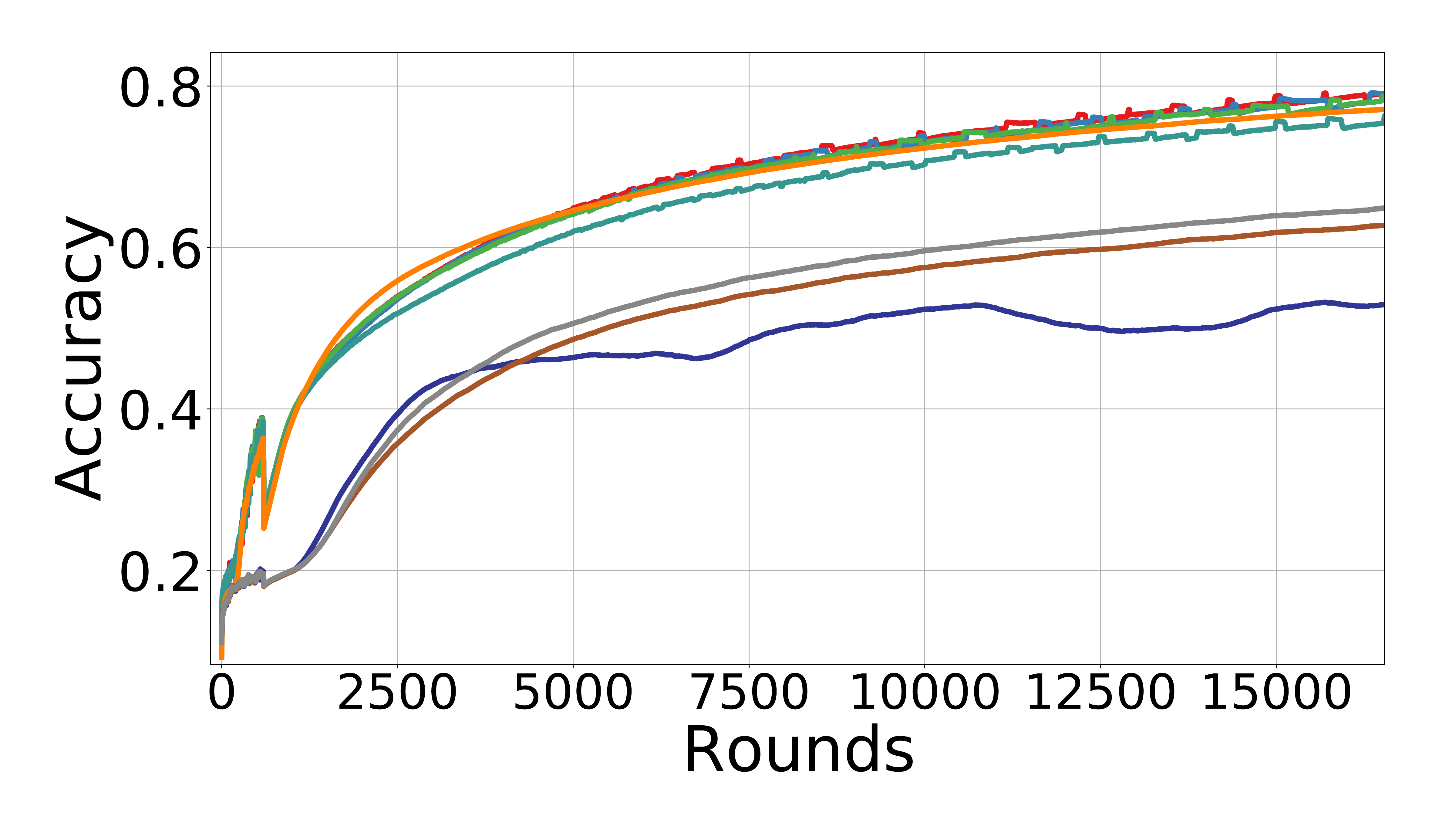} \caption{}
	\end{subfigure}
	\begin{subfigure}[ht]{0.245\textwidth}
		\includegraphics[width=\textwidth]{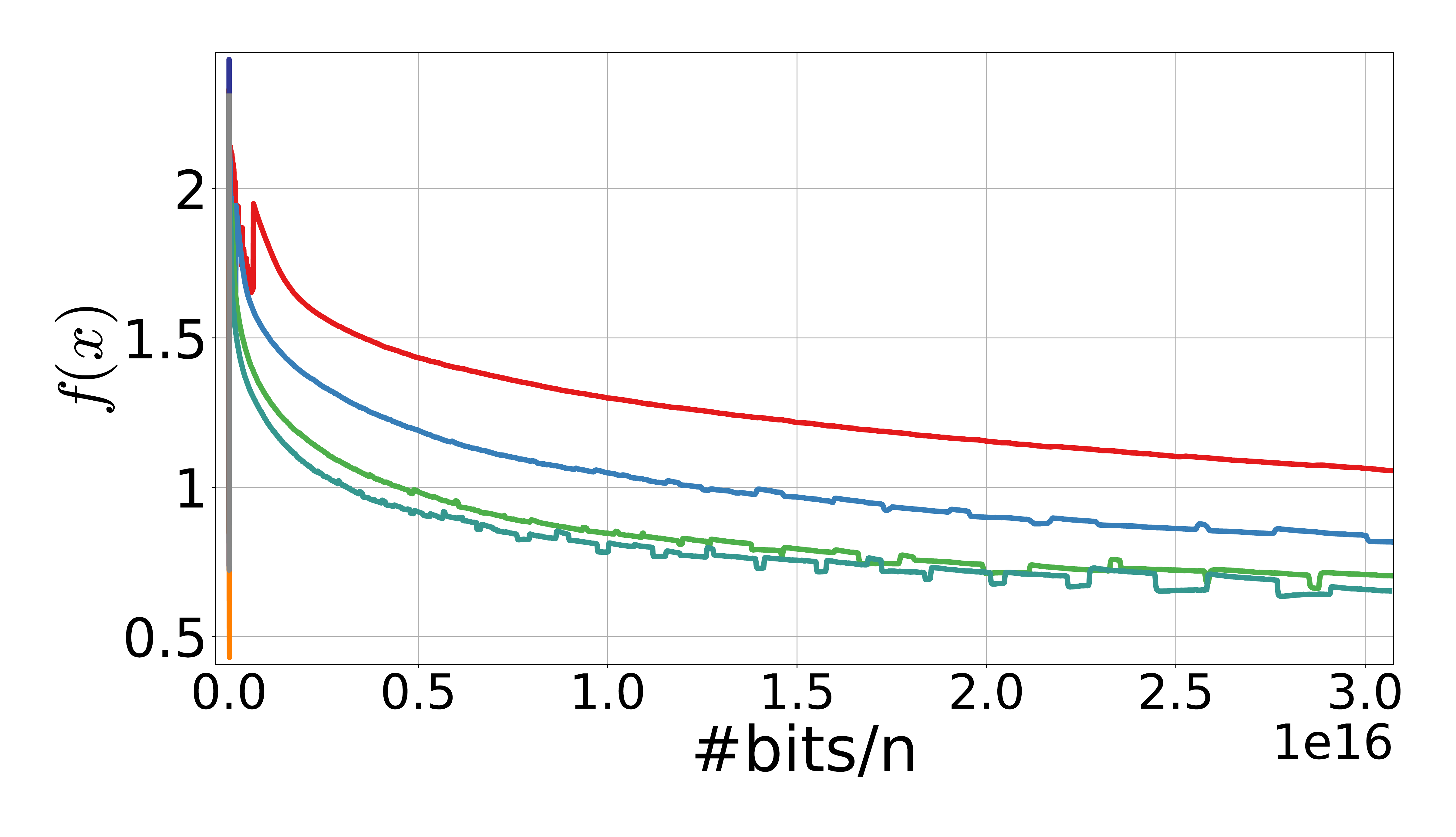} \caption{}
	\end{subfigure}
	\begin{subfigure}[ht]{0.245\textwidth}
		\includegraphics[width=\textwidth]{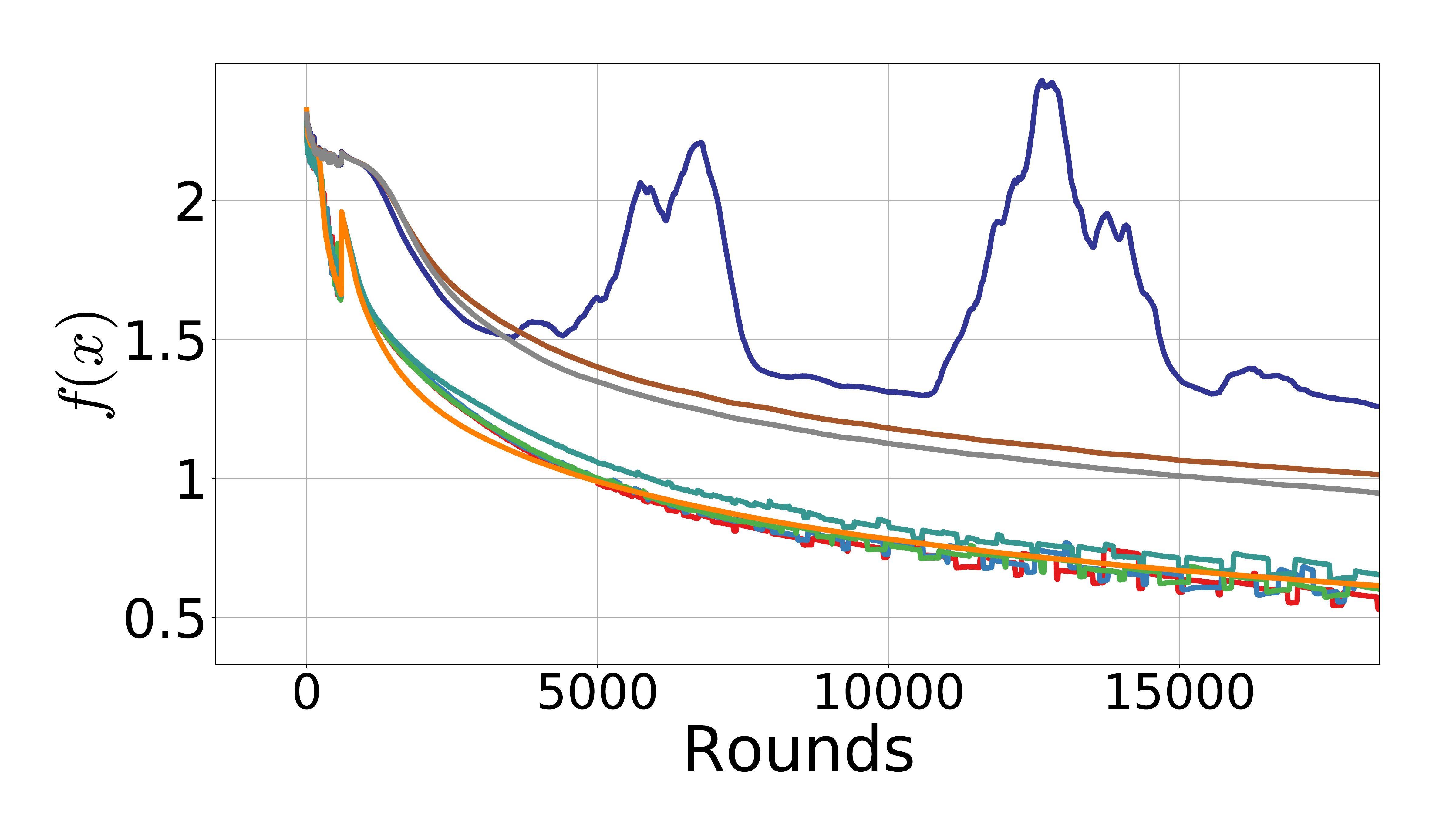} \caption{}
	\end{subfigure}
	\begin{subfigure}[ht]{0.245\textwidth}
		\includegraphics[width=\textwidth]{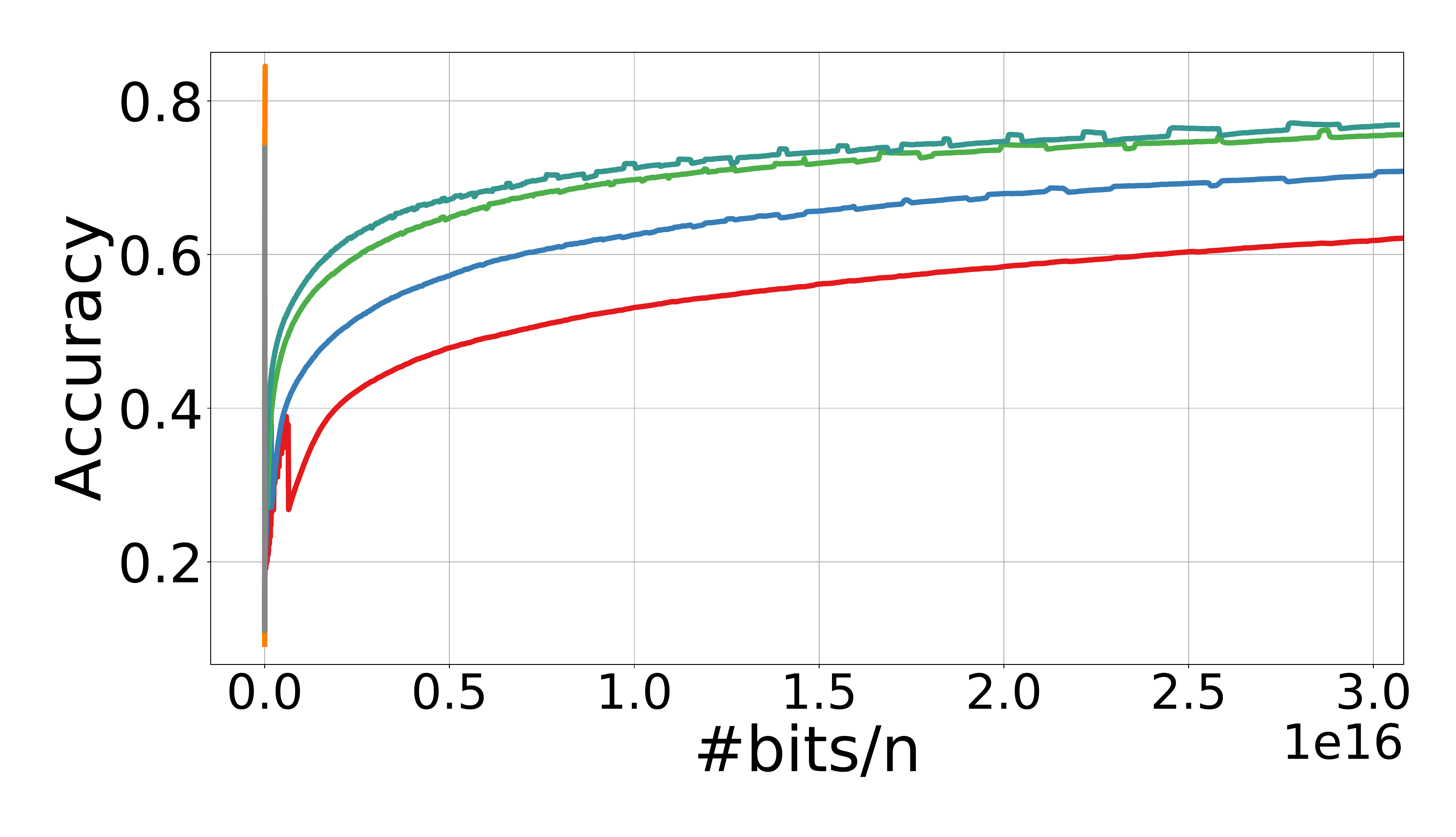}\caption{}
	\end{subfigure}
	
	\begin{subfigure}[ht]{0.245\textwidth}
		\includegraphics[width=\textwidth]{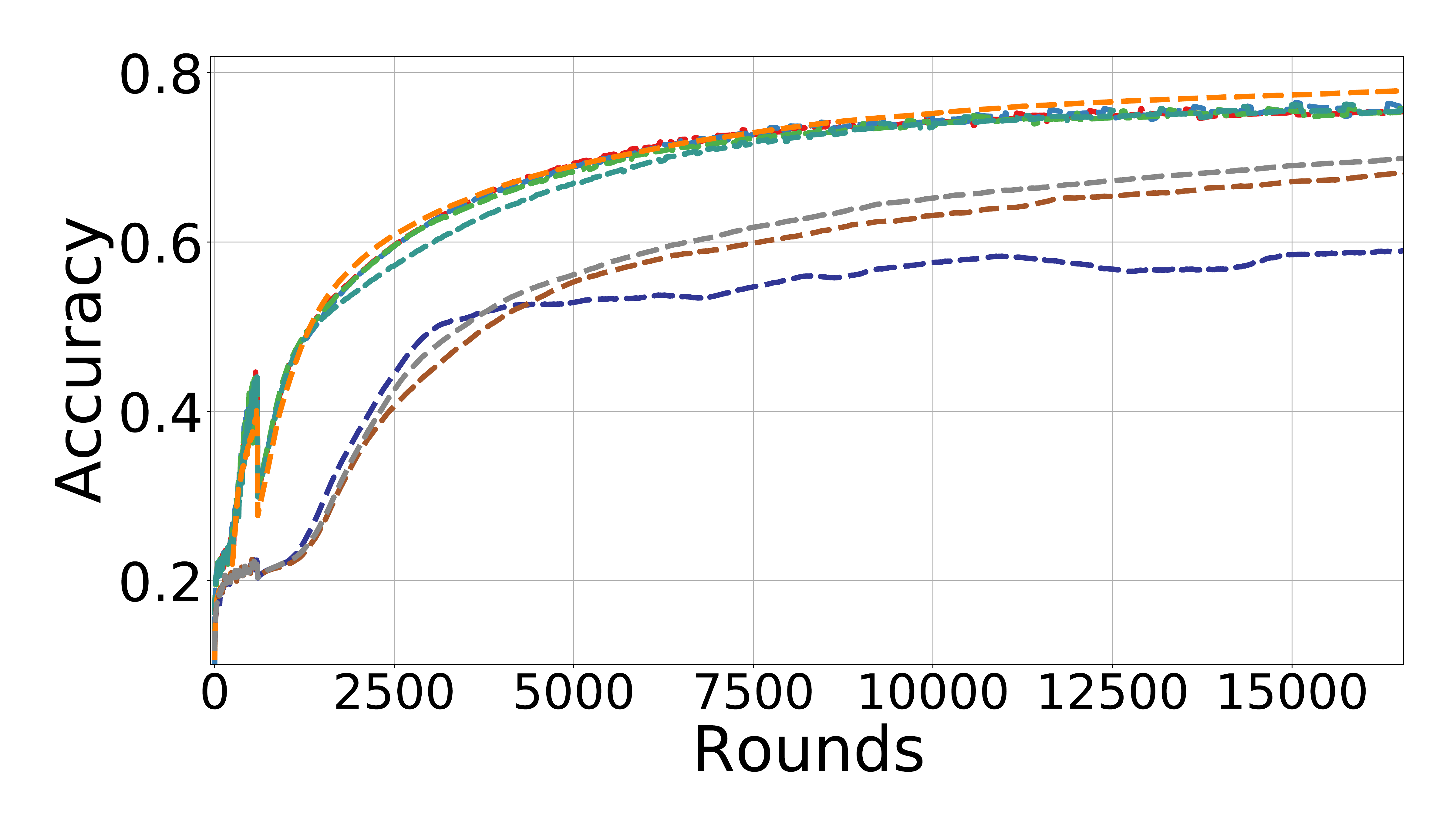} \caption{}
	\end{subfigure}
	\begin{subfigure}[ht]{0.245\textwidth}
		\includegraphics[width=\textwidth]{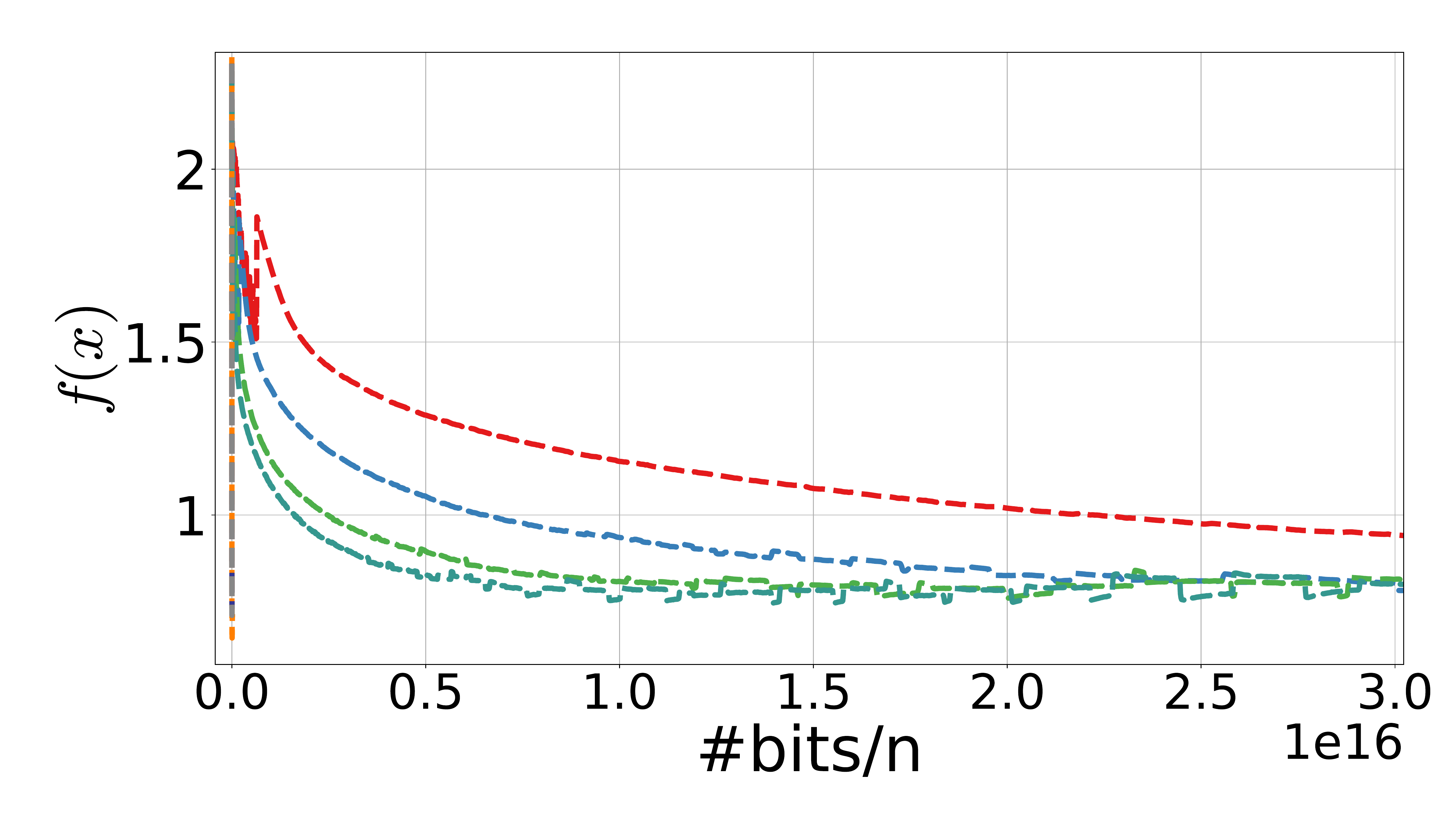} \caption{}
	\end{subfigure}
	\begin{subfigure}[ht]{0.245\textwidth}
		\includegraphics[width=\textwidth]{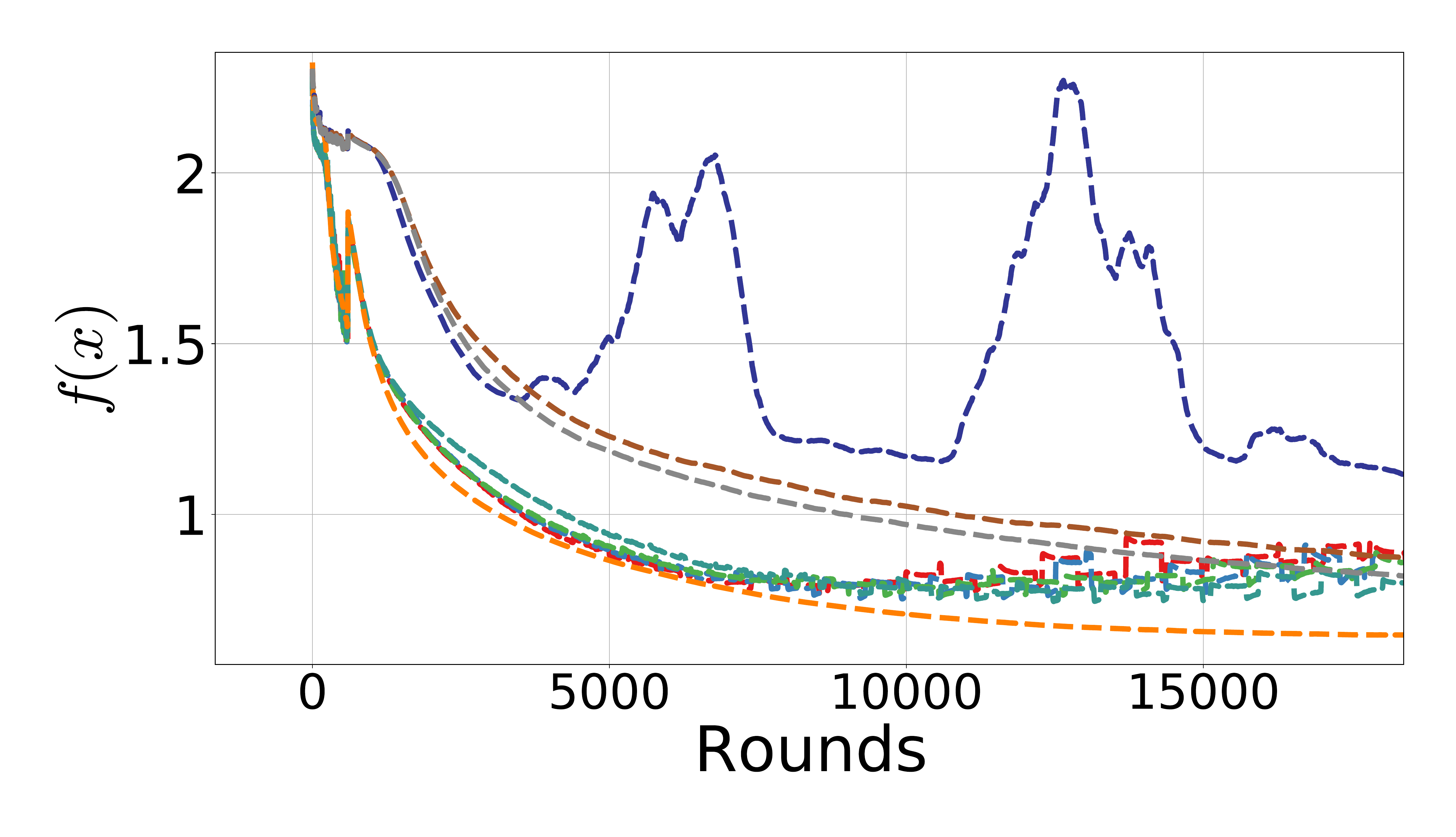} \caption{}
	\end{subfigure}
	\begin{subfigure}[ht]{0.245\textwidth}
		\includegraphics[width=\textwidth]{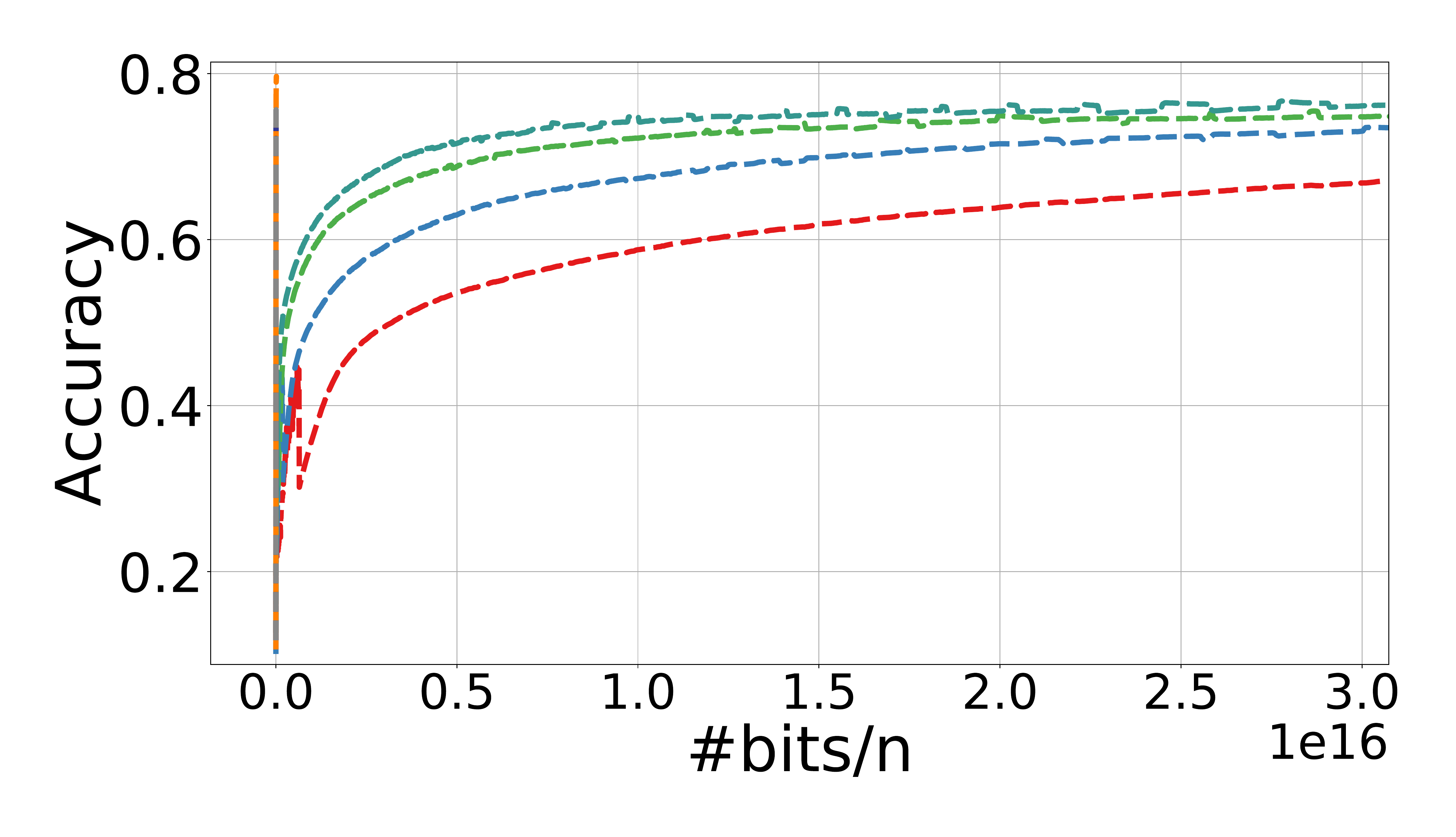}\caption{}
	\end{subfigure}
	
	\caption{\small{Training \texttt{ResNet-18} on \texttt{CIFAR-10} with $n=10$ workers. The top row represents the Top-1 accuracy vs. rounds in (a), loss functional value vs. communicated bits in (b), loss functional value vs. rounds in (c), and Top-1 accuracy vs. communicated bits in (d) on the train set. The bottom row presents the similar plots on the Test set in (e)--(h).}}
	\label{fig:training_resnet}
\end{figure*}

\begin{figure*}[t]
	\centering
	\captionsetup[sub]{font=scriptsize,labelfont={}}	
	\begin{subfigure}[ht]{0.9\textwidth}
		\includegraphics[width=\textwidth]{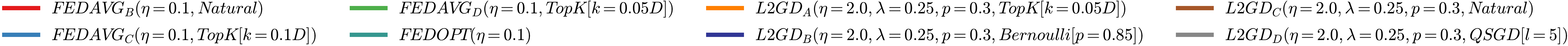}
	\end{subfigure}
	
	\begin{subfigure}[ht]{0.245\textwidth}
		\includegraphics[width=\textwidth]{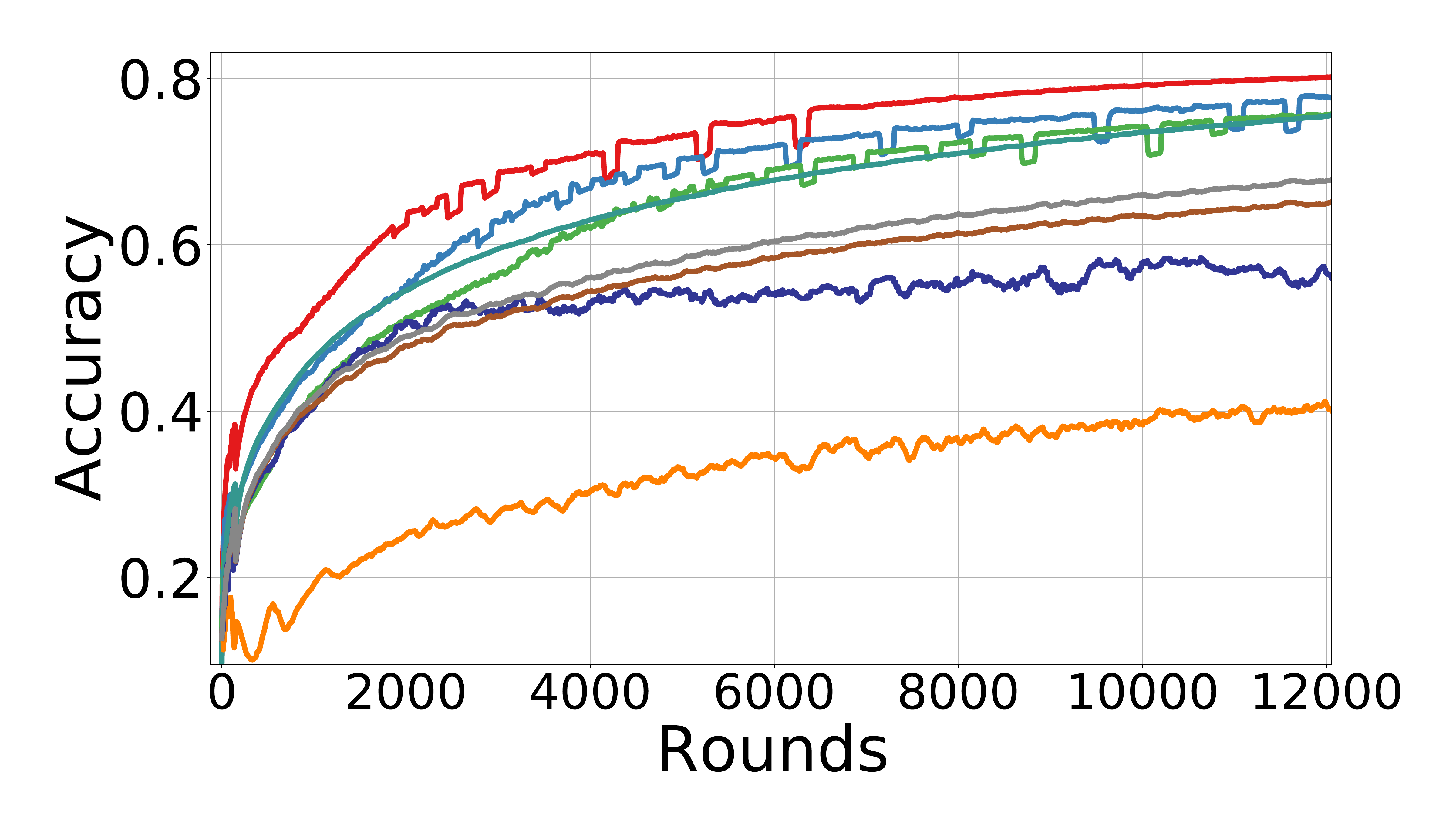} \caption{}
	\end{subfigure}
	\begin{subfigure}[ht]{0.245\textwidth}
		\includegraphics[width=\textwidth]{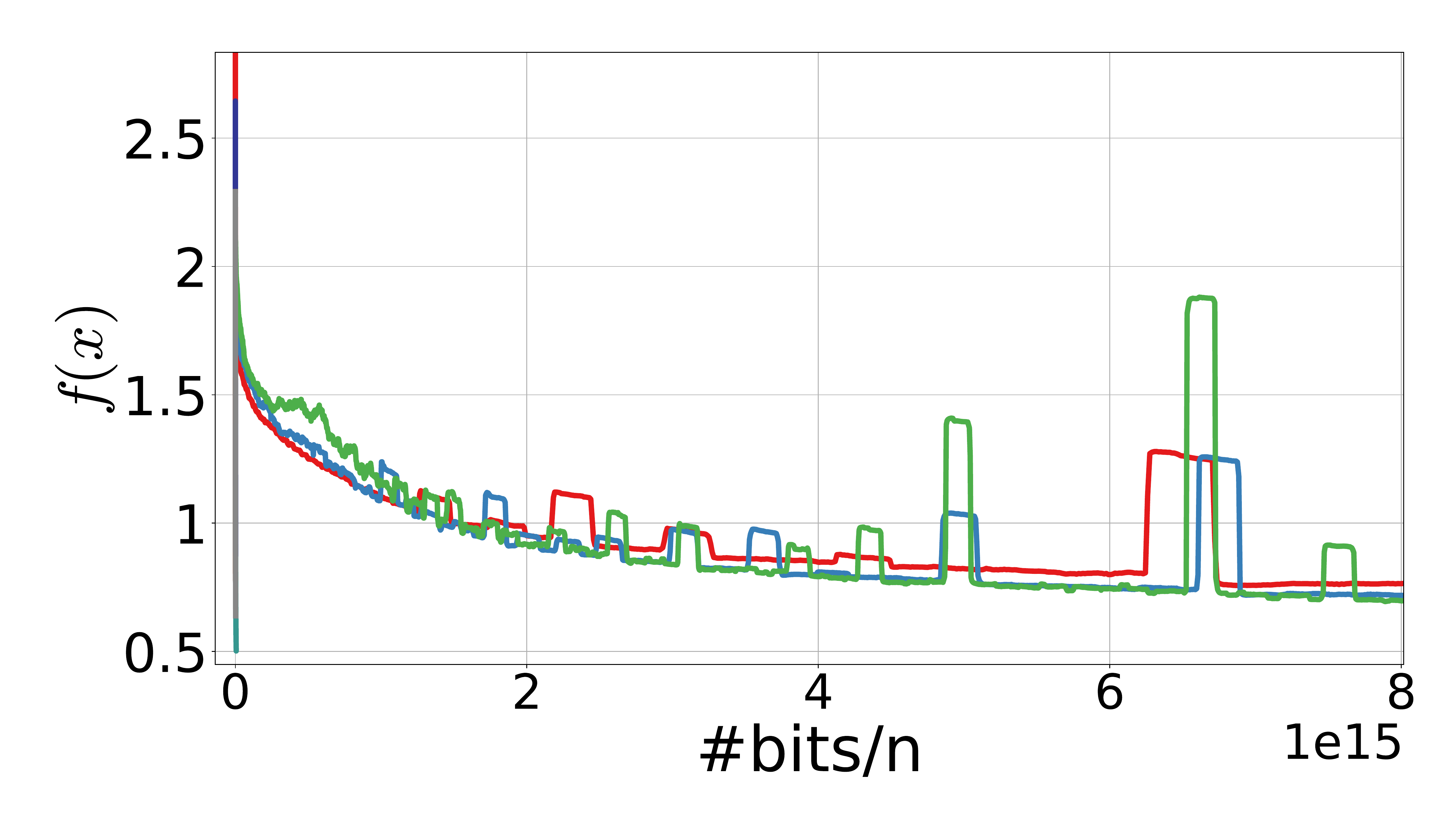} \caption{}
	\end{subfigure}
	\begin{subfigure}[ht]{0.245\textwidth}
		\includegraphics[width=\textwidth]{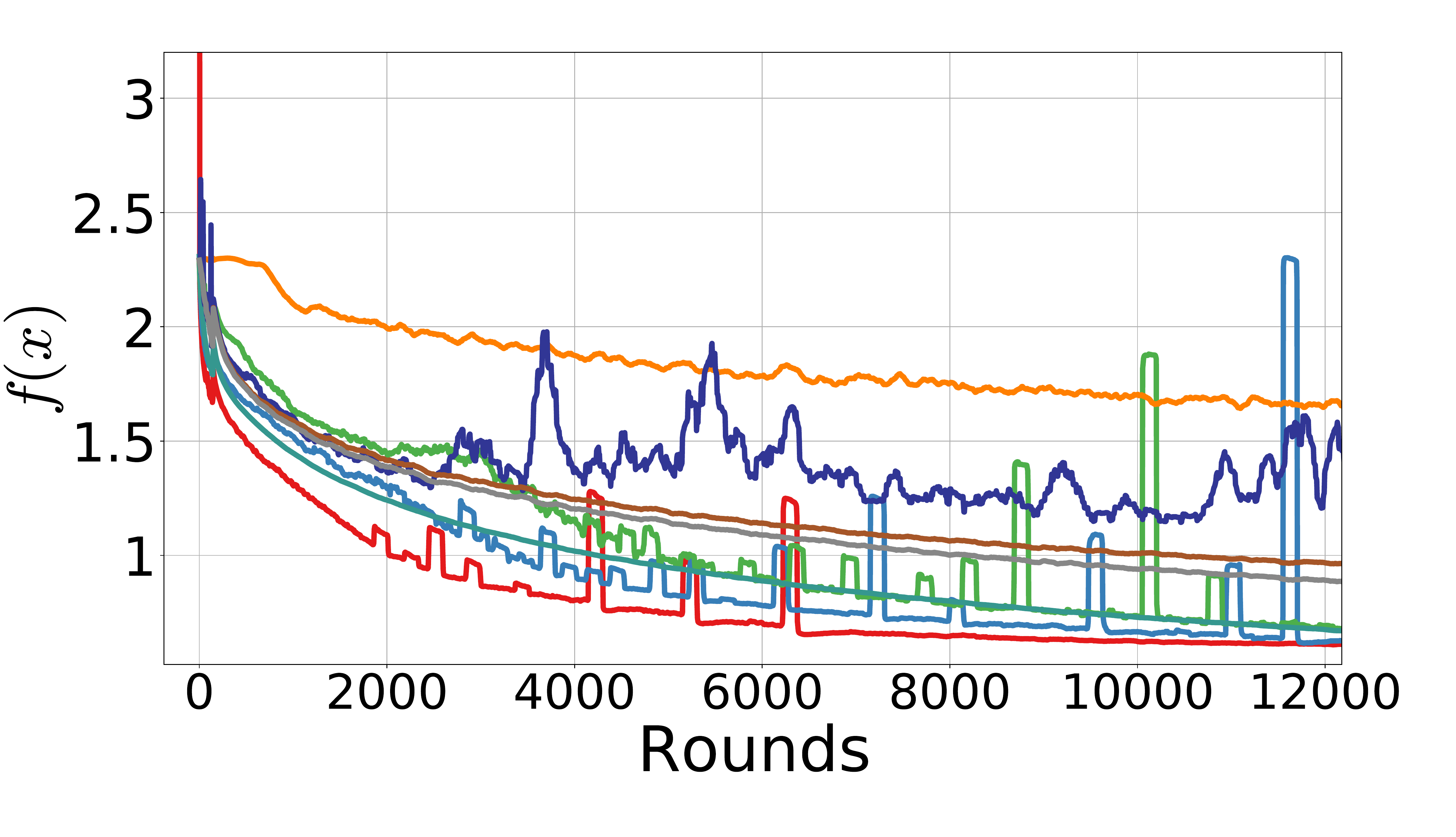} \caption{}
	\end{subfigure}
	\begin{subfigure}[ht]{0.245\textwidth}
		\includegraphics[width=\textwidth]{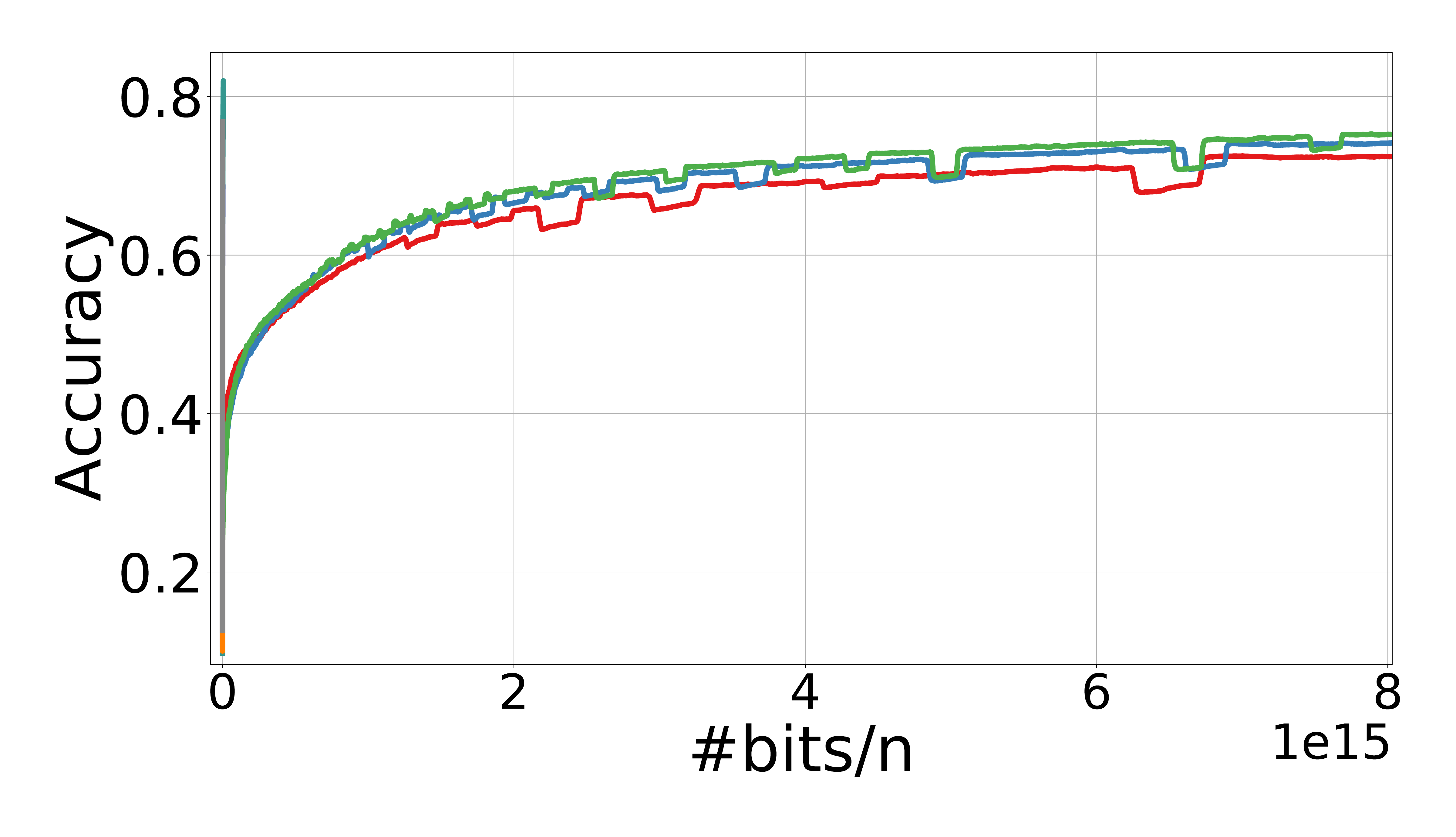}\caption{}
	\end{subfigure}
	
	\begin{subfigure}[ht]{0.245\textwidth}
		\includegraphics[width=\textwidth]{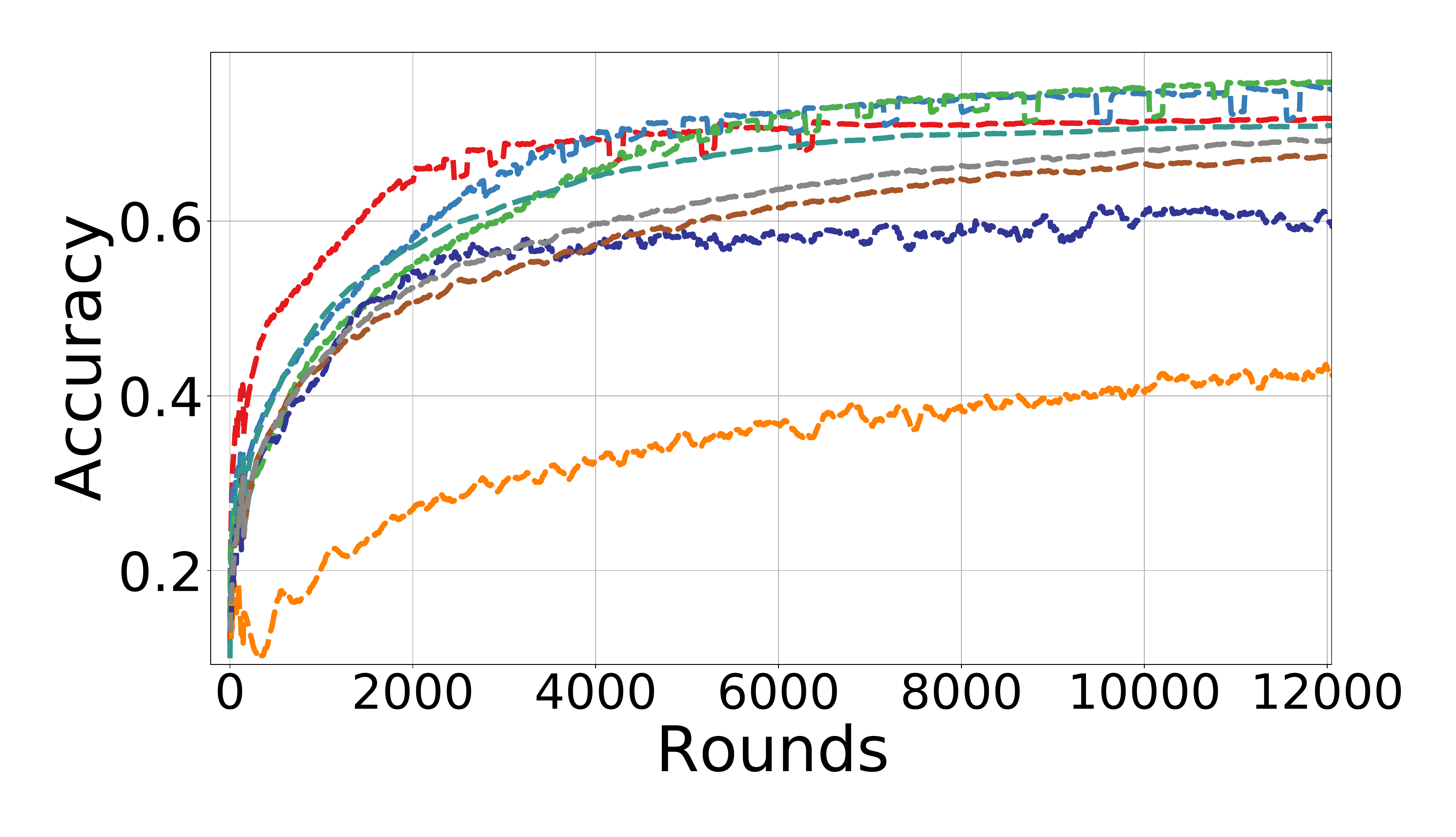} \caption{}
	\end{subfigure}
	\begin{subfigure}[ht]{0.245\textwidth}
		\includegraphics[width=\textwidth]{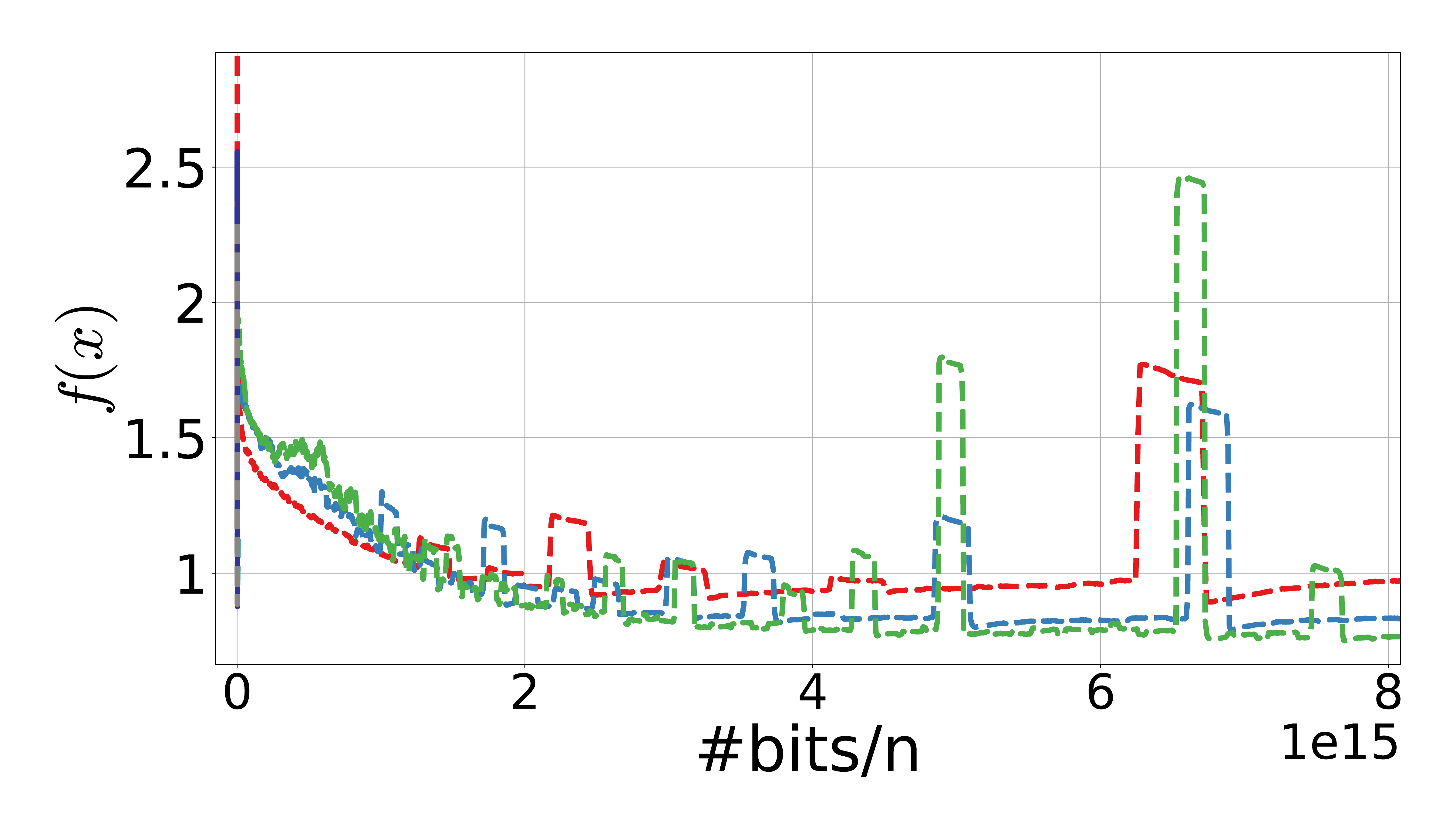} \caption{}
	\end{subfigure}
	\begin{subfigure}[ht]{0.245\textwidth}
		\includegraphics[width=\textwidth]{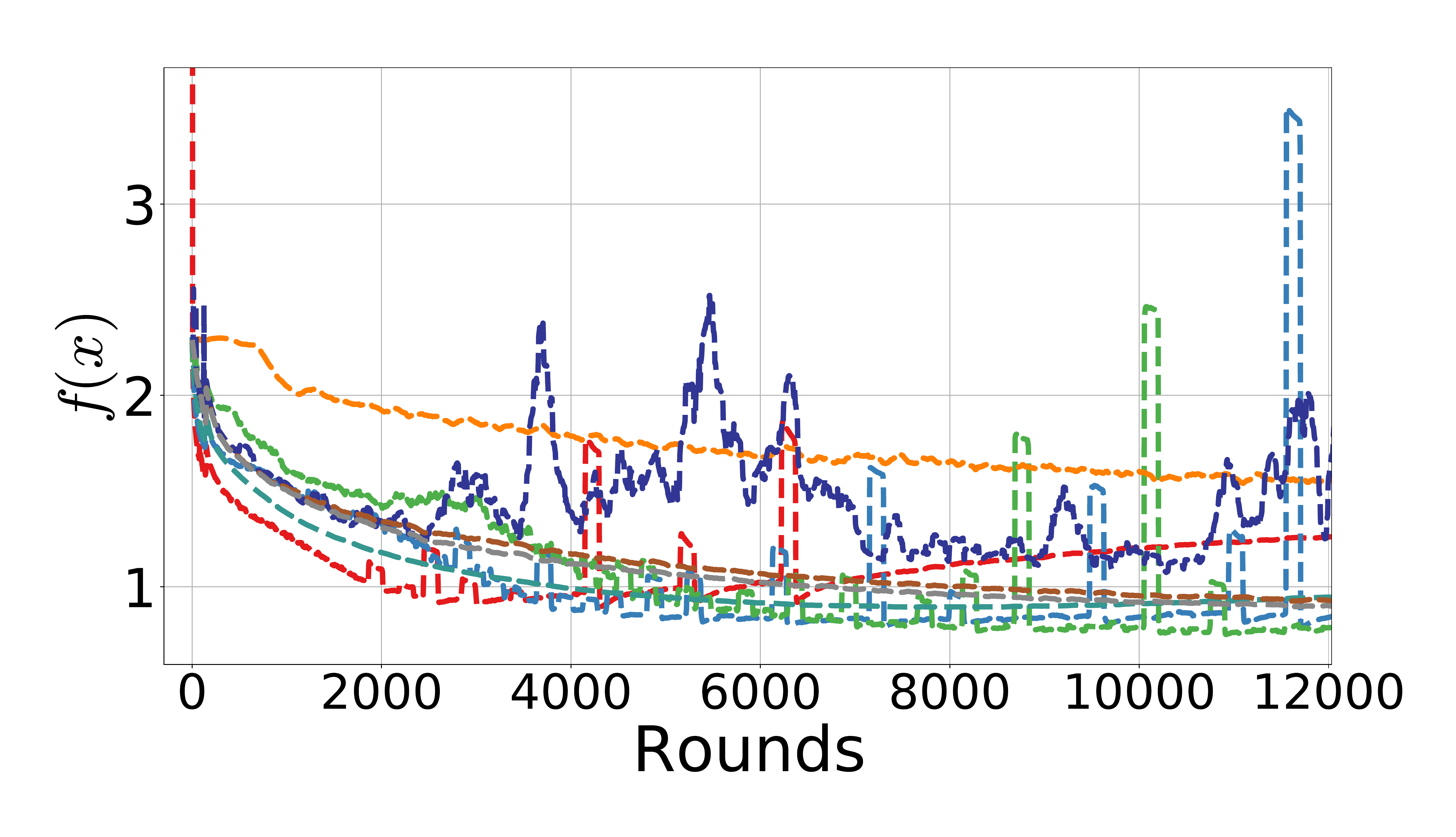} \caption{}
	\end{subfigure}
	\begin{subfigure}[ht]{0.245\textwidth}
		\includegraphics[width=\textwidth]{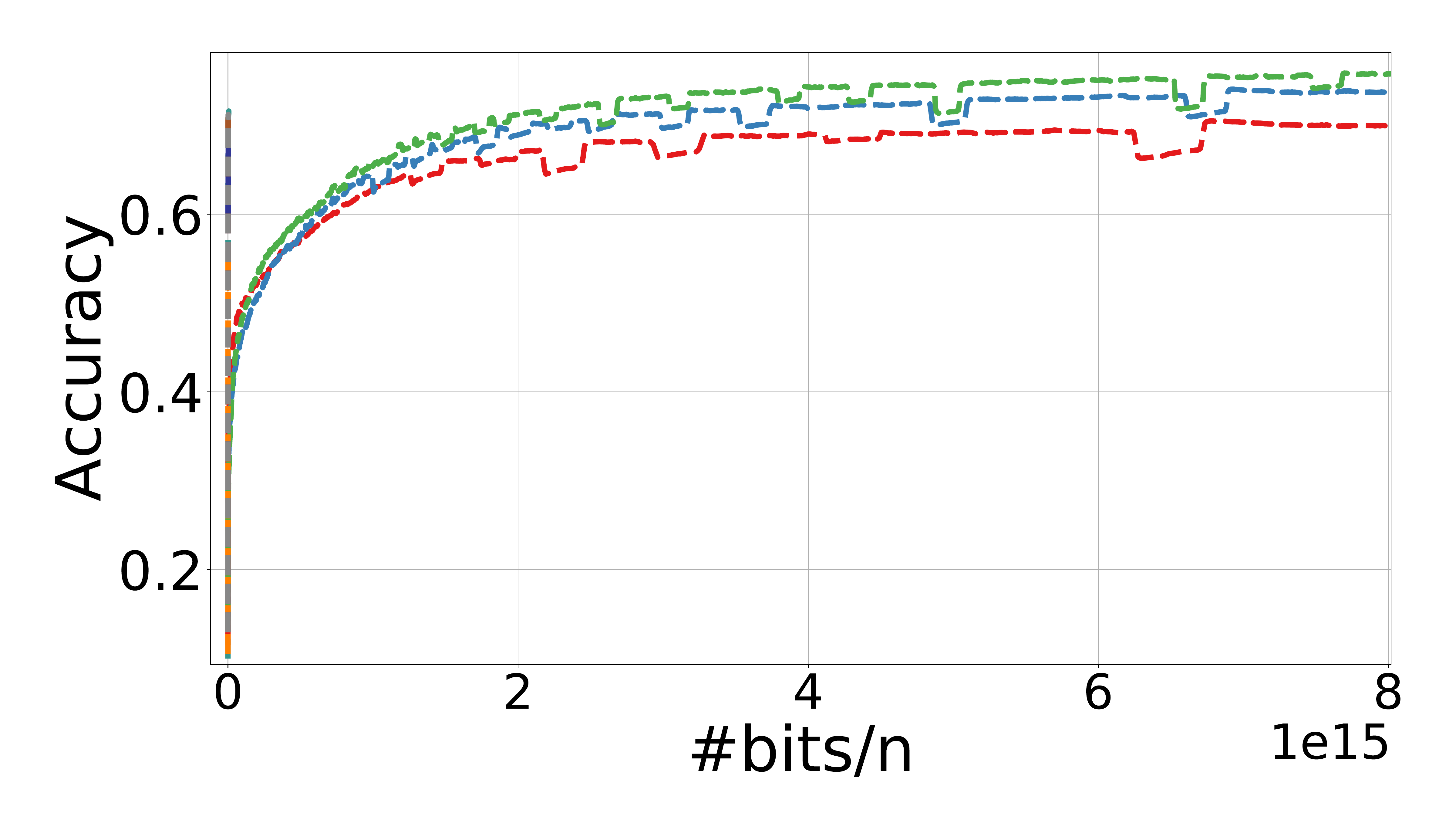}\caption{}
	\end{subfigure}
	
	\caption{\small{Training \texttt{DenseNet-121} on \texttt{CIFAR-10} with $n=10$ workers. The top row represents the Top-1 accuracy vs. rounds in (a), loss functional value vs. communicated bits in (b), loss functional value vs. rounds in (c), and Top-1 accuracy vs. communicated bits in (d) on the train set. The bottom row presents the similar quantities on the Test set in (e)--(h).}}
	\label{fig:training_densenet}
\end{figure*}


\begin{figure*}[t]
	\centering
	\captionsetup[sub]{font=scriptsize,labelfont={}}	
	\begin{subfigure}[ht]{0.9\textwidth}
		\includegraphics[width=\textwidth]{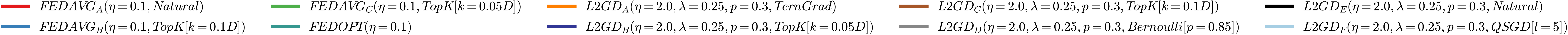}
	\end{subfigure}
	
	\begin{subfigure}[ht]{0.245\textwidth}
		\includegraphics[width=\textwidth]{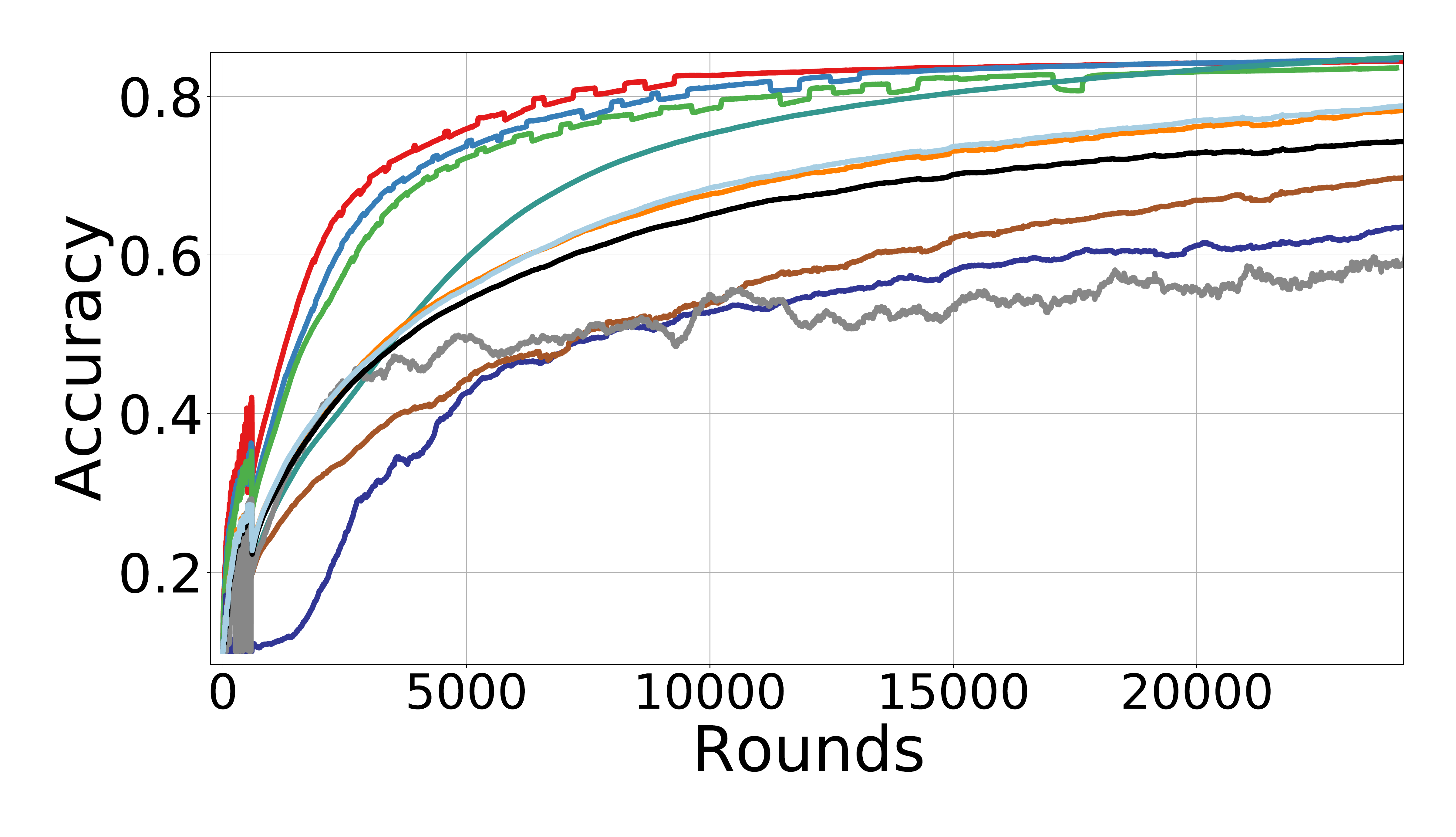} \caption{}
	\end{subfigure}
	\begin{subfigure}[ht]{0.245\textwidth}
		\includegraphics[width=\textwidth]{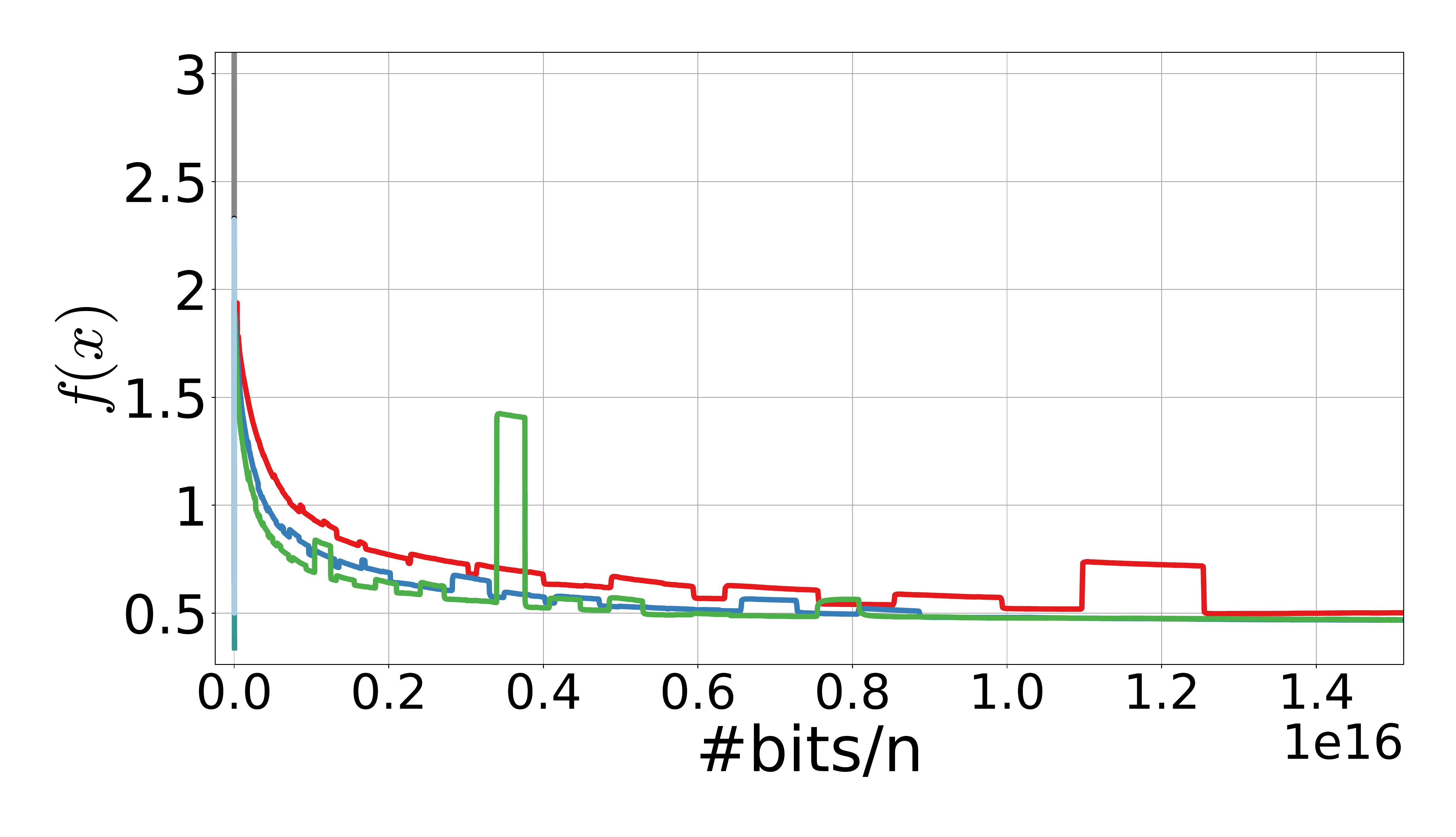} \caption{}
	\end{subfigure}
	\begin{subfigure}[ht]{0.245\textwidth}
		\includegraphics[width=\textwidth]{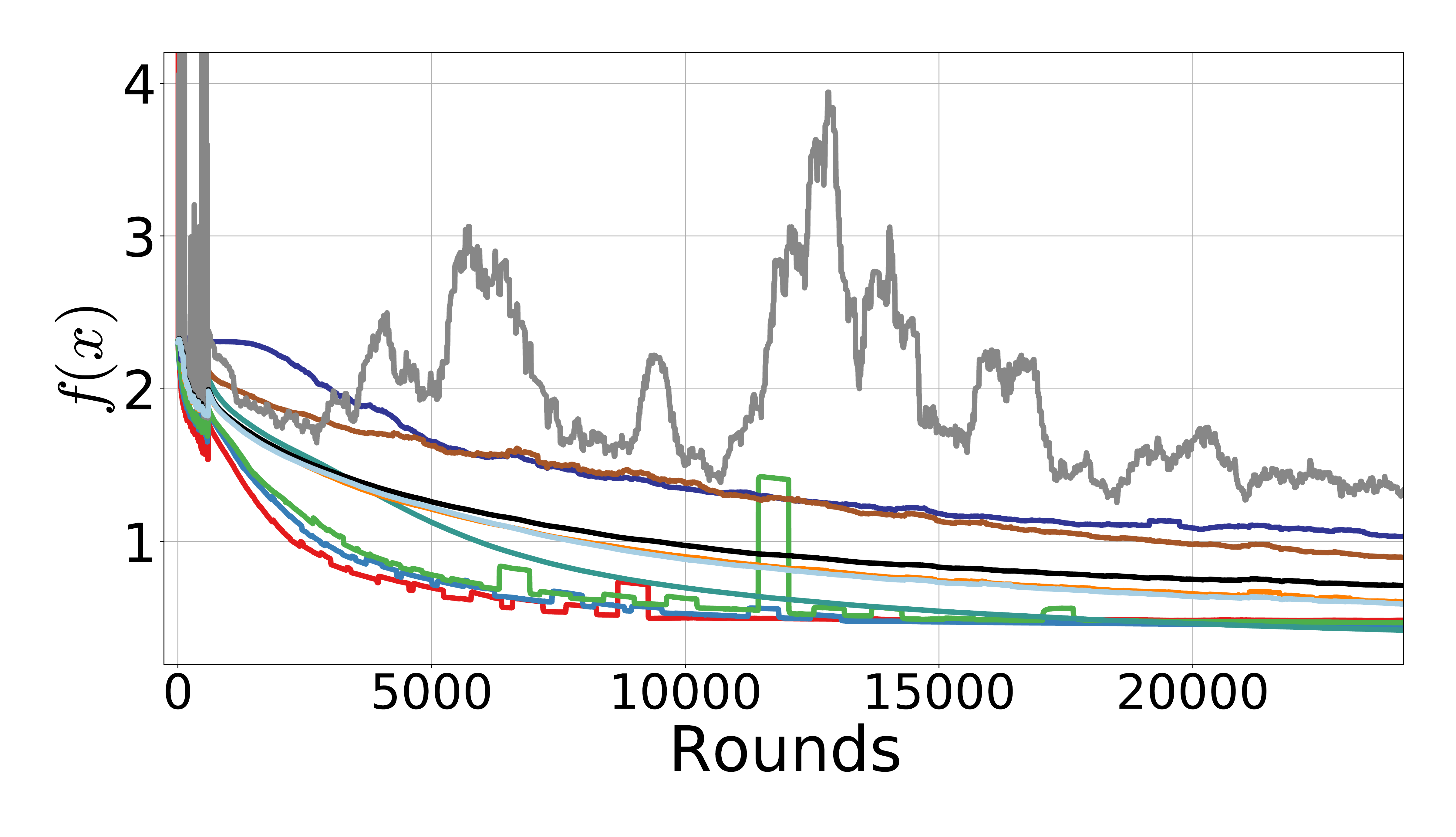} \caption{}
	\end{subfigure}
	\begin{subfigure}[ht]{0.245\textwidth}
		\includegraphics[width=\textwidth]{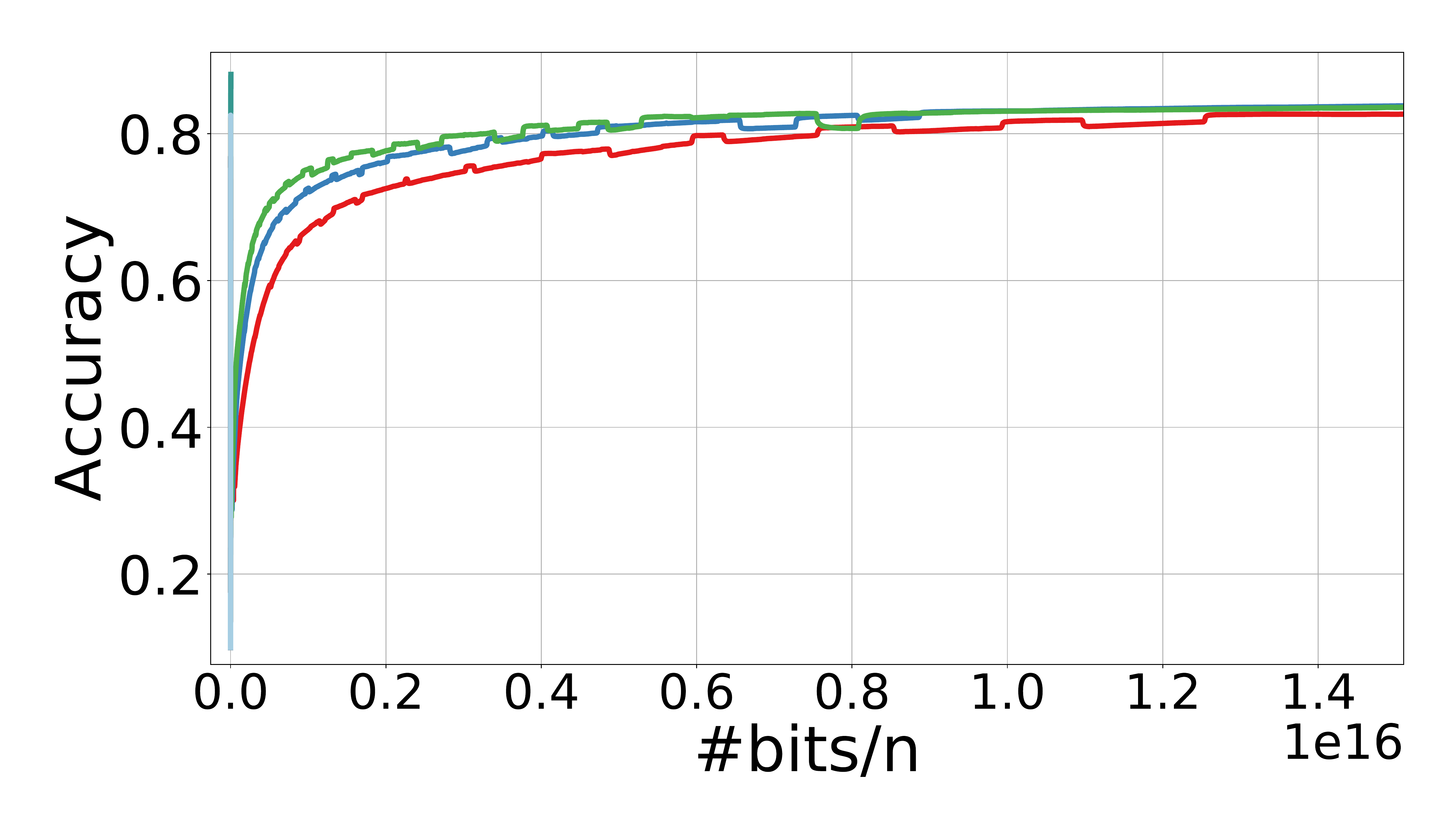}\caption{}
	\end{subfigure}
	
	\begin{subfigure}[ht]{0.245\textwidth}
		\includegraphics[width=\textwidth]{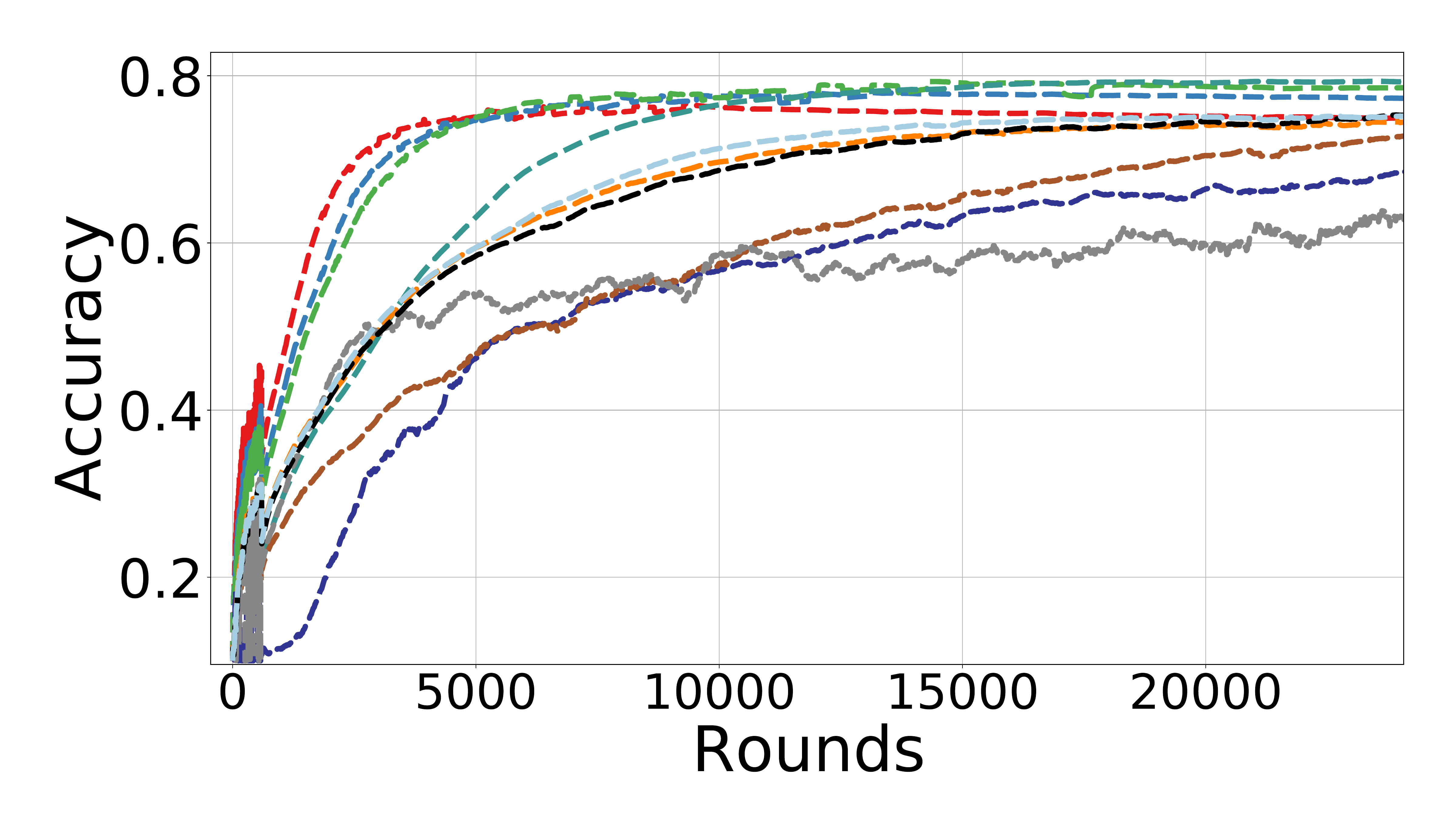} \caption{}
	\end{subfigure}
	\begin{subfigure}[ht]{0.245\textwidth}
		\includegraphics[width=\textwidth]{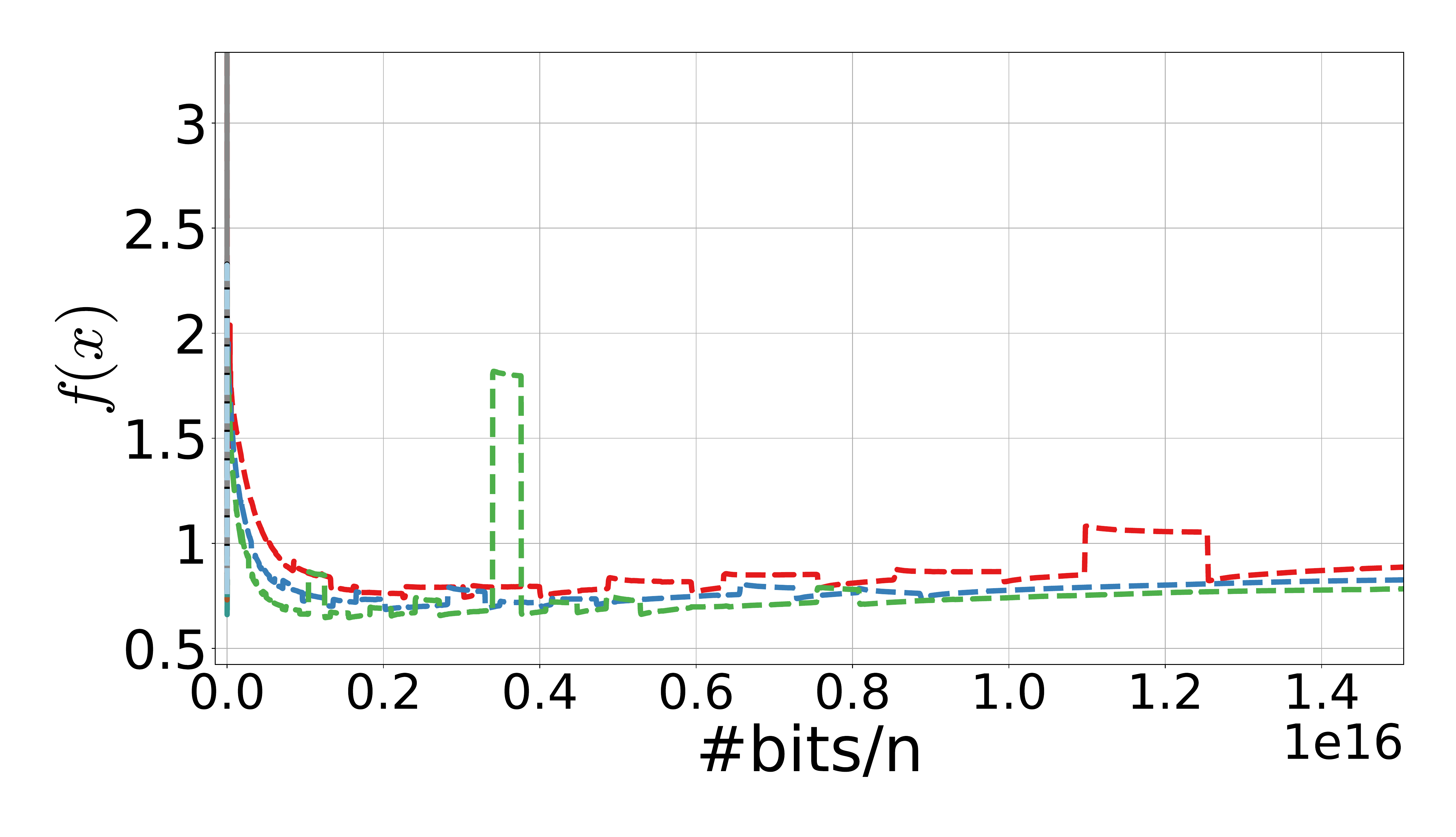} \caption{}
	\end{subfigure}
	\begin{subfigure}[ht]{0.245\textwidth}
		\includegraphics[width=\textwidth]{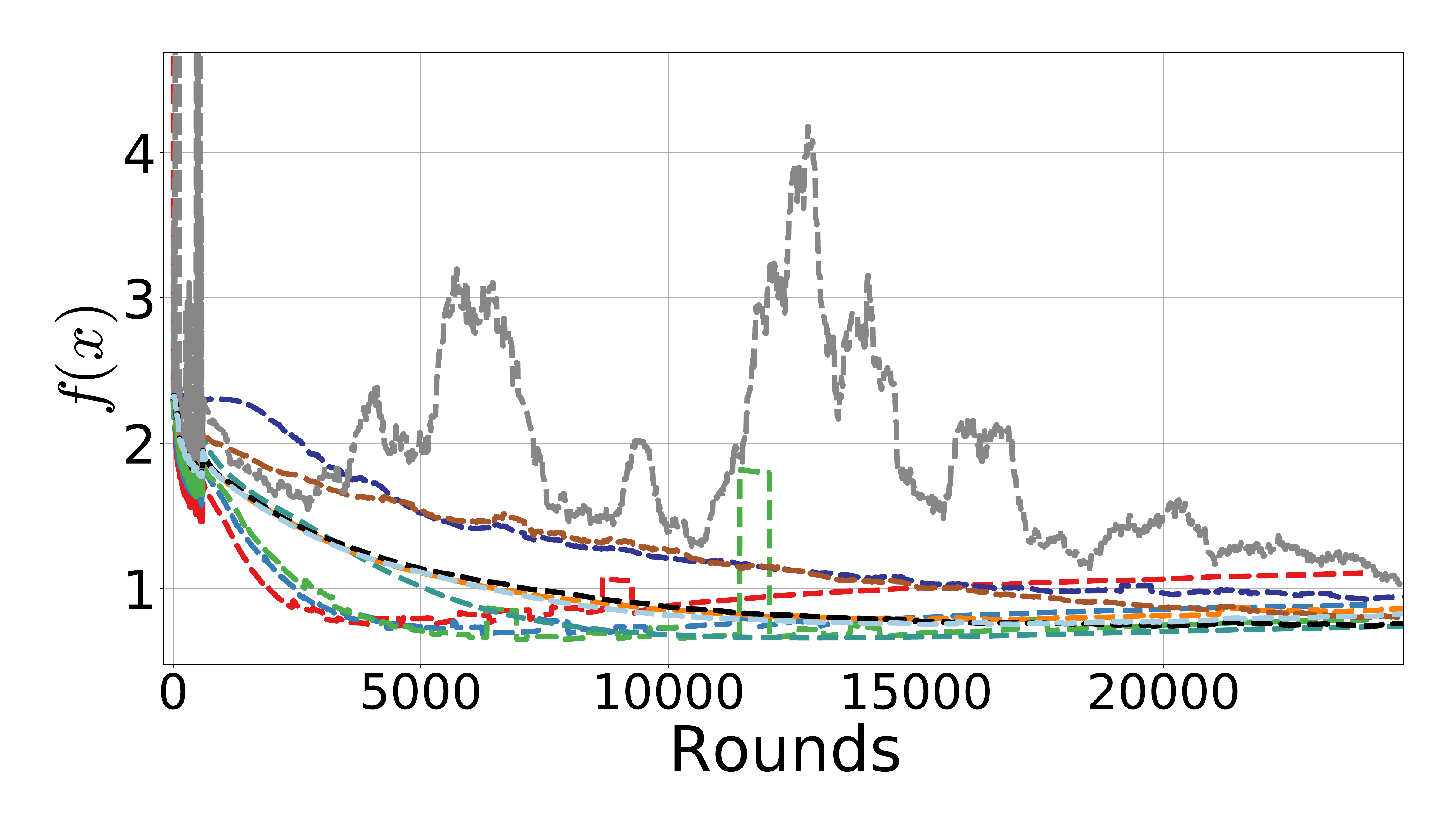} \caption{}
	\end{subfigure}
	\begin{subfigure}[ht]{0.245\textwidth}
		\includegraphics[width=\textwidth]{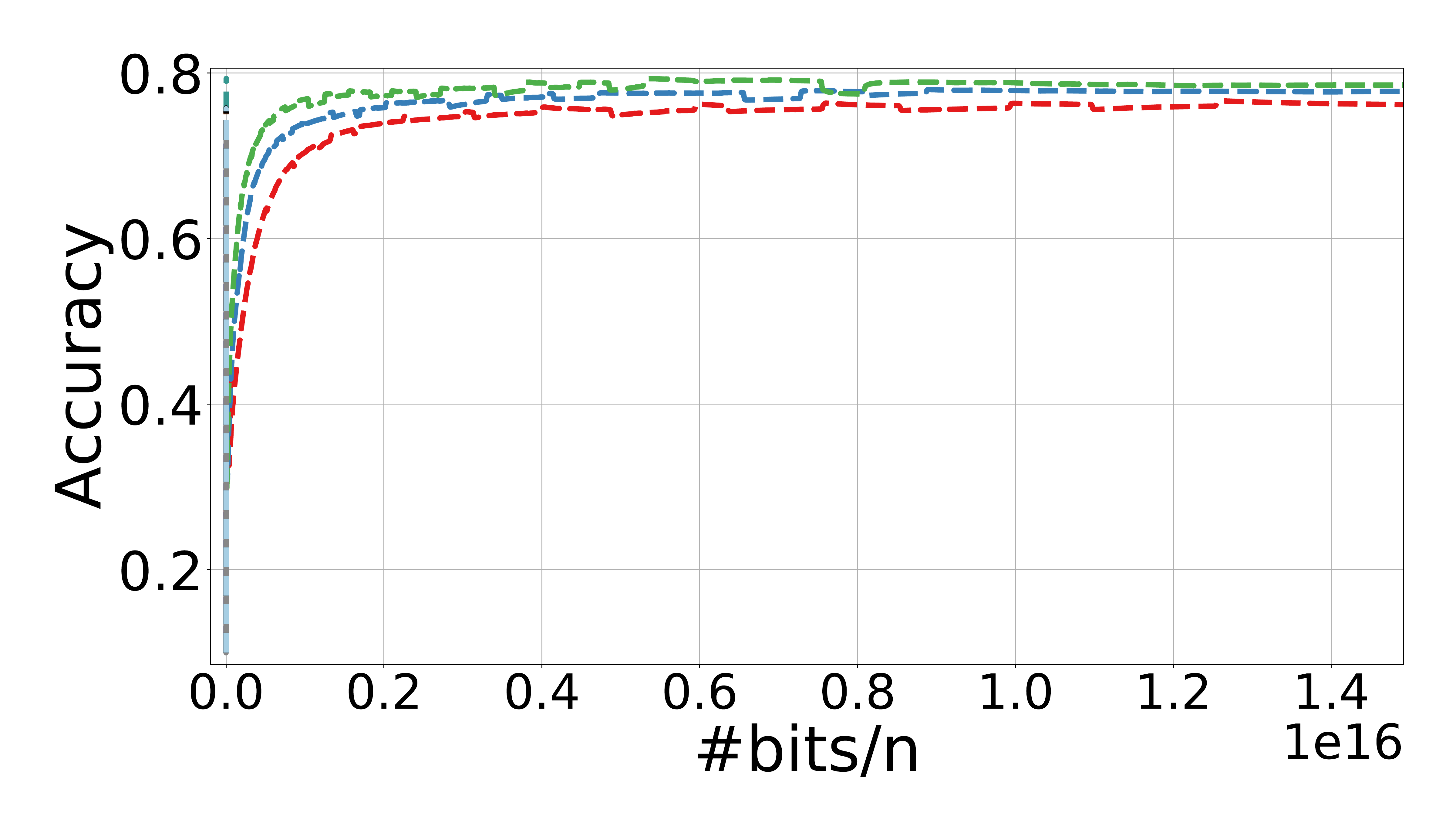}\caption{}
	\end{subfigure}
	
	\caption{\small{Training \texttt{MobileNet} on \texttt{CIFAR-10} with $n=10$ workers. The top row represents the Top-1 accuracy vs. rounds in (a), loss functional value vs. communicated bits in (b), loss functional value vs. rounds in (c), and Top-1 accuracy vs. communicated bits in (d) on the train set. The bottom row presents the similar plots on the Test set in (e)--(h).}}
	\label{fig:training_mobilenet}
\end{figure*}

\section{Empirical study}\label{sec:empirical}

We conducted diverse numerical experiments with L2GD algorithm that includes:~\myNum{i}~Analysis of algorithm meta-parameters for logistic regression in strongly convex setting; see \S\ref{sec:meta};~\myNum{ii}~analysis of compressed L2GD algorithm on image classification with DNNs; see \S\ref{sec:dnn}. 

\smartparagraph{Computing environment.} We performed experiments on server-grade machines running Ubuntu 18.04 and Linux Kernel v5.4.0, equipped with 8-cores 3.3 GHz Intel Xeon and a single NVIDIA GeForce RTX 2080 Ti.Tesla-V100-SXM2 GPU with 32GB of GPU memory. The computation backend for Logistics Regression experiments was NumPy library with leveraging MPI4PY for inter-node communication. For DNNs we used recent version of FedML \cite{he2020fedml} benchmark\footnote{{FedML.AI}} and patched it with: \myNum{i} distributed and standalone version of Algorithm \ref{alg:ComL2GD}; \myNum{ii} serializing and plotting mechanism; \myNum{iii} modifications in standalone, distributed version of FedAvg~\cite{mcmahan17fedavg} and FedOpt~\cite{reddi2020adaptive} to be consistent with \eqref{eq:problem}; \myNum{iv} not to drop the last batch while processing the dataset.

\subsection{Meta-parameter study}\label{sec:meta}
The purpose of these experiments is to study the meta-parameters involved in uncompressed L2GD algorithm. We used {L2GD} algorithm without compression for solving $\ell_2$ regularized logistic regression on LIBSVM {\tt a1a} and {\tt a2a} datasets \cite{libsvm}. Both datasets contain shuffled examples in the train set, and we did not perform any extra shuffling. To simulate the FL settings, we divided both datasets into $5$ parts. After splitting, each worker has $321$ and $453$ records for {\tt a1a}, and {\tt a2a}, respectively.

\smartparagraph{Setup and results.}
We define $f_i(x)$ to be local empirical risk minimization for logistic loss with additive regularization term for local data $D_i$ and of the form:
\[f_i(x)=\dfrac{1}{n_i} \sum_{j=1}^{n_i}\log(1+\exp(-b^{(j)} x^{\top} a^{(j)})) +  \dfrac{L_2}{2}\|x\|^2,\]
where $a^{(j)} \in \mathbb{R}^{124}, b^{(j)} \in \{+1,-1\},n_i=|D_i|$. 
We set $L_2=0.01$, and varied meta-parameters  $p$ and $\lambda$. For each parameter, we performed $100$ iterations of Algorithm \ref{alg:ComL2GD}. Note that, as meta-parameter $\lambda$ decreases the models will fit more to its local data, while $p$ provides stochastic balance between local gradient steps with probability, $1-p$ and aggregation with probability, $p$. 


\begin{figure*}
	\begin{minipage}[t]{0.35\textwidth}
		\centering
		\vspace{-27ex}	
		\resizebox{1.1\textwidth}{!}
		{
			\begin{tabular}{cccc}
				\toprule
				Model & \makecell{Training \\ parameters} & \makecell{L2GD \\ ${\rm bits}/n$} & \makecell{Baseline \\ ${\rm bits}/n$}   \\\hline
				\\
				DenseNet-121 &	$~79\times 10^5$ & $8\times 10^{11}$ & $4\cdot10^{15}$ \\ 
				
				MobileNet &	$~32\times 10^5$ & $1.7\times 10^{11}$ & $1\times 10^{15}$ \\ 
				ResNet-18 & $~11\times 10^6$ & $1.1 \times 10^{12}$  & $1.5 \times 10^{16}$\\
				\hline
			\end{tabular}
			}
		\vspace{12pt}
		\captionof{table}{Summary of the benchmarks. The measured quantity is ${\rm bits}/n$ to achieve $0.7$ Top-1 test accuracy, with $n=10$ clients. For DenseNet-121, MobileNet, Resnet-18 the baseline is FedAvg with natural compressor with $1$ local epoch; L2GD also uses natural compressor.}  
		\label{tab:modelsize}
	\end{minipage}\hfill
	\begin{minipage}[b]{0.3\textwidth}
		\includegraphics[width=\textwidth]{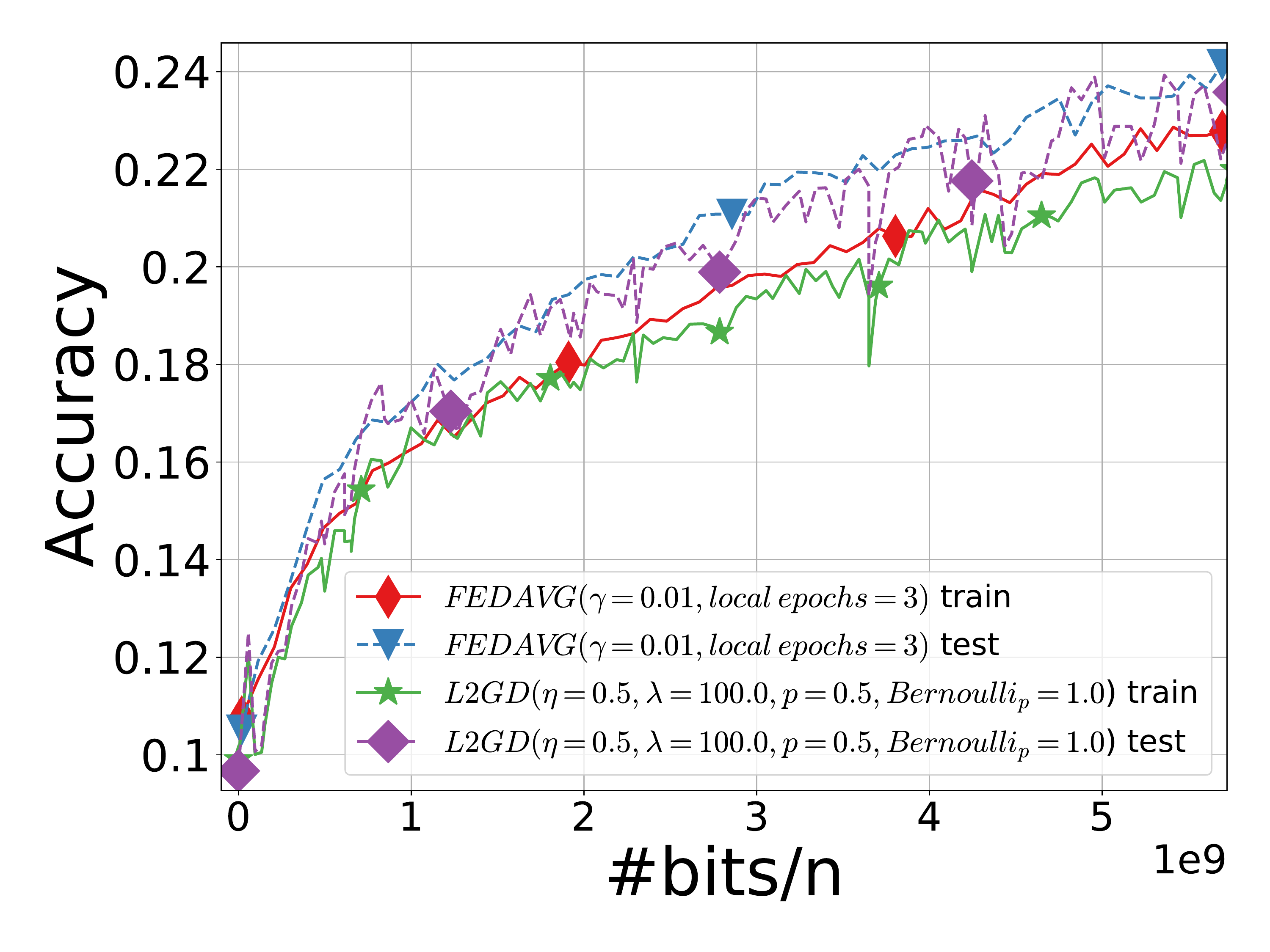}
		\captionof{figure}{FedAvg as a particular case of L2GD: Test and train accuracy for ResNet-56 on CIFAR-10.}
		\label{fig:nn_experiments_fedavg_like_acc}
	\end{minipage}\hfill
	\begin{minipage}[b]{0.3\textwidth}
		\includegraphics[width=\textwidth]{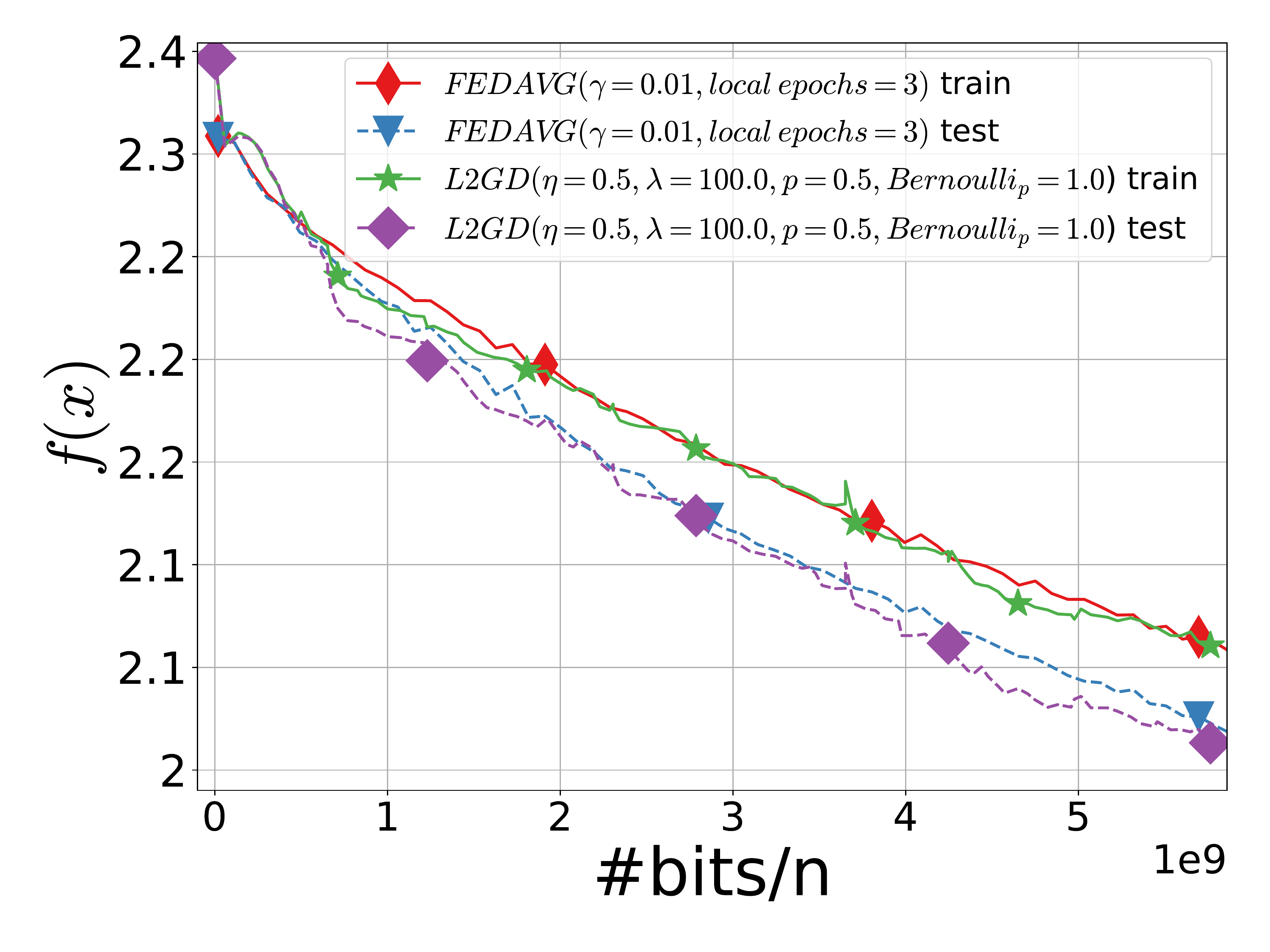}
		\captionof{figure}{FedAvg as a particular case of L2GD: Test and train loss for ResNet-56 on \\CIFAR-10.}
		\label{fig:nn_experiments_fedavg_like_f}
	\end{minipage}\hfill
	\caption*{}
 \end{figure*}



\smartparagraph{Takeaway message.} The results in Figure~\ref{fig:lambda_and_p_selection} support the theoretical finding---there exists an optimal choice of $(p, \lambda$), where the loss function, $f$ achieves the least value. Nevertheless, this choice is problem dependent. Additionally, we find small $p$ is not good due to lack of samples in a single node compared to samples available at other nodes. There is a trade-off for each node in learning from the other nodes' data and spending time to learn from its own data. In the experiments, the ``optimal" setting of our algorithm is attained for $p=0.4$ and $\lambda$ in $[0,25]$.
Finally, we observe that \emph{to get the smallest errors on the training and validation sets, it is better not to  perform the {averaging step} too often}. 

\subsection{Training DNN models}\label{sec:dnn}
We choose three practically important DNN models used for image classification, and other down-streaming tasks, such as feature extractions for image segmentation, object detection, image embedding, image captioning, to name a few. 
\begin{itemize}
	\item \texttt{ResNet-18}~\cite{resnet18_56}. The overwhelmingly popular \texttt{ResNet} architecture exploits residual connections to remedy vanishing gradients. The network supported the trend toward smaller filters and deeper architectures, more curated towards FL training. Additionally, we use \texttt{ResNet-56}~\cite{resnet18_56}. 
	
	\item \texttt{DenseNet}~\cite{densenet} contains a short connection between layers via connecting each layer to every other layer in a feed-forward fashion. Dense connection allows propagating information to the final classifier via concatenating all feature maps. Each layer in \texttt{DenseNet} is narrow and contains only 12 filters---another practical model for FL training. 
	\item \texttt{MobileNet}~\cite{howard2017mobilenets}. DNN architecture has a trade off between computational complexity and accuracy \cite[p.3, Fig.1]{bianco2018benchmark}. For mobile devices that appear in cross-device FL, the computation cost and energy consumption are both important. The energy consumption is mostly driven by memory movement \cite{chen2018understanding}, \cite{horowitz20141}. In \texttt{MobileNet} architecture standard convolution blocks performs depth-wise convolution followed by $1\times 1$ convolution. This is computationally less expensive in flops during inference time (see \cite[Fig.1, p.3]{bianco2018benchmark}) and is $\sim3.5\times$ more power efficient compare to \texttt{DenseNet} \cite[p.85, Table 7]{garcia2019estimation}. This makes \texttt{MobileNet} an attractive model for FL training. 
\end{itemize}

\smartparagraph{Dataset and Setup.}  
We consider \texttt{CIFAR-10} dataset \cite{krizhevsky2009learning} for image classification. It contains color images of resolution $28\times28$ from 10 classes. The training and the test set are of size, $5\times 10^4$ and $10^4$, respectively. The training set is partitioned heterogeneously across $10$ clients. The proportion of samples of each class stored at each local node is drawn by using the Dirichlet distribution ($\alpha=0.5$). In our experiments, all clients are involved in each communication round. Additionally, we added a linear head in all CNN models for \texttt{CIFAR-10}, as they are originally designed for classification task with $1000$ output classes. 

\smartparagraph{Loss function.} Denote ${f_i(x)=w_i \cdot \frac{1}{|D_i|} \sum_{ (a_i, b_i) \in D_i }^{} l(a_i, b_i, x)}$ to be a weighted local empirical risk associated with the local data, $\mathcal{D}_i$ stored in node, $i$. We note that $l(a_i,b_i,x)$ is a standard unweighted cross-entropy loss, $a_i \in \mathbb{R}^{28 \times 28 \times 3}$, $b_i \in \{0,1\}^{10}$ with only one component equal to $1$, the ground truth value, and the weight is set to $w_i={|D_i|}/{|D_1 \cup \dots \cup D_n|}$.

\smartparagraph{Metrics.} To measure the performance, we examine the loss function value, $f(x)$, and the Top-1 accuracy of the global model on both train and the test set. Additionally, we measure the number of rounds, and ${\rm bits}/n$---communicated bits normalized by the number of local clients, $n$. The intuition behind using the last metric is to measure the commmunicated data-volume; it is widely {\em hypothesized} that the reduced data-volume translates to a faster training in a constant speed network in distributed setup \cite{gajjala2020huffman, grace}. 

\smartparagraph{Compressors used.} The theoretical results of compressed L2GD are attributed to unbiased compressors. We used 4 different unbiased compressors at the clients: Bernoulli \cite{khirirat2018distributed}, natural compressor \cite{cnat}, random dithering a.k.a. QSGD \cite{alistarh2017qsgd}, and Terngrad \cite{DBLP:conf/nips/WenXYWWCL17}; see Table \ref{tab:summary} for  details. Additionally, we note that biased compressors (mostly sparsifiers) are popular in DNN training. 
Therefore, out of scientific curiosity, we used a popular sparsifier: Top-$k$ \cite{aji_sparse, sahu2021rethinking} as a proof of concept. We note that extending the compressed L2GD theory for biased compressors (with or without error-feedback \cite{grace}) is nontrivial and mathematically involved, and left for future work.

\smartparagraph{Algorithms used for comparison.} We used state-of-the-art FL training algorithms, FedAvg \cite{mcmahan17fedavg} and FedOpt \cite{reddi2020adaptive} as no compression baseline to compare against our L2GD. However, the performance of FedAvg is not stable but improves with the compression mechanism. The original FedAvg algorithm does not contain any compression mechanism, but for comparison, we incorporated compressors into FedAvg via the following schema which is similar to the classic error feedback \cite{grace}: \myNum{i} After local steps, client estimates change of current iterate from the previous round and formulates direction, ${g_{c,\mathrm{computed}}}^{i}$; \myNum{ii} client sends compressed difference between previous gradient estimator from previous round and currently computed gradient estimator, $\mathcal{C}({g_{c,\mathrm{computed}}}^{i} - {g_c}^{i-1})$ to the master; \myNum{iii} both master and client updating ${g_c}^{i}$ via the following schema: ${g_c}^{i} = {g_c}^{i-1} + \mathcal{C}({g_{c,\mathrm{computed}}}^{i} - {g_c}^{i-1})$. We provide the details about step size and batch size in Appendix.

\subsubsection{Results}
We show the results for training \texttt{ResNet-18}, \texttt{DenseNet-121}, and \texttt{MobileNet} with compressed L2GD and other state-of-the-art FL algorithms in Figure \ref{fig:training_resnet}--\ref{fig:training_mobilenet}. For these experiments, the communication rounds are set to $12\times 10^3$, $25\times 10^3$, and $20\times 10^3$, respectively. For the FedAvg algorithm, each client performs one epoch over the local data. We empirically tried $1,2,3,$ and $4$ epochs over the local data as local steps, but one epoch is empirically the best choice. 

For training \texttt{ResNet-18}, from Figure \ref{fig:training_resnet}~we observe that FedAvg with compression has albeit better convergence than no compression FedAvg \footnote{ We have observed that batch normalization \cite{ioffe2015batch} in ResNet is sensitive for aggregation; see our discussion in \S\ref{sec:app_nn}.}. At the same time, compressed FedAvg affects the convergence as a function of communicated rounds only negligibly (see Figure \ref{fig:training_resnet}~(d),(b)). Therefore, for training other DNN models we use FedAvg with compression and FedOpt without any compressors to enjoy the best of both baselines. 

\smartparagraph{Take away message.} Compressed L2GD with natural compressor sends the least data and drives the loss down the most in these experiments. At the same time, L2GD with natural compressor (by design it has smaller variance) reaches the best accuracy for both train and test sets. Compressed L2GD outperforms FedAvg by a huge margin---For all DNN experiments, to reach the desired Top-1 test accuracy, compressed L2GD reduces the communicated data-volume, $\mathrm{\#bits/n}$, from $10^{15}$ to $10^{11}$, rendering approximately a $10^4$ times improvement compared to FedAvg; see Table \ref{tab:modelsize}.

Interestingly, in training \texttt{MobileNet}, the performance of biased Top-$k$ compressor degrades only about 10\% compared to natural compressor, while approximately degrades 35\% in training \texttt{DenseNet}. Additionally, see  discussion in \S\ref{sec:app_nn}, Figures 9--11. This phenomena may lead the researchers to design unbiased compressors with smaller variance to empirically harvest the best behavior of compressed L2GD in personalized FL training.

Nevertheless, we also observe that compressed L2GD converges slower compared to other FL algorithms without compression in all cases. What follows, it can be argued, is that when we compare the communicated data volume for all DNN models, the convergence of compressed L2GD is much better. Additionally, the gain in terms of lowering the loss function value is significant---\emph{by sending the same amount of data, L2GD lowers the loss the most compared to the other no-compression FL baseline algorithms.} These experiments also demonstrate that when communication is a bottleneck, FedAvg is not comparable with {L2GD}. The only comparable baseline for L2GD is FedOpt; see Table II, also, see  discussion in \S\ref{sec:app_nn}, Figures 9--11. A similar observation holds for the Top-1 test and train accuracy. Taken together, these indicate that for training larger DNN models in a personalized FL settings, with resource constrained and geographically remote devices, compressed L2GD could be the preferred algorithm because its probabilistic communication protocol sends less data but obtains better test accuracy than no compression FedAvg and FedOpt.

Additionally, we observe that when $\frac{\eta \lambda}{n p} \in [0.5, 0.95],$ compressed {L2GD} incurs a significant variance in objective function during training. Empirically, the best behavior was observed for $\frac{\eta \lambda}{n p} \approx 1$ or $\frac{\eta \lambda}{n p} \in (0, 0.17]$.

\smartparagraph{FedAvg as a particular case of L2GD.} 
We note that if ${\eta \lambda}/{n p}=1$, then the aggregation step of Algorithm \ref{alg:ComL2GD} reduces to $x_i^{k+1}=\bar{x}^k$, for all devices. Thus, in this regime L2GD works similarly as FedAvg with random number of local steps. E.g.,if $p=0.5$, then Algorithm \ref{alg:ComL2GD} reduces to randomized version of FedAvg with an average of $3$ local steps.  Figures~\ref{fig:nn_experiments_fedavg_like_acc} and \ref{fig:nn_experiments_fedavg_like_f}, confirm this observation numerically, where we see that both algorithms exhibit similar performance. In that experiment we trained \texttt{ResNet-56} on \texttt{CIFAR10} with $n=100$ workers, and  $600$ rounds. For \texttt{L2GD}, we set $\frac{\eta \lambda}{p n}=1$.

\section{Conclusion and Future Direction} 
 
In this paper, we equipped the loopless gradient descent (L2GD) algorithm with a compression mechanism to reduce the communication bottleneck between local devices and the server in an FL context. We showed that the new algorithm enjoys similar convergence properties as the uncompressed L2GD with a natural increase in the stochastic gradient variance due to compression. This phenomenon is similar to classical convergence bounds for compressed SGD algorithms. We also show that in a personalized FL setting, there is a trade-off that must be considered by devices between learning from other devices' data and spending time learning from their own data. However, a particular parameterization of our algorithm recovers the well-known FedAvg Algorithm. We assessed the performance of the new algorithm compared to the state-of-the-art and validated our theoretical insights through a large set of experiments.
 
Several questions remain open and merit further investigation in the future. For example, we plan on including compression when devices calculate their local updates, especially in an FL setting, as the devices might not be powerful, and the computing energy is limited, and examine how the algorithm behaves. Additionally, we observed the efficacy of compressed L2GD with a biased compressor, such as Top-$k$. Nevertheless, extending the compressed L2GD theory for biased compressors (with or without error-feedback \cite{grace}) is nontrivial and challenging. In the future, we plan to prove a more general theory for compressed L2GD that include both biased and unbiased compressor operating in a bidirectional fashion. A more detailed meta-parameter study covering different network bandwidths, diverse ML tasks with different DNN architectures, and deploying the models on real-life, geographically remote servers will be our future empirical quest.


%
\section{Acknowledgments}
Aritra Dutta acknowledges being an affiliated researcher at the Pioneer Centre for AI, Denmark. The authors acknowledge many fruitful discussions with Md. Patel on this project while he was a remote undergraduate intern at KAUST. 

\bibliographystyle{arxiv}
\bibliography{references.bib}  


\onecolumn
\appendix

\subsection{Convergence Analysis---Proofs of the Lemmas and the Theorems}\label{sec:Proofs}

In this section, we provide the proofs of convex and non-convex convergence results of the compressed L2GD algorithm. 

\smartparagraph{Overview of results.}
In \S\ref{sec:app convergence_technical_results}, we provide the technical lemmas necessary for the analyses. \S\ref{sec:app_main conv} contains the auxiliary results pertaining to both convex and nonconvex convergence. In \S\ref{sec:app_nnc} we provide the non-convex convergence results, and \S\ref{sec:app optimal} provides the proofs for optimal rate and communication.

\subsubsection{Technical results used for convergence}\label{sec:app convergence_technical_results}
The following two Lemmas are instrumental in proving other compression related results.

\begin{lemma*}
Let $x \in R^{nd}$, then
$$\textstyle{\E_{\mathcal{C}} \left[\|\mathcal{C}(x)\|^2\right]\le (1+\omega) \|x\|^2,}$$
where $\textstyle{\omega = \max_{i=1,\ldots,n} \{\omega_i\}.}$
\end{lemma*}
\begin{proof}
By using Assumption 1, we have
\begin{eqnarray*}
\E_{\mathcal{C}} \left[\|\mathcal{C}(x)\|^2\right] &=& \E_{\mathcal{C}} \left[\sum_{i=1}^n \|\mathcal{C}_i(x_i)\|^2\right] = \sum_{i=1}^n \E_{\mathcal{C}_i} \|\mathcal{C}_i(x_i)\|^2 \le  \sum_{i=1}^n (1+\omega_i) \|x_i\|^2 \le (1+\omega) \|x\|^2.
\end{eqnarray*}
Hence the result. 
\end{proof}

\begin{lemma*}
Let Assumption \ref{ass:compression} hold, then for all $\textstyle{k\ge 0}$, 
$\textstyle{\E_{\mathcal{C},\mathcal{C}_M} \left[\mathcal{C}_M(\Bar{y}^k)\right] = \Bar{x}^k.}$
\end{lemma*}

\begin{proof}
We have
\begin{eqnarray*}
    \E_{\mathcal{C},\mathcal{C}_M} \left[\mathcal{C}_M(\Bar{y}^k)\right] =  \E_\mathcal{C} \left[\E_{\mathcal{C}_M}\left[\mathcal{C}_M(\Bar{y}^k)\right]\right] 
    =\E_\mathcal{C} \left[\frac{1}{n}\sum_{j=1}^n \mathcal{C}_j(x_j^k)\right] =\frac{1}{n}  \sum_{j=1}^n \E_{\mathcal{C}_j} \left[\mathcal{C}_j(x_j^k)\right]  =\Bar{x}^k.
\end{eqnarray*}
Hence the result. 

\end{proof}

In the following Lemma, we show that based on the randomness of the compression operators, in expectation, we recover the exact average of the local models and the exact gradient for all iterations.  

\begin{lemma*}
Let Assumptions \ref{ass:compression} hold. Then for all $k\ge 0$, knowing $x^k$,     
$G(x^k)$ is an unbiased estimator of the gradient of function $F$ at $x^k$. 
\end{lemma*}

\begin{proof}
We have 
\begin{eqnarray*}
\E_{\mathcal{C},\mathcal{C}_M} \left[G_i(x^{k})\right]&=&
 \left\{
    \begin{array}{lll}
        \frac{\nabla f_i\left(x_i^k\right)}{n(1-p)} & \text{ if } \xi_k=0
       \\
       \frac{  \lambda }{np}\left( x_i^k - \E_{C,\mathcal{C}_M} \left[\mathcal{C}_M(\Bar{y}^k)\right]\right) & \text{ if } \xi_k=1 ~\& ~\xi_{k-1}=0, \\
           \frac{  \lambda}{n p} \left( x_i^k - \Bar{x}^k \right)& \text{ if } \xi_k=1 ~\& ~\xi_{k-1}=1, \\
    \end{array}
\right.\\
&\overset{{\rm By~Lemma} ~ \ref{lem:compmean}}{=}&
 \left\{
    \begin{array}{ll}
        \frac{\nabla f_i(x_i^k)}{n(1-p)} & \text{ if } \xi_k=0,
        \\
       \frac{  \lambda}{n p} \left( x_i^k - \Bar{x}^k \right) & \text{ if } \xi_k=1.
    \end{array}
\right.
\end{eqnarray*}
Therefore, 
\begin{eqnarray*}
\E[G_i(x^k) | x^k] &=&
 \E_{\xi_k} \left[ \E_{C,\mathcal{C}_M}\left[G_i(x^k) \right]\right]\\
&=& (1-p) \frac{ \nabla f_i(x_i^k)}{n(1-p)} + p \frac{  \lambda}{n p} \left( x_i^k - \Bar{x}^k \right) \\
&=& \nabla_{x_i} f(x^k) +  \nabla_{x_i} h\left(x^k\right) = \nabla_{x_i} F(x^k).
\end{eqnarray*}
Hence the result. 
\end{proof}

\subsubsection{Main convergence results}\label{sec:app_main conv}
Based on the results given in the previous section, we are now set to quote our key convergence results. Our next lemma gives an upper bound on the  iterate at each iteration. This bound is composed of two terms---the optimality gap, $F(x^k) - F(x^*)$,  and the norm of the optimal point, $\| x^*\|$.  

\begin{lemma*}
Let Assumption \ref{ass:smoothBound} hold, then 
$$\textstyle{\left\| x^k\right\|^2 \le \frac{4}{\mu} \left( F(x^k) - F(x^*) \right) + 2 \left\| x^*\right\|^2.}$$
\end{lemma*}
\begin{proof}
We have
\begin{eqnarray*}
\left\| x^k\right\|^2 \overset{\|a+b\|^2\le 2\|a\|^2+2\|b\|^2}{\le} 2 \left\| x^k - x^*\right\|^2 + 2 \left\| x^*\right\|^2 \le \frac{4}{\mu} \left( F(x^k) - F(x^*) \right) + 2 \left\| x^*\right\|^2.
\end{eqnarray*}
Hence the result. 
\end{proof}

Recall that, inspired by the expected smoothness property \cite{Gower2019}, we use a similar idea in our convergence proofs. The next lemma is a technical Lemma that helps us to prove the expected smoothness property \cite{Gower2019}. 
The bound in Lemma \ref{lem:mathcalA} is composed of the optimality gap, $F(x^k) - F(x^*)$, the difference between the gradients of $h$ at $x^k$ and $x^*$, that is, $\left\|\nabla h(x^k) - \nabla h(x^*)\right\|$, and an extra constant, $\beta$, which depends on the used compressors. 
\begin{lemma*}
Let Assumptions \ref{ass:compression} and \ref{ass:smoothBound} hold, then
\begin{eqnarray*}
\textstyle &\mathcal{A}:=\E_{\mathcal{C}_M,\mathcal{C}} \left\| x^k - Q\mathcal{C}_M(\Bar{y}^k) - x^* + Q\mathcal{C}_M(\Bar{y}^*) \right\|^2\le\frac{4n^2}{\lambda^2}\left\|\nabla h(x^k) - \nabla h(x^*)\right\|^2+ \alpha \left(F(x^k) - F(x^*)\right) + \beta,    
\end{eqnarray*}
where $\Bar{y}^* := \frac{1}{n}\sum_{j=1}^n \mathcal{C}_j(x_j^*)$, 
$\alpha:= \frac{4\left (4\omega + 4\omega_M (1+\omega) \right)}{\mu},$
 and 
\begin{eqnarray*}
 \beta := 2 \left (4\omega + 4\omega_M (1+\omega) \right) \left\|{x}^*\right\|^2+
 4\E_{\mathcal{C}_M,\mathcal{C}}\left\|Q\mathcal{C}_M(\Bar{y}^*)-  Q\Bar{x}^*\right\|^2.
 \end{eqnarray*} 
\end{lemma*}

\begin{proof}
We have
\begin{eqnarray*}
 \mathcal{A} &=& \E_{\mathcal{C}_M,\mathcal{C}} \left\| x^k - Q\Bar{x}^k + Q\Bar{x}^k - Q\mathcal{C}_M(\Bar{y}^k) - x^* + Q\Bar{x}^*  -Q\Bar{x}^* +Q\mathcal{C}_M(\Bar{y}^*) \right\|^2\\
&=& \E_{\mathcal{C}_M,\mathcal{C}} \left\| \left(x^k - Q\Bar{x}^k - x^* + Q\Bar{x}^*\right)+ \left( Q\Bar{x}^k - Q \Bar{y}^k\right)+\left( Q \Bar{y}^k- Q\mathcal{C}_M(\Bar{y}^k)\right) +\left(Q\mathcal{C}_M(\Bar{y}^*)-  Q\Bar{x}^*\right) \right\|^2\\
 &\le& 4\left\|x^k - Q\Bar{x}^k - x^* + Q\Bar{x}^*\right\|^2+ 4\E_{\mathcal{C}}\left\| Q\Bar{x}^k - Q \Bar{y}^k\right\|^2+4\E_{\mathcal{C}_M,\mathcal{C}}\left\| Q \Bar{y}^k- Q\mathcal{C}_M(\Bar{y}^k)\right\|^2\\
 &+&
4\E_{\mathcal{C}_M,\mathcal{C}}\left\|Q\mathcal{C}_M(\Bar{y}^*)-  Q\Bar{x}^*\right\|^2\\
 &=& 4\left\|x^k - Q\Bar{x}^k - x^* + Q\Bar{x}^*\right\|^2+ 
 4n\E_{\mathcal{C}}\left\| \Bar{x}^k -  \Bar{y}^k\right\|^2+4n\E_{\mathcal{C}_M,\mathcal{C}}\left\|  \Bar{y}^k- \mathcal{C}_M(\Bar{y}^k)\right\|^2\\
 &+&
 4\E_{\mathcal{C}_M,\mathcal{C}}\left\|Q\mathcal{C}_M(\Bar{y}^*)-  Q\Bar{x}^*\right\|^2\\
  &\le& 4\left\|x^k - Q\Bar{x}^k - x^* + Q\Bar{x}^*\right\|^2+ 
 4 \sum_{i=1}^n \E_{\mathcal{C}}\left\| x_i^k -  \mathcal{C}_i(x_i^k)\right\|^2+4 n \omega_M \E_{\mathcal{C}_M} \left\|  \Bar{y}^k\right\|^2\\
 &+&
 4\E_{\mathcal{C}_M,\mathcal{C}}\left\|Q\mathcal{C}_M(\Bar{y}^*)-  Q\Bar{x}^*\right\|^2\\
  &\le& 4\left\|x^k - Q\Bar{x}^k - x^* + Q\Bar{x}^*\right\|^2+ 
 4 \sum_{i=1}^n \omega_i\left\| x_i^k\right\|^2+4  \omega_M \sum_{i=1}^n (1+\omega_i)\left\|  x_i^k \right\|^2
 +
 4\E_{\mathcal{C}_M,\mathcal{C}}\left\|Q\mathcal{C}_M(\Bar{y}^*)-  Q\Bar{x}^*\right\|^2\\
 &\le & 4\frac{n^2}{\lambda^2}\left\|\nabla h(x^k) - \nabla h(x^*)\right\|^2+\left (4\omega + 4\omega_M (1+\omega) \right) \left\| x^k\right\|^2 +4\E_{\mathcal{C}_M,\mathcal{C}}\left\|Q\mathcal{C}_M(\Bar{y}^*)-  Q\Bar{x}^*\right\|^2\\
 &\overset{{\rm By~Lemma}~ \ref{lem:xrelF}}{\le}& 4\frac{n^2}{\lambda^2}\left\|\nabla h(x^k) - \nabla h(x^*)\right\|^2+\left (4\omega + 4\omega_M (1+\omega) \right) \left( \frac{4}{\mu} \left( F(x^k) - F(x^*) \right) + 2 \left\| x^*\right\|^2 \right) \\
 &&+4\E_{\mathcal{C}_M,\mathcal{C}}\left\|Q\mathcal{C}_M(\Bar{y}^*)-  Q\Bar{x}^*\right\|^2\\
 &\le & 4\frac{n^2}{\lambda^2}\left\|\nabla h(x^k) - \nabla h(x^*)\right\|^2+ \alpha \left(F(x^k) - F(x^*)\right) + \beta.  
\end{eqnarray*}
Hence the result. 
\end{proof}

Now we are all set to prove the expected smoothness property in our setup. 
\begin{lemma*}[Expected Smoothness]
Let Assumptions \ref{ass:compression} and \ref{ass:smoothBound} hold, then
\begin{equation}
\textstyle \E\left[\|G(x^{k})\|^2|x^k\right] \le 4 \gamma \left(F(x^k) - F(x^*)\right) + \delta,
\end{equation}  
where $$\textstyle{\gamma:= \frac{\alpha \lambda^2 (1-p)}{2 n^2 p} + \max \left\{ \frac{ L_f}{(1-p)}, \frac{\lambda}{n} \left(1+\frac{4(1-p)}{p}\right)
  \right\}}$$ and $$\delta:= \frac{2 \beta \lambda^2 (1-p)}{n^2 p}+2\E \| G(x^{*})\|^2.$$
\end{lemma*}
\begin{proof}
We have 
$$\|G(x^{k}) - G(x^{*})\|^2=
 \left\{
    \begin{array}{lll}
        \frac{\|\nabla f\left(x^k\right) - \nabla f\left(x^*\right)\|^2}{(1-p)^2} & \text{ if } \xi_k=0
       \\
        \frac{  \lambda^2 }{n^2 p^2}\left\| x^k - Q\mathcal{C}_M(\Bar{y}^k) - x^* + Q\mathcal{C}_M(\Bar{y}^*) \right\|^2 
        & \text{ if } \xi_k=1 ~\& ~\xi_{k-1}=0, \\
           \frac{  1}{ p^2} \|\nabla h(x^k) - \nabla h(x^*) \|^2 & \text{ if } \xi_k=1 ~\& ~\xi_{k-1}=1.\\
    \end{array}
\right.$$
Finally,  \begin{eqnarray*}
 \E_{\xi_k,\xi_{k-1}} \|G(x^{k}) - G(x^{*})\|^2 &=& (1-p) \frac{\|\nabla f\left(x^k\right) - \nabla f \left(x^*\right)\|^2}{(1-p)^2} + p^2 \frac{  1}{ p^2} \|\nabla h(x^k) - \nabla h(x^*) \|^2 \\
 &+& p(1-p) \frac{  \lambda^2 }{n^2 p^2}\left\| x^k - Q\mathcal{C}_M(\Bar{y}^k) - x^* + Q\mathcal{C}_M(\Bar{y}^*) \right\|^2 \\
 &=&  \frac{\|\nabla f\left(x^k\right) - \nabla f \left(x^*\right)\|^2}{(1-p)} +  \|\nabla h(x^k) - \nabla h(x^*) \|^2 \\
 &+&  \frac{  \lambda^2 (1-p)}{n^2 p}\left\| x^k - Q\mathcal{C}_M(\Bar{y}^k) - x^* + Q\mathcal{C}_M(\Bar{y}^*) \right\|^2.
 \end{eqnarray*}
 Therefore, by using Lemma \ref{lem:mathcalA} we get
 \begin{eqnarray*}
 \E \|G(x^{k}) - G(x^{*})|x^k\|^2 
 &= &  \frac{\|\nabla f\left(x^k\right) - \nabla f \left(x^*\right)\|^2}{(1-p)} +  \|\nabla h(x^k) - \nabla h(x^*) \|^2+  \frac{  \lambda^2 (1-p)}{n^2 p} \mathcal{A}\\
 &\le&  \frac{\|\nabla f\left(x^k\right) - \nabla f \left(x^*\right)\|^2}{(1-p)} +  \|\nabla h(x^k) - \nabla h(x^*) \|^2\\
 &+&  \frac{\lambda^2 (1-p)}{n^2 p} \left(4\frac{n^2}{\lambda^2}\left\|\nabla h(x^k) - \nabla h(x^*)\right\|^2+ \alpha \left(F(x^k) - F(x^*)\right) + \beta\right)\\
 &=& \frac{\|\nabla f\left(x^k\right) - \nabla f \left(x^*\right)\|^2}{(1-p)} + \left(1+\frac{4(1-p)}{p}\right) \|\nabla h(x^k) - \nabla h(x^*) \|^2 \\
 &+& \frac{\alpha \lambda^2 (1-p)}{n^2 p}  \left(F(x^k) - F(x^*)\right) + \frac{\beta \lambda^2 (1-p)}{n^2 p} \\
 &\le& \frac{2 L_f}{(1-p)}\left(f(x^k) - f(x^*)\right) + \frac{2\lambda}{n} \left(1+\frac{4(1-p)}{p}\right)\left(h(x^k) - h(x^*)\right)\\
 &+& \frac{\alpha \lambda^2 (1-p)}{n^2 p}  \left(F(x^k) - F(x^*)\right) + \frac{\beta \lambda^2 (1-p)}{n^2 p} \\
&\le& 2 \gamma \left(F(x^k) - F(x^*)\right) + \frac{\beta \lambda^2 (1-p)}{n^2 p}.
 \end{eqnarray*}
Finally, we obtain
  \begin{eqnarray*}
  \E \|G(x^{k})|x^k\|^2 &\le& 2\E \|G(x^{k}) - G(x^{*})|x^k\|^2+ 2\E \| G(x^{*})\|^2\\
  &\le &4 \gamma \left(F(x^k) - F(x^*)\right) +\frac{2 \beta \lambda^2 (1-p)}{n^2 p}+2 \E \| G(x^{*})\|^2\\
   &\le& 4\gamma \left(F(x^k) - F(x^*)\right) + \delta.
  \end{eqnarray*}
  Hence the result. 
\end{proof}

Based on the above results, the convergence of Algorithm \ref{alg:ComL2GD} for strongly convex functions follows directly from Lemmas \ref{lem:unbiasGrad}, \ref{lem:ES} and Theorem 3.1 from \cite{Gower2019}.
\subsubsection{Nonconvex convergence}\label{sec:app_nnc}
To prove the convergence of Algorithm \ref{alg:ComL2GD} for $L$-smooth, nonconvex functions we need an addtional result. The next Lemma from \cite{mishchenko2020random} is a key lemma to prove the convergence of nonconvex case. Our proof of the lemma is different from \cite{mishchenko2020random}.
\begin{lemma}\label{lemma:conv rec}
Let for $0\le k\le K$ the following holds:
\begin{eqnarray}\label{eq:recur-1}
p_{k+1}\le (1+a)p_k-bq_k+c,
\end{eqnarray}
where $\{p_k\}_{k=0}^K$ and $\{q_k\}_{k=0}^K$ are non-negative sequences and $a,b,c\ge0$ are constants. Then 
\begin{eqnarray}\label{eq:recur-2}
\min_{k=0,1,\cdots K-1}q_{k}\le \frac{(1+a)^K}{bK}p_0+\frac{c}{b}.
\end{eqnarray}
\end{lemma}

\begin{proof}
Dividing both sides of \eqref{eq:recur-1} by $(1+a)^{K+1}$ and summing from $k=0,1,\cdots K$ we have
\begin{eqnarray}\label{eq:recur-3}
\sum_{k=0}^K\frac{1}{(1+a)^{K+1}}p_{k+1}\le \sum_{k=0}^K\frac{1}{(1+a)^{K}}p_k-\sum_{k=0}^K\frac{b}{(1+a)^{K+1}}q_k+\sum_{k=0}^K\frac{c}{(1+a)^{K+1}}.
\end{eqnarray}
After rearranging it becomes
\begin{eqnarray}\label{eq:recur-3}
\sum_{k=0}^K\frac{b}{(1+a)^{K+1}}q_k \le p_0- \frac{1}{(1+a)^{K+1}}p_{K+1}+\sum_{k=0}^K\frac{c}{(1+a)^{K+1}}.
\end{eqnarray}
Noting $\sum_{k=0}^K\frac{1}{(1+a)^{K+1}}\le \frac{1}{1-\frac{1}{1+a}}-1=\frac{1}{a},$ we have
\begin{eqnarray}\label{eq:recur-4}
\min_{k=0,1,\cdots K}q_{k}\sum_{k=0}^K\frac{1}{(1+a)^{K+1}}\le\sum_{k=0}^K\frac{1}{(1+a)^{K+1}}q_k \le \frac{p_0}{b}+\frac{c}{ab}.
\end{eqnarray}
Hence the result. 
\end{proof}

Now we are all set to prove the nonconvex convergence result. 

\begin{theorem}[Non convex case]Let Assumption \ref{ass:compression}  hold. Assume also that $F$ is $L$-smooth, bounded from below by $F(x^*)$. Then to reach a precision, $\textstyle{\epsilon>0}$, set the stepsize, $\textstyle{\eta=\min\{\frac{1}{\sqrt{2L\gamma K}},\frac{\epsilon^2}{L\delta}\}},$ such that for $
\textstyle{K\ge \frac{6L}{\epsilon^4}\max\{{12\gamma(F(x^0)-F(x^*))^2},{\delta}\},}$ we have 
$\min_{k=0,1,\dots,K}  \E \|\nabla F(x^k)\|_2\le \epsilon.$
\end{theorem}
\begin{proof}
From $L$-smoothness of $F$ we have
\begin{eqnarray*}
F(x^{k+1}) &\leq&  F(x^k) - \eta _k\ve{\nabla F(x^k)}{G(x^k)} + \frac{L}{2}\eta_k^2\|G(x_k)\|^2.
\end{eqnarray*} 
By taking the expectation in the above inequality, conditional on $x^k$, we get 
\begin{eqnarray*}
\E{ \left[ F(x^{k+1}) \;|\; x^k \right]} &\overset{{\rm By\;Lemma \; \ref{lem:unbiasGrad}}}{\leq}& F(x^k)- \eta_k  \|\nabla F(x^k)\|_2^{2} + \frac{L\eta_k^2}{2}\E\left(\|G(x_k)\|^2|x_k\right),
\end{eqnarray*}
which by using Lemma \ref{lem:ES} reduces to
\begin{eqnarray*}
\E{ \left[ F(x^{k+1}) \;|\; x^k \right]}&\leq & F(x^k)- \eta_k  \|\nabla F(x^k)\|_2^{2} + \frac{L\eta_k^2}{2}\left(4\gamma \left(F(x^k) - F(x^*)\right) + \delta\right)\\
&\leq & F(x^k)- \eta_k  \|\nabla F(x^k)\|_2^{2} +  {2L\eta_k^2}\gamma \left(F(x^k) - F(x^*)\right) +  \frac{L\eta_k^2\delta}{2}. 
\end{eqnarray*} 
By taking the expectation in the last inequality, we get 
\begin{eqnarray}\label{eq:final descent}
\E{ \left[ F(x^{k+1}) \right]}-F(x^*)&\leq& (1+ {2L\eta_k^2}\gamma)\left(\E{ \left[ F(x^k) \right]}-F(x^*)\right)- \eta_k \E \|\nabla F(x^k)\|_2^{2} +  \frac{L\eta_k^2\delta}{2}.
\end{eqnarray}
Consider $\eta_k=\eta> 0$ in \eqref{eq:final descent}. Then, 
equation \eqref{eq:final descent} is similar to \eqref{eq:recur-1} with $p_k=\E{ \left[ F(x^k) \right]}-F(x^*), q_k=\E\|\nabla F_k\|^2, a=2L\gamma\eta^2, b=\eta, c=\frac{L\eta^2\delta}{2}$. Therefore, 
\begin{eqnarray}\label{eq:descent-con-2}
\min_{k=0,1,\cdots K-1}\E\|\nabla F (x^{k})\|^2&\le& \frac{(1+2L\gamma\eta^2)^K}{\eta K}\left(F(x^0)-F(x^*)\right)+\frac{L\eta\delta}{2}.
\end{eqnarray}
Now setting $\eta$ such that $2L\gamma\eta^2K\le 1$, and the inequality, $(1+x)^K\le e^{xK}$ we have
\begin{eqnarray}\label{eq:descent-con-3}
\min_{k=0,1,\cdots K-1}\E\|\nabla F(x^{k})\|^2&\le& \frac{3\left(F(x^0)-F(x^*)\right)}{\eta K}+\frac{L\eta\delta}{2}.
\end{eqnarray}
For a given precision, $\epsilon>0$, to make $\min_{k=0,1,\cdots K-1}\E\|\nabla F(x^{k})\|^2\le \epsilon^2$, we require $\frac{3\left(F(x^0)-F(x^*)\right)}{\eta K}\le \frac{\epsilon^2}{2}$ and ${L\eta\delta}\le {\epsilon^2}$, resulting in
$$
K\ge \frac{6(F(x^0)-F(x^*))}{\eta\epsilon^2}\;{\rm and\;} \eta\le \frac{\epsilon^2}{L\delta}.
$$
The above, alongwith $\eta\le \frac{1}{\sqrt{2L\gamma K}}$, implies 
$$
K\ge \frac{72L\gamma(F(x^0)-F(x^*))^2}{\epsilon^4}\;{\rm and\;}K\ge \frac{6L\delta}{\epsilon^4}.
$$
Hence the result. 
\end{proof}

\subsubsection{Optimal rate and communication}\label{sec:app optimal}

The following proofs are related to optimal rate and communication as given in \S\ref{sec:optimalrate}. 

\begin{theorem*}[Optimal rate] 
The probability $p^*$ minimizing $\gamma$ is equal to 
$\max\{p_e,p_A\}$, where $p_e = \frac{7 \lambda + L - \sqrt{\lambda^2 + 14 \lambda L + L^2}}{6 \lambda}$ and $p_A$ is the optimizer of the function $A(p) = \frac{\alpha \lambda^2}{2 n^2 p } + \frac{ L}{n(1-p)}$  in $(0,1)$.
\end{theorem*}
\begin{proof}
We can rewrite $\gamma$ as follows
  \begin{eqnarray*}
 \gamma &=&  -\frac{\alpha \lambda^2}{2 n^2} + \max \left\{ A(p), B(p)
  \right\},
  \end{eqnarray*}
  where $A(p) = \frac{\alpha \lambda^2}{2 n^2 p } + \frac{ L}{n(1-p)}$ and $B(p) = \frac{\alpha \lambda^2}{2 n^2 p }+  \frac{4\lambda}{np} -\frac{3 \lambda}{n}$.
  The function $B$ is monotonically decreasing as a function of $p$. The function $A$ goes to $\infty$ as $p$ goes to zero or one, and it has one stationary point between zero and one hence it is convex in the interval $(0,1)$. Thus it admits an optimizer $p_A$ in $(0,1)$. Note that $p_e = \frac{7 \lambda + L - \sqrt{\lambda^2 + 14 \lambda L + L^2}}{6 \lambda}$ is the point for which $A(p)$ is equal to $B(p)$. Note also that near to zero $ B(p) \ge A(p)$.  Therefore if $p_e \le p_A$ then the optimizer of $\gamma$ is $p_A$ otherwise it is equal to $p_e$.  
   Thus the probability $p^*$ optimizing $\gamma$ is equal to $\max\{p_e,p_A\}$.
\end{proof}

\begin{lemma*}The optimizer probability $p_A$   of the function $A(p) = \frac{\alpha \lambda^2}{2 n^2 p } + \frac{ L}{n(1-p)}$  in $(0,1)$ is equal to 
$$p_A=
 \left\{
    \begin{array}{lll}
       \frac{1}{2}  & \text{ if } 2nL =\alpha \lambda^2
       \\
     \frac{-2 \alpha \lambda^2 + 2\lambda \sqrt{2\alpha n L}}{2(2nL -\alpha \lambda^2)}   & \text{ if } 2nL > \alpha \lambda^2
       \\
        \frac{-2 \alpha \lambda^2 - 2\lambda \sqrt{2\alpha n L}}{2(2nL -\alpha \lambda^2)}   & \text{ otherwise } 
    \end{array}
\right.$$
\end{lemma*}

\begin{proof}
If $2nL \neq \alpha \lambda^2$, then the function $A$ has the following two stationary points
 $\frac{-2 \alpha \lambda^2 + 2\lambda \sqrt{2\alpha n L}}{2(2nL -\alpha \lambda^2)}   $ and 
  $\frac{-2 \alpha \lambda^2 - 2\lambda \sqrt{2\alpha n L}}{2(2nL -\alpha \lambda^2)}$. If  $2nL =\alpha \lambda^2$, then  the function $A$ has one stationary point equal to $\frac{1}{2}$.
\end{proof}

\begin{theorem*}[Optimal communication] 
The probability $p^*$ optimizing $C$ is equal to 
$\max\{p_e,p_A\}$, where $p_e = \frac{7 \lambda + L - \sqrt{\lambda^2 + 14 \lambda L + L^2}}{6 \lambda}$ and $p_A = 1 - \frac{Ln}{\alpha \lambda^2}$.
\end{theorem*}

\begin{proof}
  We can rewrite $nC$ as follows
  \begin{eqnarray*}
 nC &=&   \max \left\{ A(p), B(p)
  \right\},
  \end{eqnarray*}
  where $A(p) = \frac{\alpha \lambda^2 p(1-p)}{2n} + \frac{\alpha \lambda^2(1-p)}{2 n} + Lp$ and $B(p) =\frac{\alpha \lambda^2 p(1-p)}{2n}  + \frac{\alpha \lambda^2(1-p)}{2 n}+  4\lambda(1-p) - 3 \lambda p (1-p)$.
  The function $B$ is monotonically decreasing as a function of $p$ in $[0,1]$. Note that $B(0)  = \frac{\alpha \lambda^2}{2 n}+  4\lambda$ and $B(1)=0$. The function $A$  admits a minimizer equal to $p_A = 1 - \frac{Ln}{\alpha \lambda^2}$. Of course $p_A$ is a  probability   under the condition that $Ln \le \alpha \lambda^2$. Thus we consider the following  2 scenarios
  \begin{enumerate}
  \item If $Ln > \alpha \lambda^2$ ($p_A<0$) then $p^* = p_e$
   \item else $p^* = \max\{p_e,p_A\}$.
   \end{enumerate}
   We conclude in both cases that $p^* = \max\{p_e,p_A\}$.
\end{proof}

\subsection{Addendum to the Experimental Results}\label{sec:app_nn}

\smartparagraph{Batch Normalization.} Beside the trainable parameters, the ResNet models contain batch normalization~\cite{ioffe2015batch} layers that are crucial for training. The logic of batch normalization depends on the estimation of running mean and variance, and these statistics can be pretty personalized for each client in a heterogeneous data regime. The implementation of FedAvg and FedOpt in FedML considers the averaging of these statistics during the aggregation phase. In our implementation, the batch normalization statistics are included into aggregation. 

\smartparagraph{Step-size.} The step-sizes for FedAvg and FedOpt tuned via selecting step sizes from the following set $\{0.01, 0.1, 0.2, 0.5, 1.0, 2.0, 4.0\}$. We consider the step size for both algorithms to be $0.1$. Starting with step size $0.2$ algorithms diverge; we also did not use use step size schedulers. Additionally, we have tuned number of local epochs for FedAvg from the following set $\{1,2,3,4\}$. The used batch size is set to $256$.

\smartparagraph{Compressed L2Gd vs. FedOpt.} From the experiments in Section \ref{sec:dnn}, Figures~\ref{fig:training_resnet}--\ref{fig:training_mobilenet}, we realized that FedAvg is not a competitive no-compression baseline for L2GD; see Table \ref{tab:modelsize}. FedOpt, on the other hand, remains a competitive no-compression baseline comparable to compressed L2GD. Therefore, we separately measure the performance of compressed L2GD and non-compression FedOpt for training \texttt{ResNet-18}, \texttt{DenseNet-121}, and  \texttt{MobileNet}. Figures~\ref{fig:training_resnet_l2g_only}--\ref{fig:training_densenet_l2g_only} demonstrate that L2GD with natural compressor (that by design has small variance) empirically behaves the best and converges approximately $5$ times faster compare to FedOpt. They also show that compressed L2GD with natural compressor sends the least data and drives the loss down the most. At the same time, L2GD with natural compressor reaches the best accuracy for both train and test sets.

\begin{figure*}[t]
	\centering
	\captionsetup[sub]{font=small,labelfont={}}	
	\begin{subfigure}[ht]{\textwidth}
		\includegraphics[width=\textwidth]{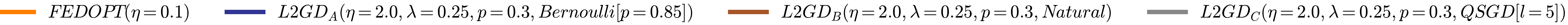}
	\end{subfigure}
	\begin{subfigure}[ht]{0.45\textwidth}
		\includegraphics[width=\textwidth]{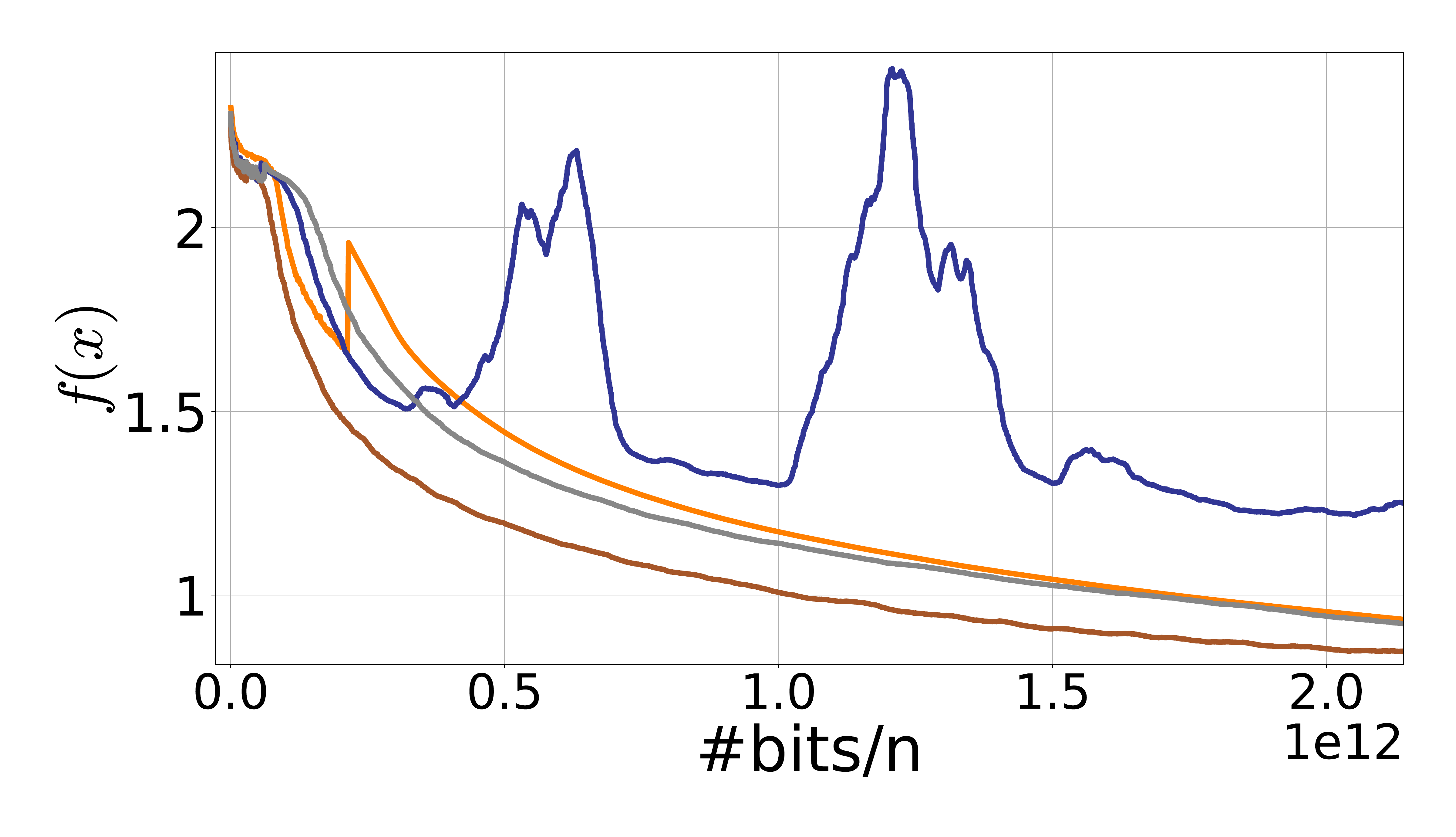} \caption{}
	\end{subfigure}
	\begin{subfigure}[ht]{0.45\textwidth}
		\includegraphics[width=\textwidth]{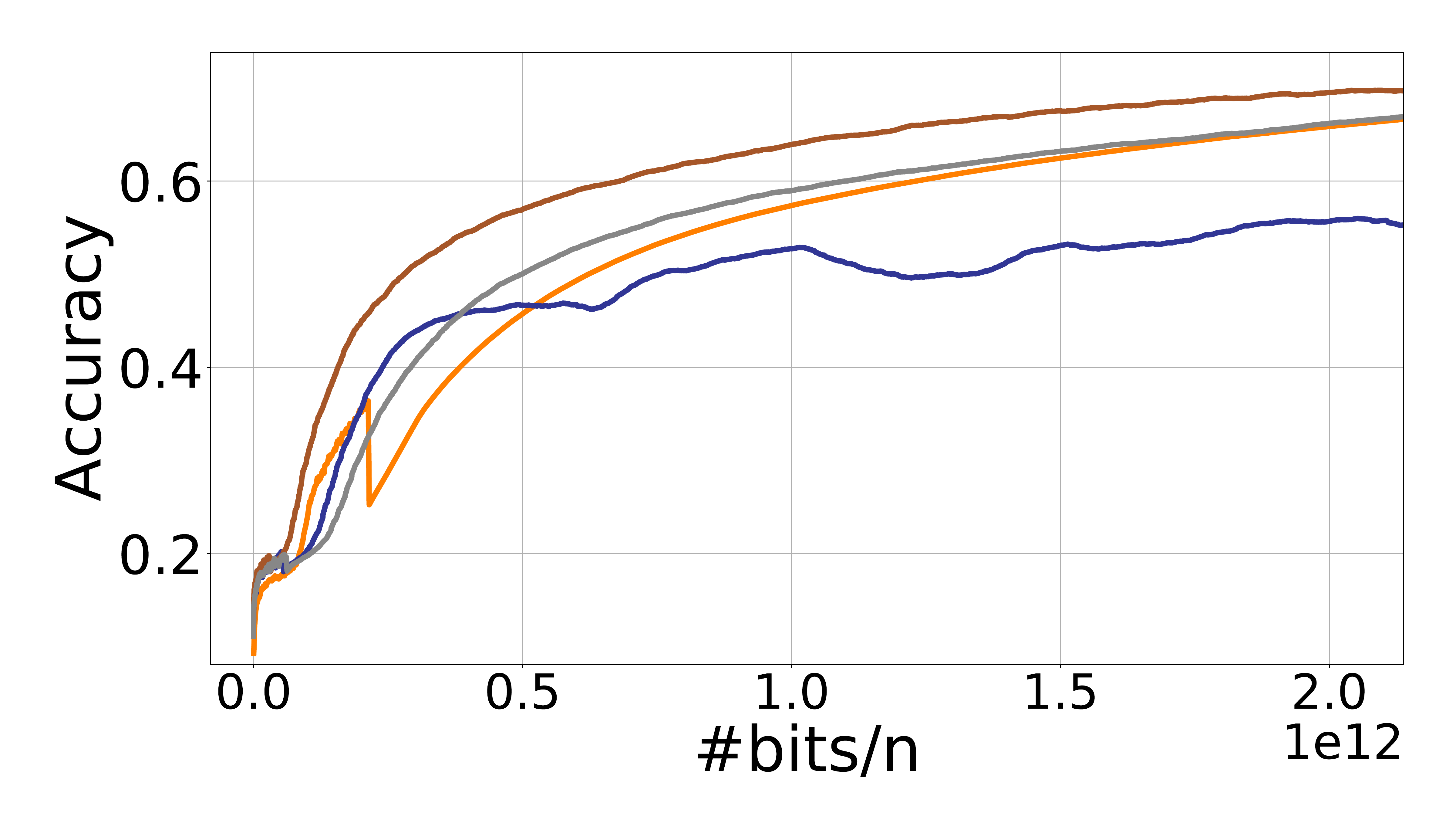} \caption{}
	\end{subfigure}

	\begin{subfigure}[ht]{0.45\textwidth}
		\includegraphics[width=\textwidth]{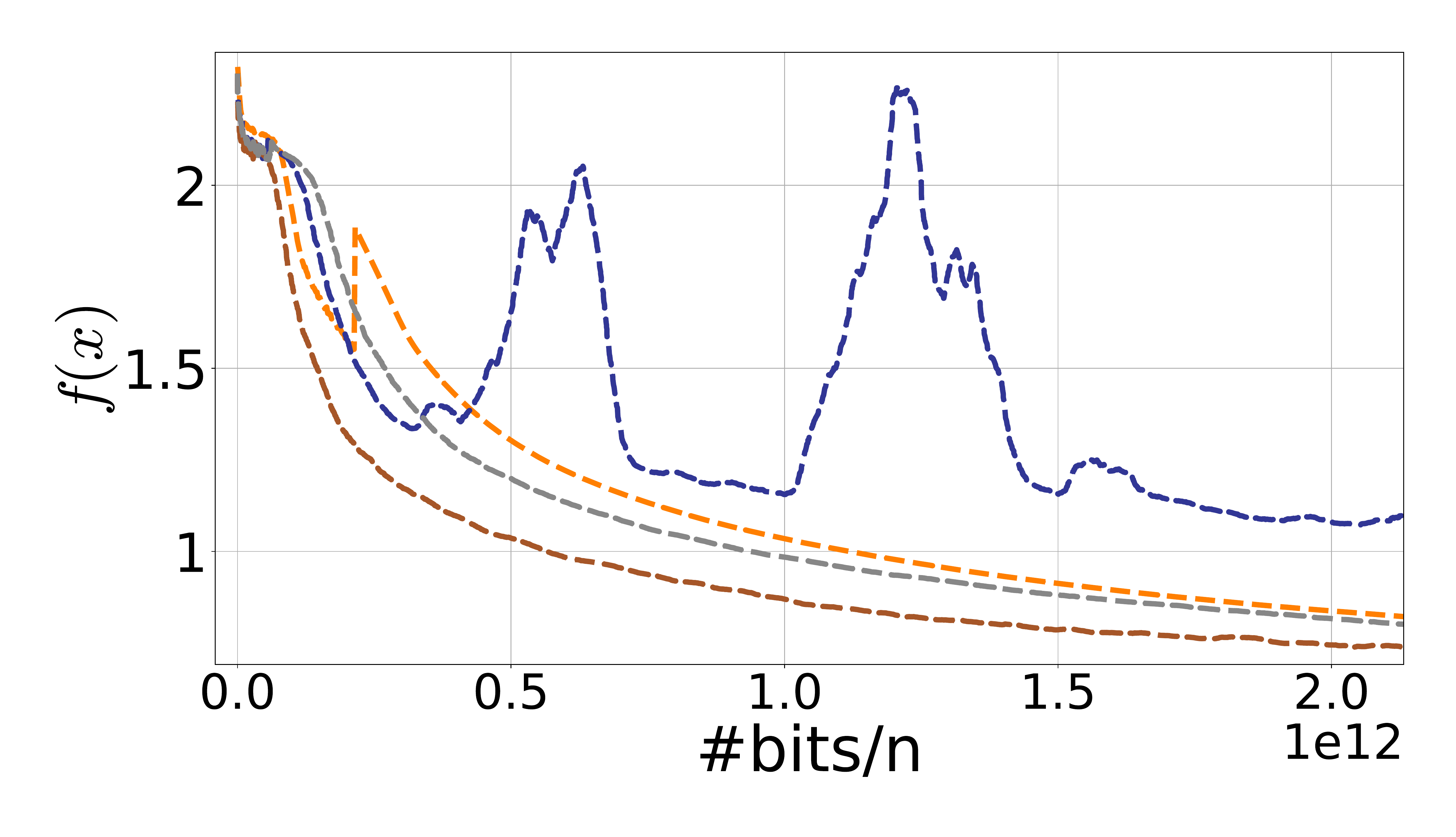}\caption{}
	\end{subfigure}
	\begin{subfigure}[ht]{0.45\textwidth}
		\includegraphics[width=\textwidth]{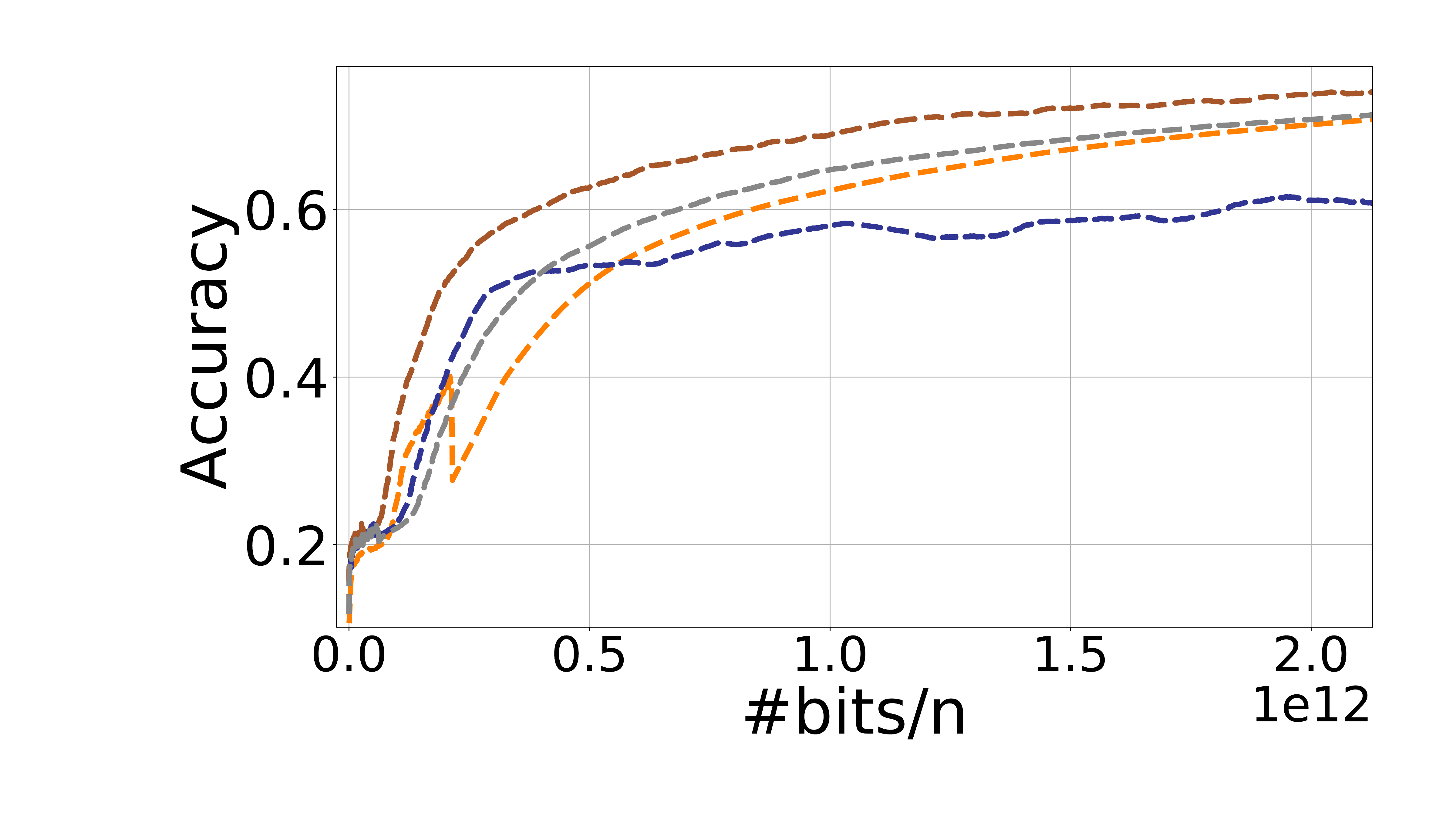} \caption{}
	\end{subfigure}
		\caption{{Training \texttt{ResNet-18} on \texttt{CIFAR-10}, with $n=10$ workers. Loss and Top-1 accuracy on train (a) - (b) and test data (c) - (d).}}
	\label{fig:training_resnet_l2g_only}
\end{figure*}

\begin{figure*}[t]
	\centering
	\captionsetup[sub]{font=small,labelfont={}}	
	\begin{subfigure}[ht]{\textwidth}
		\includegraphics[width=\textwidth]{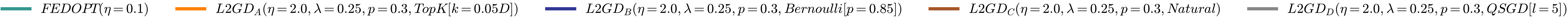}
	\end{subfigure}
	\begin{subfigure}[ht]{0.45\textwidth}
		\includegraphics[width=\textwidth]{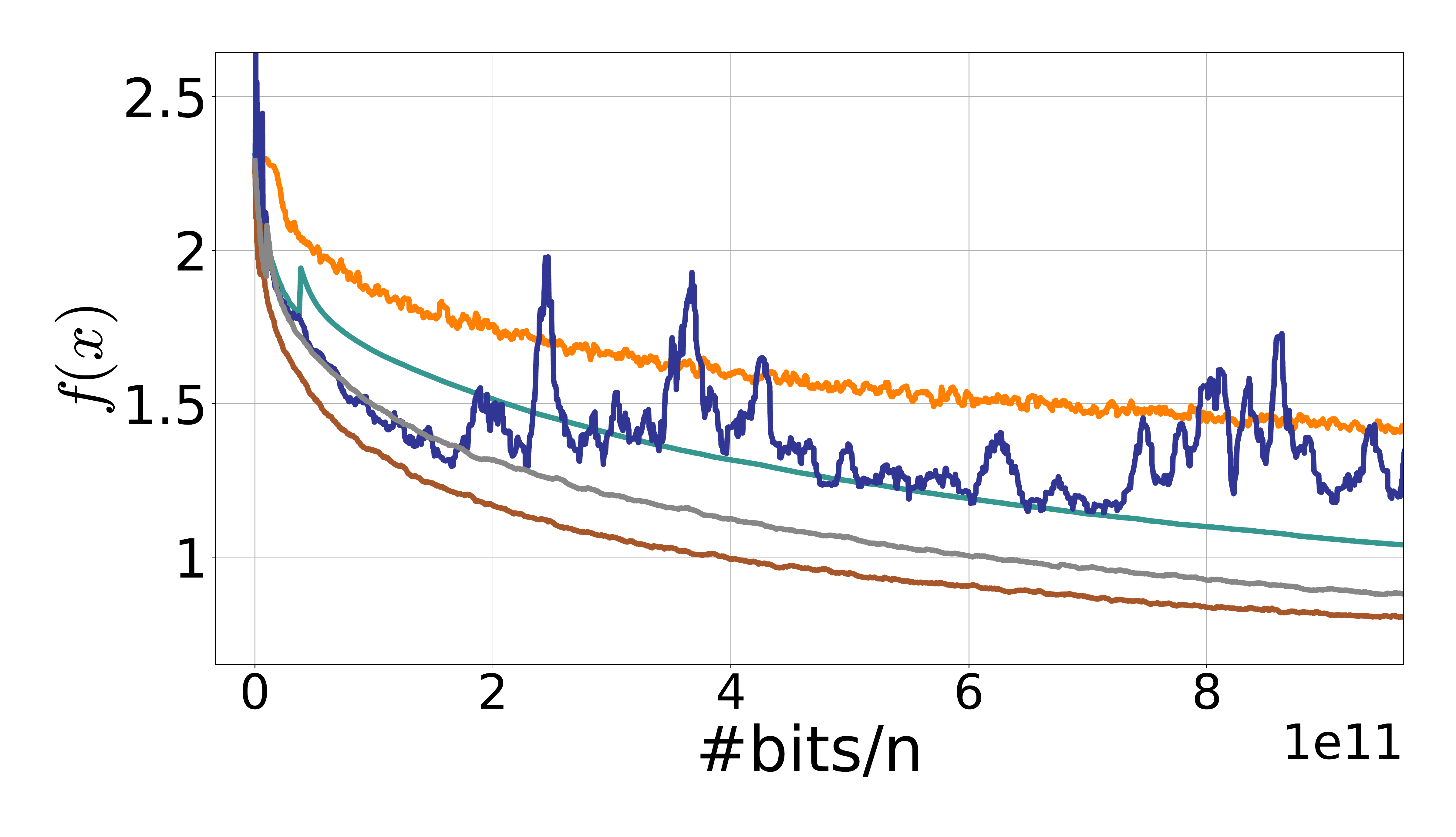} \caption{}
	\end{subfigure}
	\begin{subfigure}[ht]{0.45\textwidth}
		\includegraphics[width=\textwidth]{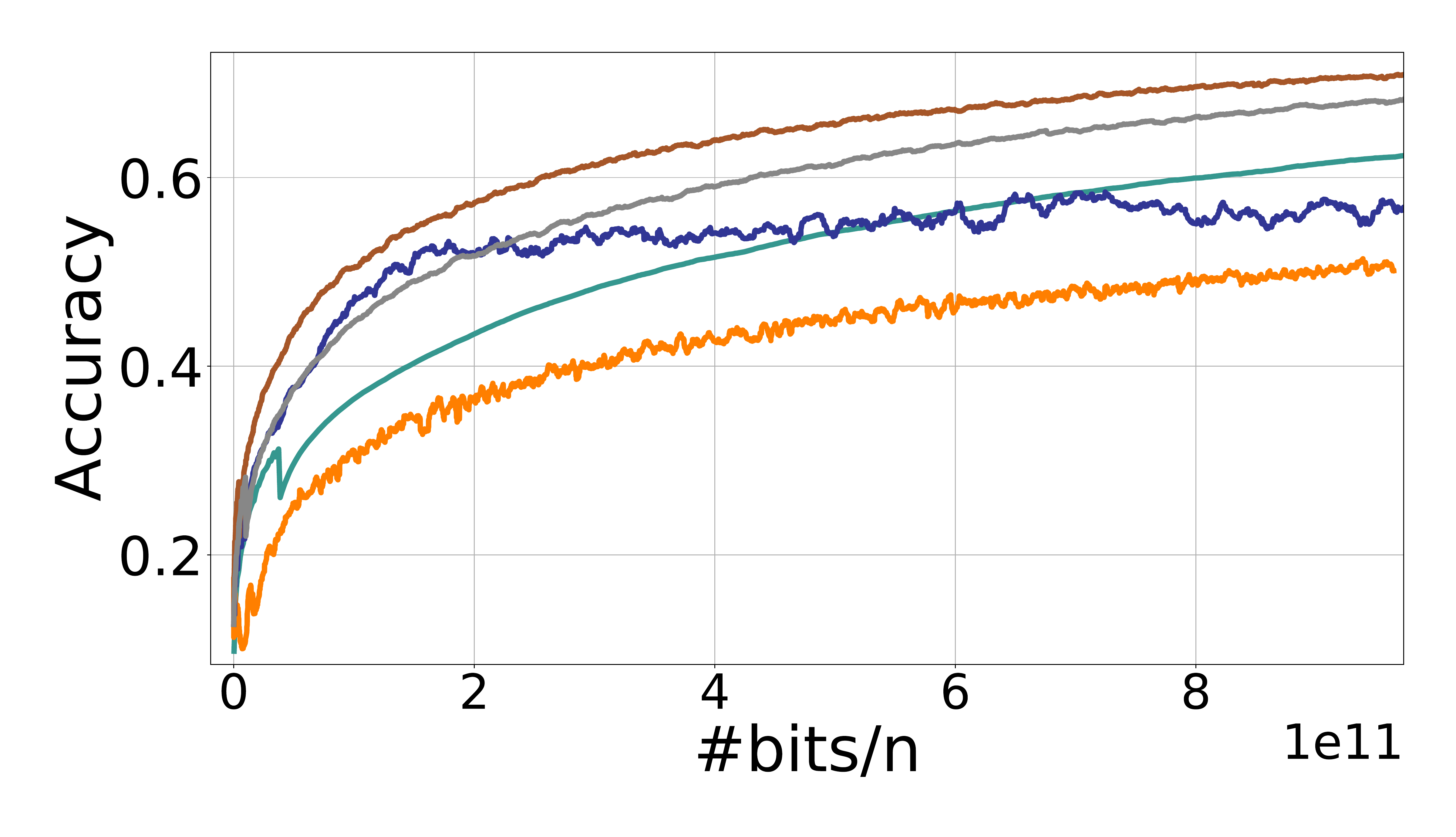} \caption{}
	\end{subfigure}

	\begin{subfigure}[ht]{0.45\textwidth}
		\includegraphics[width=\textwidth]{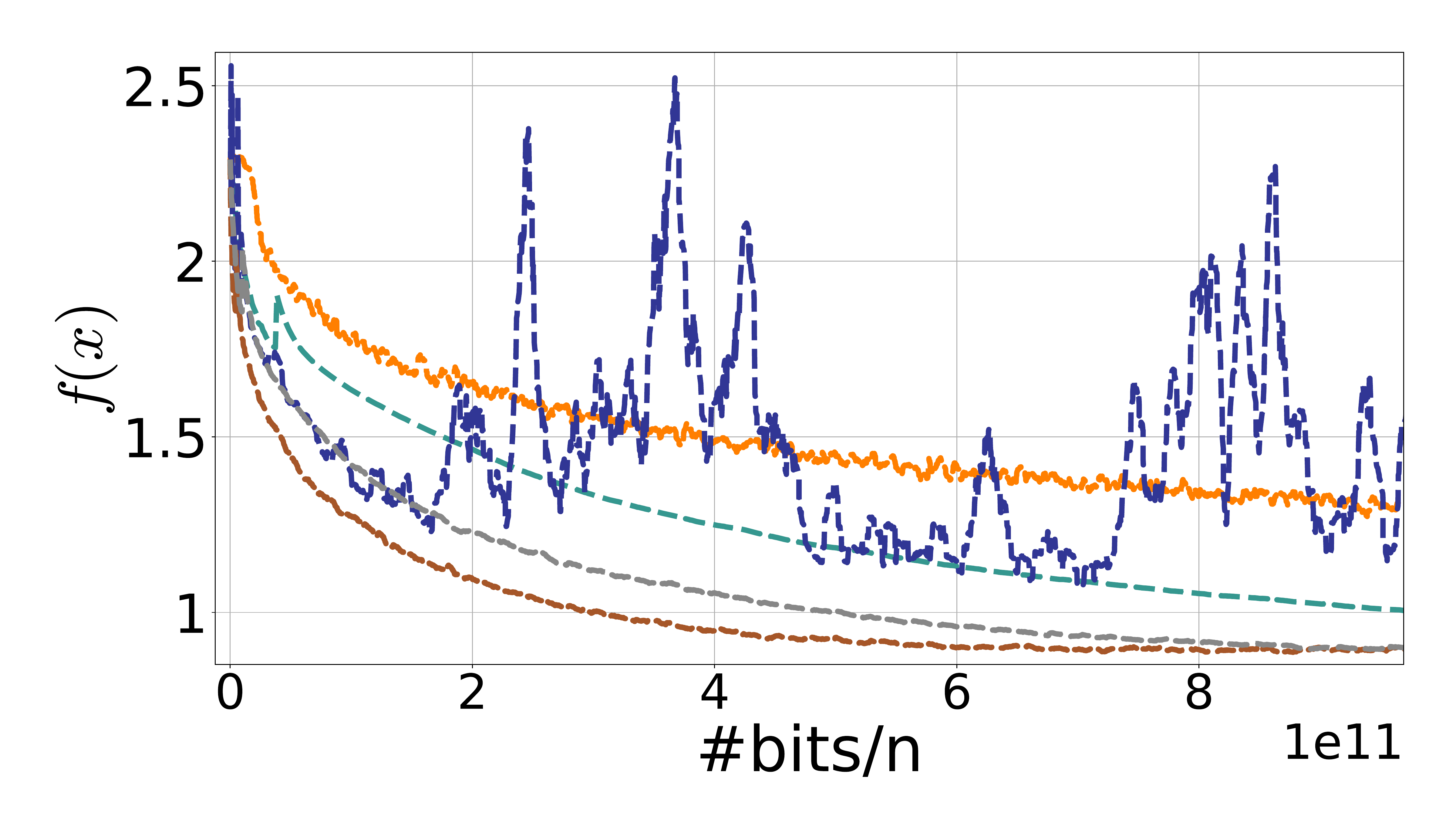} \caption{}
	\end{subfigure}
	\begin{subfigure}[ht]{0.45\textwidth}
		\includegraphics[width=\textwidth]{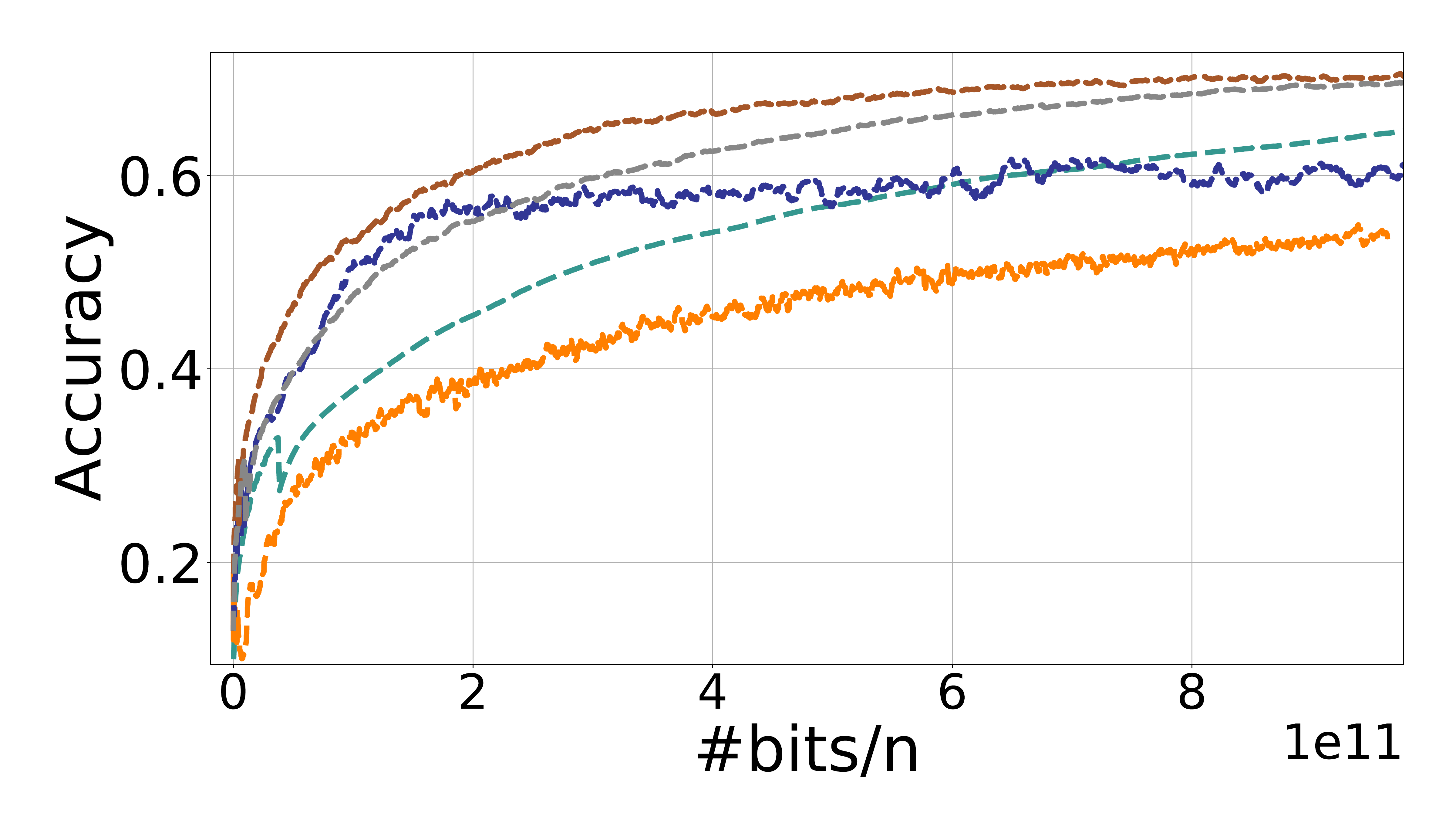}\caption{}
	\end{subfigure}
	\caption{{Training \texttt{DenseNet-121} on \texttt{CIFAR-10}, with $n=10$ workers. Loss and Top-1 accuracy on train (a) - (b), and test data (c) - (d).}}
	\label{fig:training_densenet_l2g_only}
\end{figure*}

\begin{figure*}[t]
	\centering
	\captionsetup[sub]{font=small,labelfont={}}	
	\begin{subfigure}[ht]{\textwidth}
		\includegraphics[width=\textwidth]{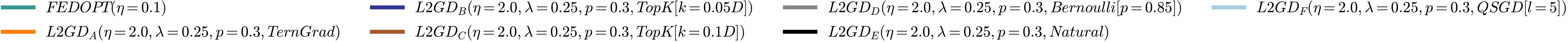}
	\end{subfigure}	
	\begin{subfigure}[ht]{0.45\textwidth}
		\includegraphics[width=\textwidth]{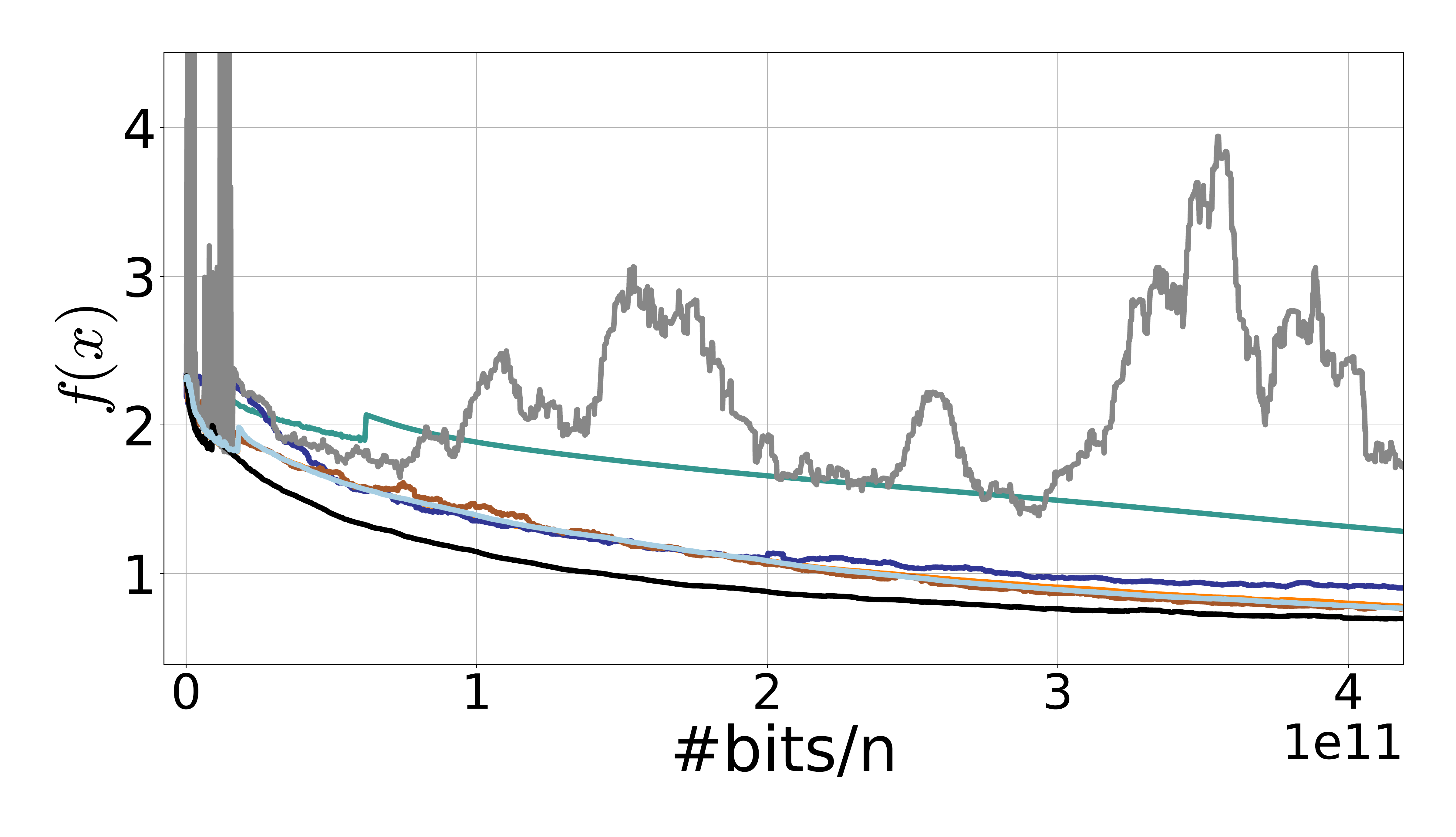} \caption{}
	\end{subfigure}
	\begin{subfigure}[ht]{0.45\textwidth}
		\includegraphics[width=\textwidth]{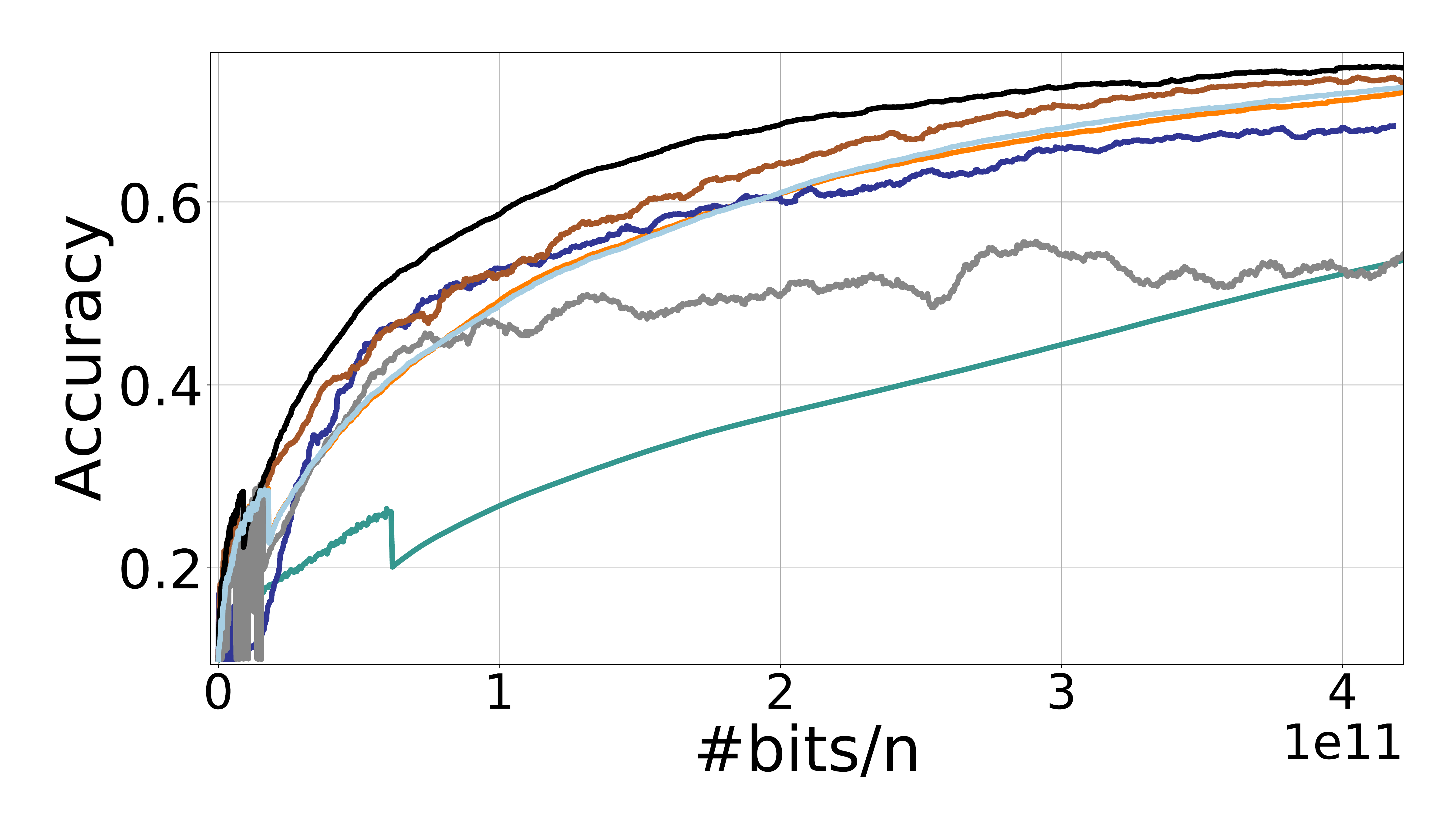} \caption{}
	\end{subfigure}

	\begin{subfigure}[ht]{0.45\textwidth}
		\includegraphics[width=\textwidth]{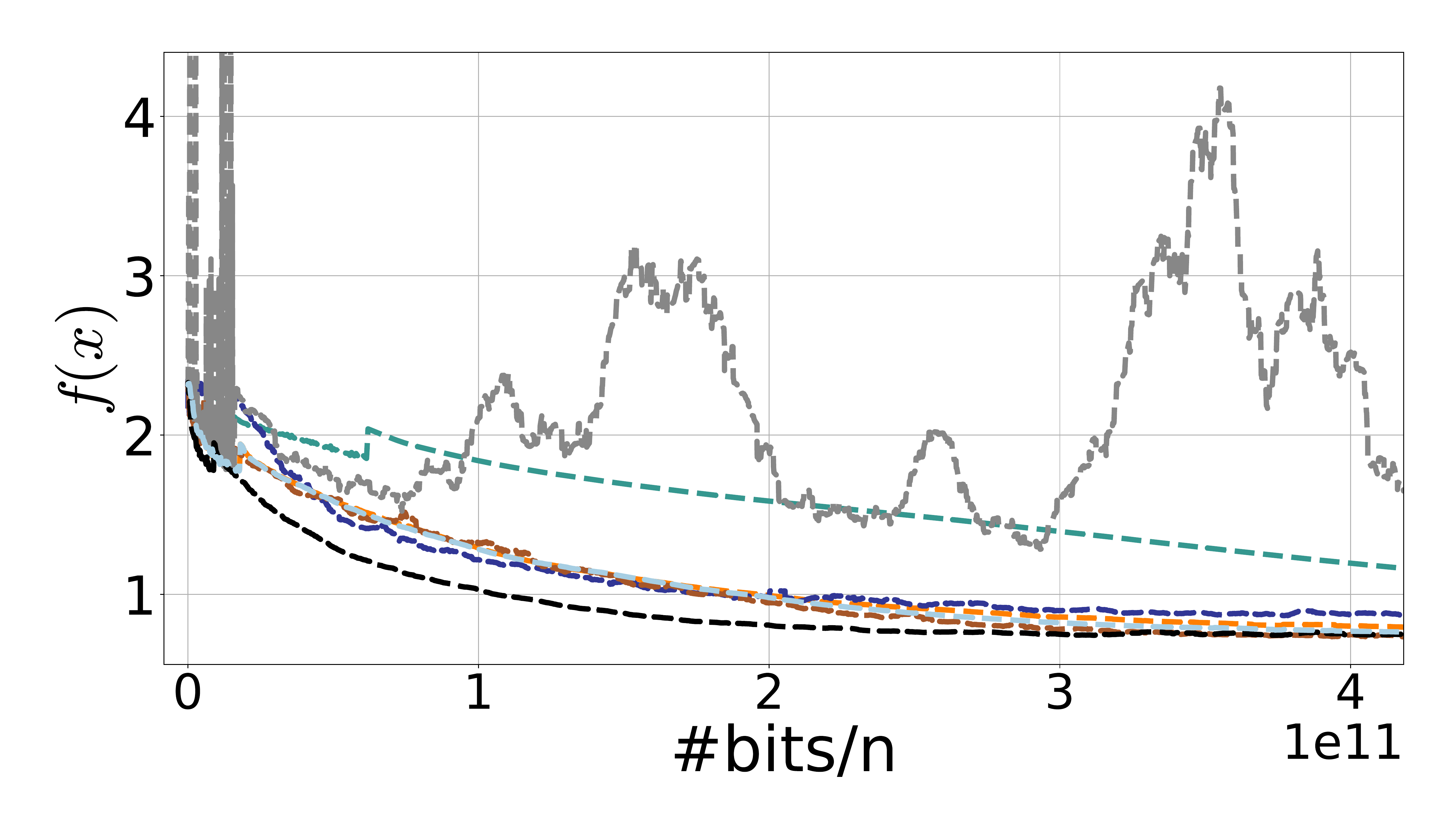} \caption{}
	\end{subfigure}
	\begin{subfigure}[ht]{0.45\textwidth}
		\includegraphics[width=\textwidth]{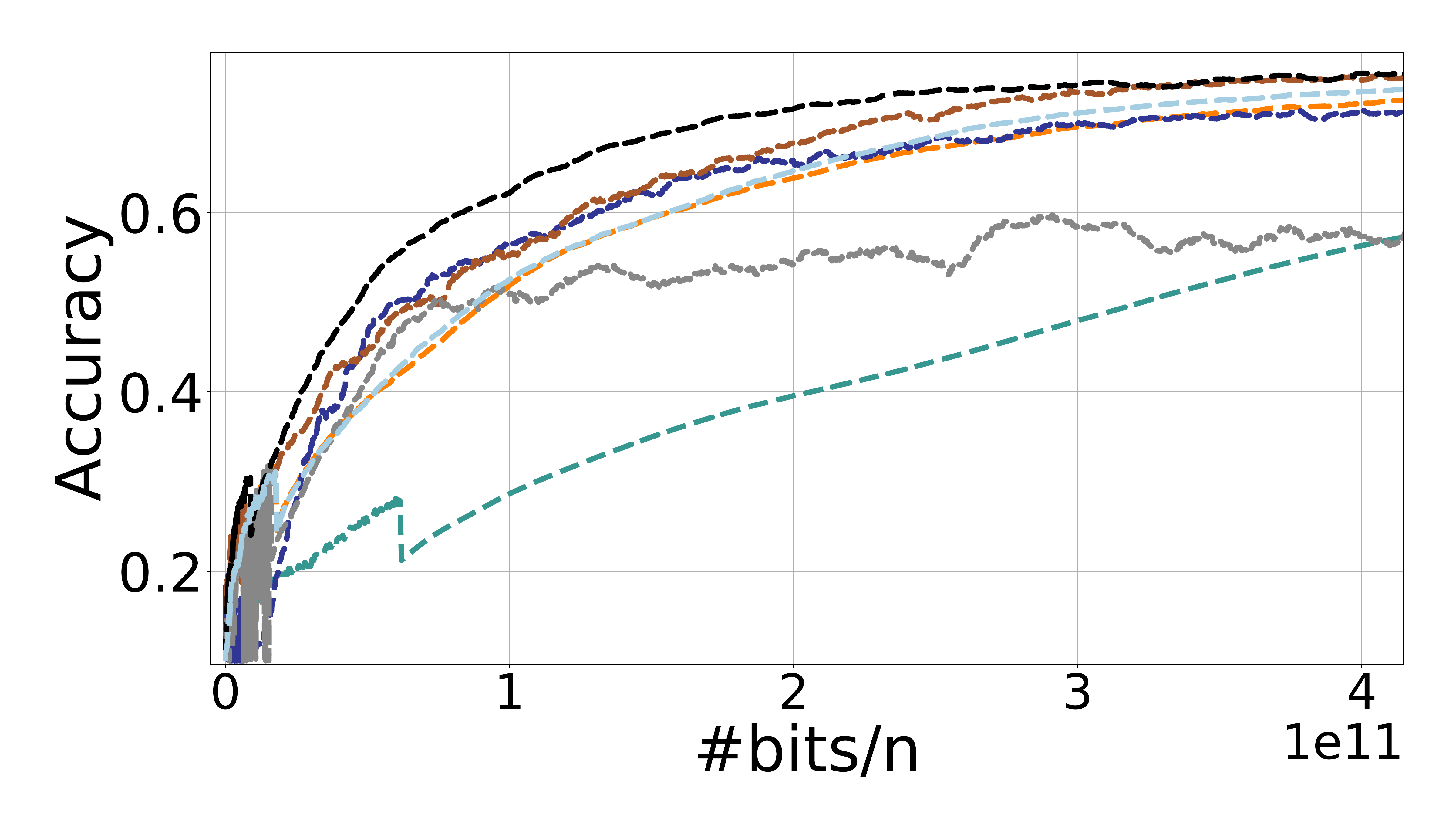}\caption{}
	\end{subfigure}
	\caption{{Training \texttt{MobileNet} on \texttt{CIFAR-10}, with $n=10$ workers. Loss and Top-1 accuracy on train (a) - (b), and test data (c) - (d).}}
	\label{fig:training_mobilenet_l2g_only}
\end{figure*}

\smartparagraph{Reproducible research.}
See our repository online: \url{https://github.com/burlachenkok/compressed-fl-l2gd-code}.
Our source codes have been constructed on top of the following version of FedML.ai: \url{https://github.com/FedML-AI/FedML/commit/3b9b68764d922ce239e0b84aceda986cfa977f96}.


\end{document}